\documentclass{article}
\PassOptionsToPackage{numbers,compress,sort}{natbib}
\usepackage[preprint]{neurips_2026}
\usepackage[utf8]{inputenc}
\usepackage[T1]{fontenc}
\usepackage{hyperref}
\usepackage{url}
\usepackage{booktabs}
\usepackage{amsfonts}
\usepackage{amsmath}
\usepackage{amssymb}
\usepackage{nicefrac}
\usepackage{microtype}
\usepackage{xcolor}
\usepackage{graphicx}
\usepackage{multirow}
\usepackage{colortbl}
\usepackage{algorithm}
\usepackage{algorithmic}
\usepackage{enumitem}
\usepackage{pgfplots}
\usepackage{tikz}
\usetikzlibrary{positioning, arrows.meta, shapes.geometric, fit, calc}
\pgfdeclarelayer{background}
\pgfsetlayers{background,main}
\usepackage{pifont}
\usepackage{amsthm}

\newtheorem{proposition}{Proposition}
\newtheorem{observation}{Observation}

\pgfplotsset{compat=1.18}

\title{Raw-Routed Mixture of Adapters: A Causal Intervention for Routing Collapse in Time Series Foundation Models}

\author{
  Hung Phan
  \And
  Thuy T. Nguyen
  \And
  Minh Ngoc Dinh
  \And
  Nhat-Quang Tran
}

\begin{document}

\maketitle

\begin{abstract}
Time series foundation models (TSFMs) commonly adapt to new data by attaching a single trainable head to a frozen backbone, a one-size-fits-all setup that underfits heterogeneous regimes. Replacing the head with a mixture of experts is the standard upgrade, but on instance-normalized backbones (the dominant TSFM design class) it fails: routing entropy collapses to zero and one expert absorbs every input, a failure we call \emph{normalization-induced routing collapse}. Standard MoE rescue mechanisms do not repair it, because the cause is in the router's input, not its optimization. Pre-encoder normalization strips the statistics a router would need to tell regimes apart. A mutual-information decomposition makes this precise and yields a signal-ratio that, computed before training, predicts dataset vulnerability (Spearman $\rho{=}{-}0.88$). Eight causal controls, including a vision-modality replication, isolate instance normalization as the cause. The prescription is a minimal causal intervention: \emph{Raw-Routed Mixture of Adapters} (RR-MoA), which routes on the raw, pre-normalization input. Under a strictly frozen backbone, RR-MoA wins $54/54$ comparisons against the strongest fixed adapter and significantly outperforms LoRA, TRACE, AdaMix, and full fine-tuning. The effect generalizes across six backbones and an imputation task. Frozen RR-MoA also beats full fine-tuning by $12$--$79\%$ (the \emph{Frozen Paradox}); two architecturally distinct variants confirm the principle generalizes beyond this specific router. Code is provided in the supplementary material.
\end{abstract}

\section{Introduction}
\label{sec:intro}

Time series foundation models (TSFMs)~\citep{goswami2024moment, das2024timesfm, ansari2024chronos, liu2025timerxl, liu2024timer, woo2024moirai, lee2024units, rasul2024lagllama, ekambaram2024ttm, cohen2025toto, liu2025sundial, jin2024timellm, zhou2023onefitsall, liu2024autotimes, pan2024s2ipllm, liu2024unitime, shi2024timemoe, liu2026timers1} are pre-trained backbones for general-purpose time-series analysis that transfer to unseen datasets with little target-specific training. Pretrained on heterogeneous time series spanning domains like energy, weather, traffic, and finance, they aim to learn regularities that generalize across tasks. The dominant deployment recipe holds the backbone fixed and trains only a small \emph{adapter head}, a per-task module mapping hidden states to target-specific predictions. Because the backbone is frozen, this head carries all dataset-specific information, making its design an important factor in transfer quality. In a major class of TSFMs (MOMENT~\citep{goswami2024moment}, TTM~\citep{ekambaram2024ttm}, and others), this head is uniform: a single pooling-and-projection rule applied to every input window, regardless of regime.

Probing analyses of TSFMs~\citep{wilinski2025tsfm_representations, pandey2025tsfm_semantics} indicate this is suboptimal: distinct temporal concepts (trends, dispersion, change-time) localize in different parts of the hidden state. A calm baseline and a volatile spike differ in these concepts, so they route their predictive signal through different subspaces, and no fixed pooling rule reads both. The remedy is to replace the single rule with a dispatcher over several specialized readers~\citep{wang2022adamix}, selecting for each window the expert whose pooling best matches its regime: an AdaMix-style mixture-of-experts (MoE) head.

On every TSFM whose encoder begins with an instance normalizer, this head degenerates as soon as a single transformer block is unfrozen for joint training (Figure~\ref{fig:problem_overview}b): gating mass contracts to one-hot in tens of steps (Table~\ref{tab:adamix}), and the favored expert's gradients reshape the backbone to serve it alone (Figure~\ref{fig:trajectory}). We name the phenomenon \emph{normalization-induced routing collapse} and trace it upstream of the optimizer: RevIN~\citep{kim2021revin} subtracts each window's mean and variance before encoding, leaving the statistics that distinguish experts structurally absent from the router's input (Observation~\ref{obs:routing_loss}). Eleven optimization-side rescues (load balancing, Z-loss, ReMoE, expert choice, top-$k$ sharpening, and seven more; $720$ runs) close only $10.9\%$ of the gap that an input-side fix closes ($2.7\times$ shortfall, Table~\ref{tab:rescue}); Moirai, whose I/O scaling is not learnable-affine, does not collapse (\S\ref{sec:cross_backbone}), narrowing the cause to learnable-affine encoder normalization.

We apply a minimal causal intervention: route on $\mathbf{X}_\text{raw}$, the input \emph{before} normalization, with experts, optimizer, and a strictly frozen backbone unchanged. We call the head \emph{Raw-Routed Mixture of Adapters} (RR-MoA). The same diagnosis predicts a \emph{Frozen Paradox}: backbone updates only serve the live expert, so unfreezing should hurt (Proposition~\ref{prop:frozen}); empirically frozen wins $4/6$ datasets and unfreezing never recovers more than $13\%$.

Our main contributions are:
\begin{itemize}[leftmargin=1.2em, itemsep=0.15em, topsep=0.2em]
\item We formulate routing collapse on instance-normalized TSFMs as a router-\emph{input} failure mode, distinct from and upstream of the optimization-side failure modes prior MoE work targets; the rescues above leave $89\%$ of the gap on the table that an input-side fix closes.
\item We derive a mutual-information decomposition (Observation~\ref{obs:routing_loss}) that makes precise what is removed before the router sees the input, and from it we obtain $R(\mathcal{D})$, a training-free signal-ratio computable from raw dataset statistics. $R(\mathcal{D})$ anti-correlates with collapse risk at Spearman $\rho{=}{-}0.88$ ($p{<}0.002$) across six backbones (three normalization families) and a vision modality, and correctly abstains on Traffic ($R{=}0.14$).
\item We propose \textbf{RR-MoA}, which changes only the router's input and wins $54/54$ cells on the primary benchmark ($26$--$79\%$ MSE), beating LoRA, TRACE, AdaMix, and full fine-tuning at $p{<}0.001$, while resisting naive $(\mu,\sigma)$ re-injection ($+42$--$239\%$ MSE; $11/18$ cells re-collapse). Two distinct instantiations confirm the principle is broader than this router: \textbf{SR-MoA} replaces the external router with per-expert sigmoids ($+13$--$42\%$ over RR-MoA on six datasets), and \textbf{Residual-IA\textsuperscript{+}} carries it into each expert's body, closing the DLinear gap across six backbones $\times$ four horizons.
\end{itemize}

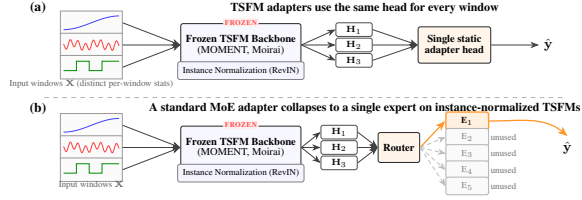
\begin{figure}[t]
\centering
\resizebox{0.55\columnwidth}{!}{%
\begin{tikzpicture}[
    >=Stealth,
    every node/.style={font=\scriptsize},
    box/.style={rectangle, draw=black!70, rounded corners=1.5pt, fill=white,
        minimum height=0.7cm, align=center, line width=0.6pt},
    bbbox/.style={box, draw=black!85, fill=blue!3, minimum width=2.4cm, minimum height=1.0cm,
        font=\scriptsize\bfseries},
    smallbox/.style={box, minimum width=0.85cm, minimum height=0.32cm, font=\tiny, inner xsep=3pt, inner ysep=1pt},
    waveframe/.style={rectangle, draw=black!40, fill=gray!3, line width=0.4pt,
        minimum width=1.45cm, minimum height=0.55cm, inner sep=1pt},
    headbox/.style={box, draw=black!85, fill=orange!8, minimum width=1.7cm, minimum height=0.85cm,
        font=\scriptsize\bfseries},
    routerbox/.style={box, draw=black!85, fill=orange!8, minimum width=0.95cm, minimum height=0.7cm,
        font=\scriptsize\bfseries},
    expertactive/.style={box, draw=orange!80, fill=orange!12, minimum width=1.0cm,
        minimum height=0.36cm, font=\tiny\bfseries, line width=0.9pt},
    expertinactive/.style={box, draw=black!30, fill=gray!4, minimum width=1.0cm,
        minimum height=0.36cm, font=\tiny, text=black!50},
    arr/.style={->, line width=0.55pt, color=black!75},
    activearr/.style={->, line width=0.85pt, color=orange!85},
    deadarr/.style={->, line width=0.4pt, color=black!25, densely dashed},
    banner/.style={font=\footnotesize\bfseries, align=center},
    sublabel/.style={font=\tiny, color=black!60, align=center},
]

\node[banner] (titleA) at (7.3, 3.65)
    {TSFM adapters use the same head for every window};

\node[waveframe] (w1a) at (0.9, 3.30) {};
\draw[blue!75, line width=0.5pt] plot[domain=-0.65:0.65, samples=30, smooth]
    ({\x+0.9},{3.30+0.16*sin(deg(2.4*\x))});
\node[waveframe] (w2a) at (0.9, 2.78) {};
\draw[red!75, line width=0.4pt] plot[domain=-0.65:0.65, samples=70]
    ({\x+0.9},{2.78+0.09*sin(deg(35*\x))+0.05*sin(deg(11*\x))});
\node[waveframe] (w3a) at (0.9, 2.26) {};
\draw[green!55!black, line width=0.5pt]
    (0.25,2.18) -- (0.55,2.18) -- (0.55,2.40) -- (0.85,2.40) --
    (0.85,2.18) -- (1.15,2.18) -- (1.15,2.40) -- (1.45,2.40) -- (1.55,2.40);
\node[sublabel] at (0.9, 1.86) {Input windows $\mathbf{X}$ (distinct per-window stats)};

\node[bbbox] (bbA) at (4.4, 2.78)
    {Frozen TSFM Backbone\\[-1pt]\scriptsize\mdseries(MOMENT, Moirai)};
\node[font=\tiny\bfseries, text=red!70, fill=red!7, rounded corners=1pt, inner sep=1pt]
    at ([yshift=0.02cm]bbA.north) {\textsc{frozen}};
\node[box, draw=black!50, fill=blue!4, minimum width=2.3cm, minimum height=0.32cm,
    font=\tiny] (revA) at (4.4, 2.18) {Instance Normalization (RevIN)};

\node[smallbox] (hA1) at (7.05, 3.14) {$\mathbf{H}_1$};
\node[smallbox] (hA2) at (7.05, 2.78) {$\mathbf{H}_2$};
\node[smallbox] (hA3) at (7.05, 2.42) {$\mathbf{H}_3$};

\node[headbox] (headA) at (9.4, 2.78) {Single static\\[-1pt]adapter head};

\node[font=\small] (yA) at (11.6, 2.78) {$\hat{\mathbf{y}}$};

\foreach \w in {w1a, w2a, w3a} {\draw[arr] (\w.east) -- (bbA.west);}
\draw[arr] (bbA.east) -- (hA1.west);
\draw[arr] (bbA.east) -- (hA2.west);
\draw[arr] (bbA.east) -- (hA3.west);
\foreach \h in {hA1, hA2, hA3} {\draw[arr] (\h.east) -- (headA.west);}
\draw[arr] (headA.east) -- (yA.west);

\draw[black!30, dashed, line width=0.4pt] (-0.2,1.55) -- (12.0,1.55);

\node[banner] (titleB) at (7.3, 1.30)
    {A standard MoE adapter collapses to a single expert on instance-normalized TSFMs};

\node[waveframe] (w1b) at (0.9, 0.90) {};
\draw[blue!75, line width=0.5pt] plot[domain=-0.65:0.65, samples=30, smooth]
    ({\x+0.9},{0.90+0.16*sin(deg(2.4*\x))});
\node[waveframe] (w2b) at (0.9, 0.38) {};
\draw[red!75, line width=0.4pt] plot[domain=-0.65:0.65, samples=70]
    ({\x+0.9},{0.38+0.09*sin(deg(35*\x))+0.05*sin(deg(11*\x))});
\node[waveframe] (w3b) at (0.9, -0.14) {};
\draw[green!55!black, line width=0.5pt]
    (0.25,-0.22) -- (0.55,-0.22) -- (0.55,0.00) -- (0.85,0.00) --
    (0.85,-0.22) -- (1.15,-0.22) -- (1.15,0.00) -- (1.45,0.00) -- (1.55,0.00);
\node[sublabel] at (0.9, -0.50) {Input windows $\mathbf{X}$};

\node[bbbox] (bbB) at (4.4, 0.38)
    {Frozen TSFM Backbone\\[-1pt]\scriptsize\mdseries(MOMENT, Moirai)};
\node[font=\tiny\bfseries, text=red!70, fill=red!7, rounded corners=1pt, inner sep=1pt]
    at ([yshift=0.02cm]bbB.north) {\textsc{frozen}};
\node[box, draw=black!50, fill=blue!4, minimum width=2.3cm, minimum height=0.32cm,
    font=\tiny] (revB) at (4.4, -0.22) {Instance Normalization (RevIN)};

\node[smallbox] (hB1) at (6.7, 0.74) {$\mathbf{H}_1$};
\node[smallbox] (hB2) at (6.7, 0.38) {$\mathbf{H}_2$};
\node[smallbox] (hB3) at (6.7, 0.02) {$\mathbf{H}_3$};

\node[routerbox] (routerB) at (8.1, 0.38) {Router};

\node[expertactive] (e1) at (9.7, 1.00) {E$_1$};
\node[expertinactive] (e2) at (9.7, 0.62) {E$_2$};
\node[expertinactive] (e3) at (9.7, 0.24) {E$_3$};
\node[expertinactive] (e4) at (9.7, -0.14) {E$_4$};
\node[expertinactive] (e5) at (9.7, -0.52) {E$_5$};
\node[font=\tiny, color=black!55] at (10.55, 0.62) {unused};
\node[font=\tiny, color=black!55] at (10.55, 0.24) {unused};
\node[font=\tiny, color=black!55] at (10.55, -0.14) {unused};
\node[font=\tiny, color=black!55] at (10.55, -0.52) {unused};

\node[font=\small] (yB) at (12.05, 0.38) {$\hat{\mathbf{y}}$};

\foreach \w in {w1b, w2b, w3b} {\draw[arr] (\w.east) -- (bbB.west);}
\draw[arr] (bbB.east) -- (hB1.west);
\draw[arr] (bbB.east) -- (hB2.west);
\draw[arr] (bbB.east) -- (hB3.west);
\foreach \h in {hB1, hB2, hB3} {\draw[arr] (\h.east) -- (routerB.west);}
\draw[activearr] (routerB.east) -- (e1.west);
\draw[deadarr] (routerB.east) -- (e2.west);
\draw[deadarr] (routerB.east) -- (e3.west);
\draw[deadarr] (routerB.east) -- (e4.west);
\draw[deadarr] (routerB.east) -- (e5.west);
\draw[activearr] (e1.east) to[out=0, in=140] (yB.north);

\node[font=\small\bfseries, text=black] at (-0.35, 3.65) {(a)};
\node[font=\small\bfseries, text=black] at (-0.35, 1.30) {(b)};

\end{tikzpicture}%
}%
\caption{\textbf{Problem overview.} \textbf{(a)}~TSFM adapters use the same head for every window. \textbf{(b)}~The standard mixture-of-experts extension fails on instance-normalized TSFMs under unfreezing: entropy collapses ($0.000{\pm}0.000$) and the remaining experts receive no gradient. Same backbone, different heads.}
\label{fig:problem_overview}
\vspace{-1.0em}
\end{figure}

\section{Related Work}
\label{sec:related_work}

RR-MoA sits at the intersection of TSFM adaptation~\citep{wen2023tssurvey} and MoE design~\citep{cai2024moesurvey}. \emph{TSFM adaptation:} compact specialists rival heavyweight backbones on LTSF~\citep{wu2025srsnet, wang2025timemixerpp, xu2024fits, chen2023tsmixer, fu2025selective}; LLM-reprogrammed forecasters inherit upstream normalization~\citep{jin2024timellm, zhou2023onefitsall, liu2024autotimes}; the adapter literature varies finetuning strategy or backbone selection~\citep{houlsby2019adapters, qiao2025msft, zhao2025prune, faw2025icf, ning2025tsrag, gupta2024beyondlora, zhang2025template, woo2024gifteval}, and AdaPTS~\citep{benechehab2025adapts} treats adapter topology as a per-task hyperparameter; to our knowledge no prior TSFM work uses a per-window mixture over topologically distinct adapters. \emph{MoE design space:} scaling axes~\citep{lepikhin2021gshard, riquelme2021vmoe, komatsuzaki2023sparseupcycling, raposo2024mod, jiang2024mixtral, zhu2024llamamoe, dai2024deepseekmoe} and in-backbone TS-MoE~\citep{shi2024timemoe, he2025sempo} act on token routing within the encoder; SR-MoA's per-expert sigmoid generalizes MMoE's multi-gate pattern~\citep{ma2018mmoe} from multi-task to single-task specialization, and \citet{wang2026myth} argues MoE specialization reflects hidden-state geometry, which we extend upstream by characterizing what shapes that geometry on TSFMs. The instance-normalization origin~\citep{ulyanov2016instancenorm} treats stripping as benign preprocessing; we identify a downstream MoE-adapter consequence not previously characterized.

\textbf{MoE routing collapse and rescue.} Existing diagnoses cluster into three threads. \emph{Optimization-side}: load-balancing~\citep{fedus2022switch}, z-loss~\citep{zoph2022stmoe}, and AdaMix's hidden-state router~\citep{wang2022adamix} target training dynamics; \citet{guo2025expertspec} argue load-balance loss itself causes expert overlap and propose orthogonality regularizers. \emph{Representation-geometry}: \citet{chi2022representation} attribute collapse to token clustering and propose hyperspherical scores; \citet{wu2024mhmoe} split tokens into sub-tokens to raise expert activation. \emph{Architectural / paradigm}: Expert-Choice~\citep{zhou2022expertchoice} inverts token--expert assignment, Soft MoE~\citep{puigcerver2024softmoe} replaces hard top-$k$ with a soft mixture, ReMoE~\citep{wang2025remoe} replaces top-$k$ with ReLU gating, \citet{panda2024densebackprop} restore a dense gradient under sparse forward passes, and \citet{hua2025inputaware} decouple the router from the task objective. Our setting differs from all three: the routing signal is structurally absent from the router's input \emph{by design of an upstream architectural normalizer (RevIN)}, before training dynamics, geometry, or routing-paradigm choices can act. Sweeping five rescue families across twelve configurations on MOMENT$+$RevIN, none recover MSE to within $2.7{\times}$ of RR-MoA (Table~\ref{tab:rescue}). Concurrent work eliminates the external router (AoE~\citep{lv2025aoe}, Routing-Free MoE~\citep{liu2026routingfree}) or replaces softmax+load-balance with eigenbasis routing (ERMoE~\citep{cheng2025ermoe}); in our setting (\S\ref{sec:main_results}), self-gated experts on raw input match RR-MoA, while the same architecture on hidden states collapses.

\textbf{Normalization in Time Series.} RevIN~\citep{kim2021revin} adapts instance normalization to forecasting distribution shift; DishTS~\citep{fan2023dishts}, SAN~\citep{liu2023san}, and Non-stationary Transformers~\citep{liu2022nonstationary} partially recover stripped statistics, FAN~\citep{ye2024fan} extends beyond $(\mu,\sigma)$ to Fourier components, and DDN~\citep{dai2024ddn} normalizes jointly in time and frequency; \citet{berthelier2026revin} and \citet{zou2025ibnorm} concurrently flag RevIN redundancy and the limits of variance-centric normalization. We identify a specific MoE consequence not previously characterized: when stripped statistics carry the \emph{routing} signal, instance normalization collapses the router under unfreezing, distinct from the optimization-induced collapse standard rescues target.

\section{Raw-Routed Mixture of Adapters (RR-MoA)}
\label{sec:method}

RR-MoA is a per-window mixture in which the router reads the raw, pre-normalization input while the experts consume the frozen TSFM's hidden states. We derive it from a routing-information loss (Observation~\ref{obs:routing_loss}, below) that locates the bottleneck in the router's \emph{input}, not its optimizer; task setup, mechanism, and architecture follow in that order.

\textbf{Task setup.} Given a window $\mathbf{x}\in\mathbb{R}^{T\times c}$ ($T{=}512$, $c$ channels), forecasting predicts $\mathbf{y}\in\mathbb{R}^{h\times c}$ at horizons $h\in\{96,192,336,720\}$ under MSE loss (imputation uses the same formalism with $\mathbf{y}$ as the reconstruction target). A pretrained TSFM acts as a feature extractor $f_\theta:\mathbb{R}^{T\times c}\to\mathbb{R}^{P\times d}$ ($P$ patches of dim $d$) with $\theta$ \emph{frozen} ($\nabla_\theta\!\equiv\!0$); we train only an adapter $g_\phi:\mathbb{R}^{P\times d}\to\mathbb{R}^{h\times c}$ with a small parameter budget ($|\phi|\leq 500$K), so that many task-specific adapters can be hot-swapped onto a single shared backbone (App.~\ref{app:deployment}). A \emph{per-window mixture of adapters} replaces $g_\phi$ with $K$ experts $\{\mathrm{Expert}_j\}_{j=1}^K$ plus a router $G_\psi$ that returns a sparse weight $\widetilde{w}\in\Delta^{K-1}$ supported on a Top-$k$ subset ($k\leq K$), giving $\hat{\mathbf{y}}=\sum_{j} \widetilde{w}_j\,\mathrm{Expert}_j(f_\theta(\mathbf{x}))$ trained by $\min_{\phi,\psi}\,\mathbb{E}\,\|\hat{\mathbf{y}}-\mathbf{y}\|_2^2$.

\textbf{Problem statement.} The standard MoE recipe~\citep{shazeer2017outrageously, fedus2022switch} lets the router $G_\psi$ read the backbone's hidden states $f_\theta(\mathbf{x})$, and this is the choice taken by the closest TSFM-adapter precedent, AdaMix~\citep{wang2022adamix}. We show that on TSFMs equipped with instance normalization (RevIN~\citep{kim2021revin}, the de facto standard for distribution-shift handling in MOMENT~\citep{goswami2024moment}, PatchTST~\citep{nie2023patchtst}, and TimesFM~\citep{das2024timesfm}) this choice has a structural failure mode not previously characterized: the router's input is missing, by construction of an upstream architectural normalizer, the per-window first- and second-order statistics that distinguish regimes (formalized in Observation~\ref{obs:routing_loss}), so $G_\psi$ collapses to a single expert and the remaining $K{-}1$ experts receive no gradient. We name this failure \emph{normalization-induced routing collapse} and state it as the central problem of this work:
\begin{quote}
\itshape Given a frozen instance-normalized TSFM, find the routing input $u(\mathbf{x})$ such that the per-window mixture $\sum_j G_\psi(u(\mathbf{x}))_j\,\mathrm{Expert}_j(f_\theta(\mathbf{x}))$ trains without entropy collapse and improves over a single static head.
\end{quote}
This collapse is upstream of the optimization-induced collapse that existing rescues target -- load balancing~\citep{shazeer2017outrageously, fedus2022switch}, z-loss~\citep{zoph2022stmoe}, ReLU routing~\citep{wang2025remoe}, and expert-choice~\citep{zhou2022expertchoice} -- and is not addressed by prior MoE-adapter~\citep{wang2022adamix, wu2024mole}, MoE routing-collapse~\citep{chi2022representation, wu2024mhmoe, puigcerver2024softmoe, panda2024densebackprop, hua2025inputaware, guo2025expertspec}, or time-series-normalization~\citep{liu2022nonstationary, fan2023dishts, liu2023san, ye2024fan, dai2024ddn, berthelier2026revin, zou2025ibnorm} work. We next formalize the loss (Observation~\ref{obs:routing_loss}) and define a training-free predictor $R(\mathcal{D})$ for it; the rest of the section then specifies RR-MoA, the answer to the problem with $u(\mathbf{x})\!=\!\mathbf{x}$, and \S\ref{sec:main_results} verifies the answer with eight causal controls and a $720$-run rescue sweep (Table~\ref{tab:rescue}).

\textbf{Mechanism.}\label{sec:mechanism}\label{sec:diagnosis}
We formalize the failure as a routing-information loss across the normalization step. Each window decomposes into mean $M_i$, scale $\Sigma_i$, and shape $\mathbf{S}_i = (\mathbf{x}_i - M_i)/\Sigma_i$; an instance normalizer keeps $\mathbf{S}_i$ and discards $(M_i,\Sigma_i)$. The chain rule then quantifies the routing signal in the discarded part.

\begin{observation}[Routing Information Loss Under Instance Normalization]\label{obs:routing_loss}
Let $X$ be a random input window, $E \in \{1,\ldots,K\}$ the expert assignment produced by any deterministic router, $S = (X - M)/\Sigma$ the RevIN-normalized shape. \emph{(i) Exact decomposition:} $I(X;E) - I(S;E) = I(M,\Sigma;E\mid S) \geq 0$. \emph{(ii) Quantitative bound (no independence assumption):} $I(X;E) - I(S;E) \geq I(M,\Sigma;E) - I(M,\Sigma;S)$, with $\varepsilon := I(M,\Sigma;S) \geq 0$ capturing residual leakage of stripped statistics through the shape; equality iff $(M,\Sigma) \perp\!\!\!\perp S \mid E$.
\end{observation}

When $E=f(M,\Sigma)$ and the stripped statistics are independent of the shape ($(M,\Sigma)\perp\!\!\!\perp S$; App.~\ref{app:proofs}), $I(X;E){-}I(S;E){=}H(E)$: instance normalization destroys the entire routing distribution. The signal ratio $R(\mathcal{D}) = \bigl(\sum_c\mathrm{Var}(M_c){+}\sum_c\mathrm{Var}(\Sigma_c)\bigr)/\overline{\mathrm{Var}(S)}$ (channel-summed numerator over window-and-channel-averaged shape variance) is a tractable, training-free monotone proxy: high when the stripped statistics carry routing information, low otherwise. Empirical tightness of the bound, the predictor's correlation with raw-routing benefit, and the boundary cases that falsify a generic-SNR alternative appear in \S\ref{sec:main_results}.

\begin{figure}[t!]
\centering
\resizebox{0.55\columnwidth}{!}{%
\begin{tikzpicture}[
    >=Stealth,
    box/.style={rectangle, draw=black!70, rounded corners=2pt, fill=white,
        minimum height=0.75cm, font=\small, align=center, line width=0.7pt},
    frozenbox/.style={box, draw=blue!60, fill=blue!4, minimum width=2.6cm, minimum height=1.4cm},
    gatebox/.style={box, draw=orange!70, fill=orange!4, minimum width=2.0cm, minimum height=0.9cm},
    hiddenbox/.style={box, draw=blue!50, fill=blue!3, minimum width=1.4cm},
    topkbox/.style={box, draw=orange!60, fill=orange!3, minimum width=1.3cm},
    expertactive/.style={box, draw=orange!75!black, fill=orange!18, line width=1.0pt, minimum width=2.0cm, minimum height=0.65cm, font=\small\bfseries, text=black},
    expertinactive/.style={box, draw=gray!45, fill=gray!8, minimum width=2.0cm, minimum height=0.65cm, font=\small, text=black},
    sumnode/.style={circle, draw=black!70, fill=white, minimum size=0.65cm, font=\normalsize, line width=0.7pt},
    outputbox/.style={box, draw=teal!60, fill=teal!4, minimum width=1.2cm},
    bluearr/.style={->, line width=0.8pt, color=blue!55},
    orangearr/.style={->, line width=0.8pt, color=orange!70},
    dashedarr/.style={->, line width=0.8pt, color=orange!70, densely dashed},
    grayarr/.style={->, line width=0.7pt, color=gray!50},
    outarr/.style={->, line width=0.8pt, color=teal!60},
]

\node[box, minimum width=1.4cm, minimum height=1.0cm] (input) at (0, 0) {
    $\mathbf{X}_\mathrm{raw}$\\[-2pt]
    {\scriptsize raw input}\\[-2pt]
    {\scriptsize 512 steps}
};

\node[frozenbox] (tsfm) at (3.5, 1.8) {
    \textbf{Frozen TSFM}\\[-1pt]
    \textbf{Backbone}\\[-1pt]
    {\scriptsize\color{red!60} no gradients}
};
\node[font=\scriptsize\bfseries, text=red!65, fill=red!8, rounded corners=2pt, inner sep=1.5pt]
    at ([xshift=-0.6cm, yshift=0.18cm]tsfm.north east) {\textsc{frozen}};

\node[hiddenbox] (hidden) at (6.8, 1.8) {
    $\mathbf{H}$\\[-2pt]
    {\scriptsize hidden}\\[-2pt]
    {\scriptsize states}
};

\node[gatebox] (gate) at (3.5, -1.8) {
    \textbf{Conv1d Gate}\\[-2pt]
    {\scriptsize 853 params}
};

\node[topkbox] (topk) at (6.8, -1.8) {
    \textbf{Top-2}\\[-2pt]
    {\scriptsize sparse}
};

\node[expertinactive] (e1) at (9.8, 3.2) {Mean Pool};
\node[expertactive]   (e2) at (9.8, 2.0) {\textbf{Last Token}};
\node[expertinactive] (e3) at (9.8, 0.8) {Max Pool};
\node[expertactive]   (e4) at (9.8, -0.4) {\textbf{Attention}};
\node[expertinactive] (e5) at (9.8, -1.6) {Conv1d};

\draw[gray!60, line width=0.6pt, line cap=round]
    (10.95, 3.55) -- (11.05, 3.55) -- (11.05, -1.95) -- (10.95, -1.95);
\node[font=\scriptsize, text=black, anchor=north] at (10.5, -2.15) {$K{=}5$ experts};

\node[sumnode] (sum) at (12.5, 0.8) {$\Sigma$};
\node[font=\tiny, text=black, fill=white, inner sep=1pt] at (12.5, 0.2) {weighted};
\node[outputbox] (output) at (14.3, 0.8) {
    $\hat{\mathbf{y}}$\\[-2pt]
    {\scriptsize forecast}
};

\node[font=\scriptsize\bfseries, text=blue!55, anchor=west] at (0.30, 3.10) {\textsc{Path A: Embeddings}};
\node[font=\scriptsize\bfseries, text=orange!70, anchor=west] at (0.30, -3.10) {\textsc{Path B: Routing}};

\begin{pgfonlayer}{background}
\draw[bluearr]   (input.north) |- (tsfm.west);
\draw[orangearr] (input.south) |- (gate.west);

\draw[bluearr] (tsfm.east) -- (hidden.west);
\draw[bluearr] (hidden.east) -- ++(0.40,0) |- (e1.west);
\draw[bluearr] (hidden.east) -- ++(0.55,0) |- (e2.west);
\draw[bluearr] (hidden.east) -- ++(0.70,0) |- (e3.west);
\draw[bluearr] (hidden.east) -- ++(0.85,0) |- (e4.west);
\draw[bluearr] (hidden.east) -- ++(1.00,0) |- (e5.west);

\draw[orangearr] (gate.east) -- (topk.west);
\draw[dashedarr] (topk.east) -- ++(0.4,0) |- ([yshift=-0.18cm]e2.west);
\draw[dashedarr] (topk.east) -- ++(0.4,0) |- ([yshift=-0.18cm]e4.west);

\draw[bluearr] (e2.east) -| (sum.north);
\draw[bluearr] (e4.east) -| (sum.south);
\draw[outarr]  (sum.east) -- (output.west);
\end{pgfonlayer}

\node[font=\scriptsize, text=orange!70, anchor=north east] at (8.85, 1.78) {$w_2$};
\node[font=\scriptsize, text=orange!70, anchor=north east] at (8.85, -0.62) {$w_4$};

\end{tikzpicture}%
}%
\caption{\textbf{RR-MoA architecture.} \textit{Path~A} (blue): frozen TSFM produces $\mathbf{H}$ for all experts. \textit{Path~B} (orange): Conv1d gate ($853$ params) reads $\mathbf{X}_\mathrm{raw}$ and selects Top-2 experts; inactive experts (gray) receive no gradient. Standard hidden-state MoE collapses (Table~\ref{tab:adamix}); RR-MoA wins \textbf{54/54} (Table~\ref{tab:rrmoa}).}
\label{fig:framework}
\vspace{-0.8em}
\end{figure}
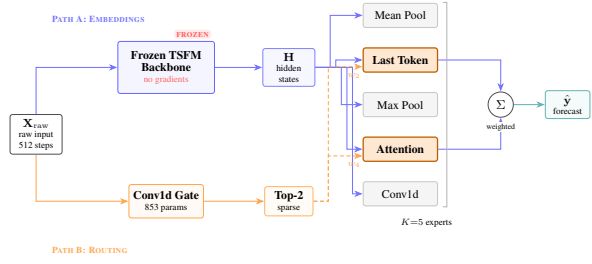

\label{sec:rrmoa}

\textbf{Architecture.} Per-window routing should dispatch a quiet baseline and a volatile spike to different expert topologies, but doing so requires the router to see the per-window statistics that distinguish regimes -- exactly what RevIN strips. RR-MoA splits the two needs into two streams: the frozen TSFM produces embeddings $\mathbf{H}{=}f_\theta(\mathbf{x})$ for the experts, while a lightweight gate reads the \textit{raw, pre-normalization} input $\mathbf{X}_\text{raw}$ to select which experts to activate:
\begin{equation}
\begin{aligned}
\mathbf{Y} &= \sum_{j=1}^{K} \widetilde{w}_j \cdot \text{Expert}_j(\mathbf{H}_\text{FM}), \qquad
\mathcal{T} = \mathrm{TopK}\!\bigl(G_\psi(\mathbf{X}_\text{raw}),\, k\bigr), \\
\widetilde{w}_j &= \begin{cases}
\dfrac{\exp\!\bigl(G_\psi(\mathbf{X}_\text{raw})_j\bigr)}{\sum_{j'\in\mathcal{T}}\exp\!\bigl(G_\psi(\mathbf{X}_\text{raw})_{j'}\bigr)}, & j\in\mathcal{T},\\[4pt]
0, & j\notin\mathcal{T}.
\end{cases}
\end{aligned}
\label{eq:rrmoa}
\end{equation}
where $G_\psi: \mathbb{R}^{T\times c} \to \mathbb{R}^K$ is a small Conv1d\,+\,pooling\,+\,linear gate ($853$ params) operating directly on the raw time series, $\mathcal{T} \subset \{1,\ldots,K\}$ is the set of Top-$k$ indices, and each $\mathrm{Expert}_j$ is a canonical adapter head. By construction $\sum_j \widetilde{w}_j = 1$ with support only on $\mathcal{T}$, so only $k$ of $K$ experts execute per sample. With $K{=}5$, $k{=}2$ we obtain $40\%$ of the dense expert FLOPs while retaining $63\%$ of the dense MSE improvement over the best fixed adapter (App.~\ref{app:topk}). Routing on raw input preserves the temporal statistics (trend, amplitude, volatility) that RevIN strips away.

\textbf{Bounded overhead, diverse experts, multi-tenant hot-swap.} The router is a Conv1d gate of $853$ parameters ($0.4\%$ of one expert head, ${<}10^{-3}$ of the frozen backbone); with $k{=}2$ of $K{=}5$ active per window, RR-MoA adds two adapter forward passes plus the gate to one frozen TSFM forward. The expert pool itself comprises $K{=}5$ topologically distinct adapters (mean-pool, last-token, max-pool, attention-pool, Conv1d-pool), spanning pooling families with distinct expressivity regimes for transformer-based sequence models~\citep{ennadir2025pooling}; architectural diversity matters, since five identical experts (same topology, different initialization) degrade MSE by $8$--$15\%$ (Appendix~\ref{app:diversity}; pseudocode in Algorithm~\ref{alg:rrmoa}, Appendix~\ref{app:setup}). Because the backbone is strictly frozen, a single TSFM held in GPU memory can be shared across many tenants while per-task adapter pools ($354$--$466$K params each) hot-swap from host RAM, an arrangement DLinear and per-dataset-from-scratch baselines cannot match (Appendix~\ref{app:benchmark}).

We now validate that this causal intervention recovers routing diversity, beats the seven-method baseline suite (fixed adapters, LoRA, TRACE, independent ensembles, AdaMix, full fine-tuning, DLinear), and holds under eight causal controls.

\section{Experiments}

\subsection{Setup}

\textbf{Datasets and tasks.} We evaluate on six LTSF benchmarks (ETTh1/h2/m1/m2~\citep{zhou2021informer}, Weather, Electricity~\citep{wu2021autoformer}) spanning hourly to 15-minute sampling and $7$--$321$ channels, with chronological splits and channel-wise normalization~\citep{nie2023patchtst}. Tasks: forecasting at $H{=}96$ primary (multi-horizon in Appendix~\ref{app:horizon}) and $20\%$-masked imputation.

\textbf{Backbones.} We use six backbones spanning three normalization regimes: (i)~learnable-affine RevIN-inside-encoder MOMENT-small/large~\citep{goswami2024moment} (internal RevIN~\citep{kim2021revin}); (ii)~Moirai~\citep{woo2024moirai} and Moirai-MoE, which use RMSNorm internally with only non-learnable per-instance I/O scaling (no learnable-affine RevIN inside the encoder pipeline); (iii)~Chronos~\citep{ansari2024chronos} (T5-style, no instance normalization) and Timer-XL~\citep{liu2025timerxl} (LayerNorm only, no instance normalization) as no-instance-normalization negative controls (per-backbone normalization specifications in Appendix~\ref{app:setup}; full cross-backbone results in App.~\ref{app:cross_backbone}). The collapse mechanism predicts that only regime~(i) is vulnerable: routing signal is destroyed by RevIN's learnable-affine normalization upstream of the encoder, whereas non-learnable I/O scaling preserves the (M,$\Sigma$) statistics in the hidden states the router reads. MOMENT-small is primary. Primary results use a strictly frozen backbone; freeze-level ablations follow.

\textbf{Baselines.} We organize baselines into four tiers, kept separate to avoid confounding frozen-vs-unfrozen and adapter-vs-from-scratch comparisons: \emph{frozen-adapter} (three fixed adapters; $5$-expert independent ensemble; AdaMix~\citep{wang2022adamix}); \emph{frozen PEFT} (LoRA~\citep{hu2022lora}, $108$-run sweep, App.~\ref{app:lora_sweep}; TRACE~\citep{li2025trace}); \emph{unfrozen} (full fine-tuning, all blocks, best of $5$ heads $\times 2$ LRs; extended schedules in App.~\ref{app:extended_ft}); \emph{from-scratch calibration} (DLinear~\citep{zeng2023dlinear}).

\textbf{Training and statistics.} Adam ($\eta{=}10^{-3}$), MSE loss, $15$ epochs, batch $128$; $5$ seeds for the core grid, $3$ for ablations. Significance via Wilcoxon signed-rank, Bonferroni-corrected over 6 datasets $\times$ 5 seeds ($^{*}\!/^{**}\!/^{***}$ for $p{<}0.05/0.01/0.001$). Full hyperparameters in App.~\ref{app:setup}; code and per-run results are in the supplementary material ($\sim$235 A10G GPU-hours total).

\textbf{Research questions.} The experiments answer three questions implied by the mechanism (\S\ref{sec:mechanism}). The first concerns the \emph{phenomenon (RQ1)}: does the standard hidden-state MoE collapse on instance-normalized TSFMs, and does a raw-routed mixture beat fixed adapters and full fine-tuning? The second concerns the \emph{mechanism (RQ2)}: what is the underlying cause, and do its falsifiable predictions hold under controls and rescue attempts? The third concerns \emph{generality (RQ3)}: does the diagnosis transfer to the seven-baseline suite (fixed adapters, LoRA, TRACE, independent ensembles, AdaMix, full fine-tuning, DLinear), other backbones, horizons, and tasks?

\subsection{Main Results}
\label{sec:main_results}

We report findings for RQ1 (phenomenon), RQ2 (mechanism), and RQ3 (generality) in turn.

\subsubsection*{Standard MoE collapses; RR-MoA wins 54/54 vs fixed adapters}

{\looseness=-1 \textbf{RR-MoA vs.\ fixed adapters.} RR-MoA wins all \textbf{54/54} configurations (6 datasets $\times$ 3 freeze levels $\times$ 3 seeds) against the best fixed adapter from the same five-head pool, with tight standard deviations (Table~\ref{tab:rrmoa}).}

\begin{table}[!htbp]
\centering
\caption{\textbf{Freeze-level ablation} (test MSE, mean$\pm$std, 3 seeds $\{42,43,44\}$, H=96). RR-MoA (sparse Top-2 routing over 5 experts) vs.\ the best single-head adapter from $\{$linear, attention, conv$\}$ trained with the same protocol. ``MSE $\Delta\%$ vs.\ single-head'' is the relative MSE change of RR-MoA over the single-head baseline; negative $=$ lower MSE $=$ better. RR-MoA wins \textbf{27/27} paired (cell, seed) comparisons in this table; combined with Appendix~\ref{app:extended_rrmoa} the score is \textbf{54/54} across 6 datasets $\times$ 3 freeze levels $\times$ 3 seeds. The gain is largest when frozen (``Frozen Paradox'').}
\label{tab:rrmoa}
\small
\begin{tabular}{@{}llccc@{}}
\toprule
Dataset & Freeze level & RR-MoA & Best single-head & MSE $\Delta\%$ vs.\ single-head \\
\midrule
\multirow{3}{*}{ETTh1}
 & Frozen (0/8)   & $\mathbf{0.690 \pm 0.021}$ & $1.241 \pm 0.015$ & $\mathbf{-44.4\%}$ \\
 & Last-2 (2/8)   & $\mathbf{0.727 \pm 0.074}$ & $1.030 \pm 0.139$ & $-29.4\%$ \\
 & Last-4 (4/8)   & $\mathbf{0.749 \pm 0.036}$ & $1.101 \pm 0.120$ & $-31.9\%$ \\
\midrule
\multirow{3}{*}{ETTm1}
 & Frozen (0/8)   & $\mathbf{0.571 \pm 0.073}$ & $1.115 \pm 0.075$ & $\mathbf{-48.7\%}$ \\
 & Last-2 (2/8)   & $\mathbf{0.623 \pm 0.032}$ & $0.891 \pm 0.049$ & $-30.1\%$ \\
 & Last-4 (4/8)   & $\mathbf{0.571 \pm 0.034}$ & $0.866 \pm 0.016$ & $-34.1\%$ \\
\midrule
\multirow{3}{*}{Weather}
 & Frozen (0/8)   & $\mathbf{0.289 \pm 0.008}$ & $0.530 \pm 0.020$ & $-45.5\%$ \\
 & Last-2 (2/8)   & $\mathbf{0.251 \pm 0.005}$ & $0.478 \pm 0.025$ & $-47.4\%$ \\
 & Last-4 (4/8)   & $\mathbf{0.256 \pm 0.014}$ & $0.497 \pm 0.033$ & $\mathbf{-48.4\%}$ \\
\bottomrule
\end{tabular}
\end{table}

{\looseness=-1 \textbf{Why does standard MoE routing fail?} The standard alternative -- AdaMix~\citep{wang2022adamix}, a hidden-state softmax router over the same five experts following the established sparsely-gated MoE recipe~\citep{fedus2022switch} -- collapses to a single expert whenever backbone layers are unfrozen (Table~\ref{tab:adamix}, Appendix~\ref{app:adamix_details}): entropy reaches $0.000\pm0.000$ in 7/12 unfrozen configurations; the remaining 5 stay below $0.55$ (max ${\approx}1.61$) with $\sigma$ up to $0.71$. The high seed-variance distinguishes \emph{stochastic} collapse (router commits to one expert per run, but to different experts across seeds) from \emph{stable} partial routing (which would yield low $\sigma$ at intermediate entropy); the rescue-baseline sweep below confirms no optimization-side intervention closes either failure mode. Disabling RevIN inside MOMENT recovers entropy to $0.66$--$1.32$ (BatchNorm1d/GroupNorm collapse identically, Appendix~\ref{app:norm_generalization}). The contrast is sharpest cross-backbone (Figure~\ref{fig:causal_contrast}): under identical last-4 unfreezing on identical data, MOMENT entropy crashes to $0.000$ in ${\sim}40$ steps while Timer-XL (no RevIN) stays at $1.60$ for $400$ steps.}

\begin{table}[!htbp]
\centering
\caption{\textbf{AdaMix routing collapse} (entropy max ${\approx}1.609$). Under unfreezing, entropy ${\to}0.000$. \colorbox{green!8}{Green}: disabling RevIN inside MOMENT recovers entropy ($0.66$--$1.32$), proving RevIN is the causal mechanism. Extended grid in Appendix~\ref{app:extended_rrmoa}.}
\label{tab:adamix}
\small
\begin{tabular}{@{}llcc@{}}
\toprule
Dataset & Freeze level & MSE (mean$\pm$std) & Routing entropy (mean$\pm$std) \\
\midrule
\multirow{4}{*}{ETTh1}
 & Frozen  & $1.105 \pm 0.026$ & $0.629 \pm 0.436$ \\
 & Last-2  & $1.153 \pm 0.000$ & $\mathbf{0.000 \pm 0.000}$ \\
 & Last-4  & $1.154 \pm 0.001$ & $\mathbf{0.000 \pm 0.000}$ \\
 & \cellcolor{green!8} Last-4 (no RevIN) & \cellcolor{green!8} $0.575 \pm 0.075$ & \cellcolor{green!8} $\mathbf{1.315 \pm 0.148}$ \\
\midrule
\multirow{4}{*}{ETTm1}
 & Frozen  & $1.008 \pm 0.012$ & $0.487 \pm 0.371$ \\
 & Last-2  & $1.061 \pm 0.088$ & $0.218 \pm 0.309$ \\
 & Last-4  & $1.123 \pm 0.000$ & $\mathbf{0.000 \pm 0.000}$ \\
 & \cellcolor{green!8} Last-4 (no RevIN) & \cellcolor{green!8} $0.426 \pm 0.072$ & \cellcolor{green!8} $\mathbf{1.103 \pm 0.376}$ \\
\midrule
\multirow{4}{*}{Weather}
 & Frozen  & $0.459 \pm 0.017$ & $0.509 \pm 0.307$ \\
 & Last-2  & $0.607 \pm 0.002$ & $\mathbf{0.000 \pm 0.000}$ \\
 & Last-4  & $0.607 \pm 0.002$ & $\mathbf{0.000 \pm 0.000}$ \\
 & \cellcolor{green!8} Last-4 (no RevIN) & \cellcolor{green!8} $0.232 \pm 0.045$ & \cellcolor{green!8} $\mathbf{0.661 \pm 0.501}$ \\
\bottomrule
\end{tabular}
\end{table}

\begin{figure}[!htbp]
\centering
\includegraphics[width=0.50\columnwidth]{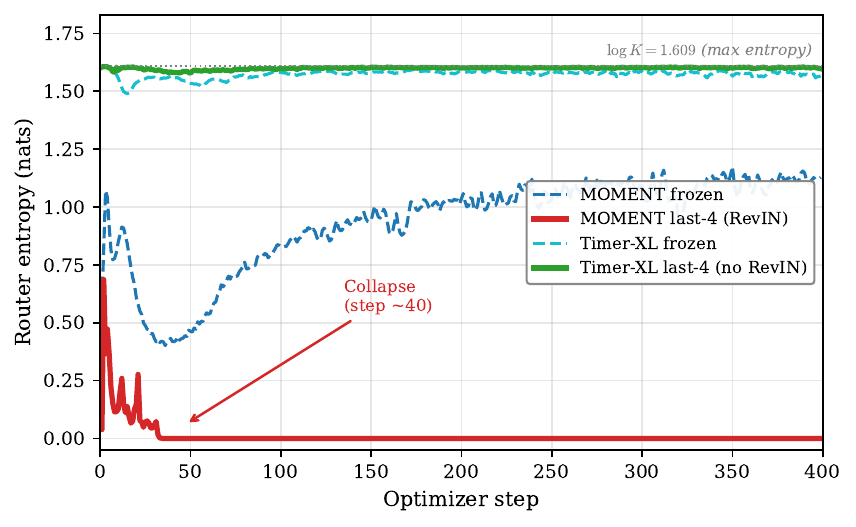}
\vspace{-6pt}
\caption{\textbf{Causal contrast: MOMENT vs Timer-XL} (ETTh1, last-4 unfreezing). Solid: unfrozen; MOMENT entropy collapses $0.51{\to}0.000$ in ${\sim}40$ steps; Timer-XL holds ${\approx}1.60$. Dotted: $\log K{=}1.609$ ceiling. Per-step trajectories in Figure~\ref{fig:trajectory} (App.~\ref{app:adamix_details}).}
\label{fig:causal_contrast}
\end{figure}

In contrast, \textbf{RR-MoA} maintains healthy entropy ($0.93$--$1.57$) and wins \textbf{54/54} configurations with $26$--$79\%$ MSE improvements over the best fixed adapter. The 54/54 covers the six primary-benchmark datasets (ETTh1/h2, ETTm1/m2, Weather, Electricity); on Traffic ($R{=}0.14$, defined in \S\ref{sec:diagnosis}), held out as a low-signal-ratio control, RR-MoA correctly does not improve, consistent with Observation~\ref{obs:routing_loss}'s low-$R$ prediction (Figure~\ref{fig:signal_ratio}).

\subsubsection*{The Frozen Paradox: full fine-tuning loses by $12$--$79\%$}

{\looseness=-1 Within a fixed MoE architecture, does freezing match or beat fine-tuning (\emph{RQ1})? Frozen RR-MoA beats \emph{full fine-tuning} (all 8 blocks, best of 5 heads $\times$ 2 LRs) by $12$--$79\%$ on all 6 LTSF datasets (Figure~\ref{fig:frozen_paradox}). The headline gap decomposes into \emph{(M1)} an MoE architectural advantage ($26$--$79\%$, Table~\ref{tab:rrmoa}) and \emph{(M2)} a smaller within-architecture freeze effect ($\leq 13\%$; frozen best on $4/6$, light unfreezing wins by ${\leq}13\%$ on the other two; SR-MoA replicates $4/6$, Table~\ref{tab:self_routed}). Proposition~\ref{prop:frozen} attributes M2 to gradient co-adaptation: a Router-Detached Gradient Flow ablation degrades MSE by $+35$--$100\%$ (Appendix~\ref{app:routing_ablations}), and disabling RevIN recovers full-FT by $54$--$61\%$ (Appendix~\ref{app:extended_ft}), pinning the mechanism on normalization. A $90$-config tune (LR $10^{-5}$--$10^{-6}$, $100$ epochs, cosine+warmup, layerwise decay) still loses by $51$--$71\%$.}

\begin{proposition}[Gradient Co-Adaptation Drives Entropy Collapse]
\label{prop:frozen}
Consider a $2$-expert MoE with backbone $\mathbf{h}{=}A\mathbf{x}$, router $p{=}\sigma(\mathbf{w}^\top\mathbf{h})$, loss $\ell{=}p\,\ell_1{+}(1{-}p)\,\ell_2$. Under gradient descent on $(A,\mathbf{w})$: \textup{(i)} the dominant expert's loss decreases faster ($\dot{\ell}_k \propto {-}p_k\|\nabla_\mathbf{h}\ell_k\|^2$), and $\dot{p}$ amplifies whichever expert currently has lower loss, creating self-reinforcing collapse toward $p{\in}\{0,1\}$; \textup{(ii)} freezing $A$ removes the co-adaptation feedback through the backbone (full non-collapse additionally requires $(\ell_1{-}\ell_2)$ to flip sign across input regimes, i.e., expert specialization, which is satisfied by the topologically distinct RR-MoA pool). Proof in Appendix~\ref{app:proofs}.
\end{proposition}

{\looseness=-1 \textit{Remark.} Proposition~\ref{prop:frozen} covers linear $A\mathbf{x}$; Figure~\ref{fig:trajectory}b confirms it in the full 8-block Transformer within ${\sim}50$ steps. Independent ensembles are $37$--$46\%$ worse (Table~\ref{tab:baselines}), so learned routing -- not mere diversity -- drives the gains. \S\ref{sec:mechanism} answered \emph{what} the router needs to read with Observation~\ref{obs:routing_loss} and the predictor $R(\mathcal{D})$; we now turn to the \emph{mechanism (RQ2)}.}

\subsubsection*{$R(\mathcal{D})$ makes three predictions, and all three hold}

Observation~\ref{obs:routing_loss} and $R(\mathcal{D})$ make three testable predictions: the bound is empirically tight, $R(\mathcal{D})$ predicts raw-routing benefit across datasets, and the predictor respects a falsifiable boundary distinct from a generic-SNR explanation. We verify each in turn (\emph{RQ2}).

\textbf{$R(\mathcal{D})$ predicts raw-routing benefit and respects a falsifiable boundary.} The bound is empirically tight (Appendix~\ref{app:mi_tightness}) and the learnable mixing coefficient satisfies $\alpha{<}0.5$ on $30/30$ runs (Appendix~\ref{app:learnable_alpha}). $R(\mathcal{D})$ predicts raw-routing benefit at Spearman $\rho{=}{-}0.88$ ($p{<}0.002$, $n{=}9$; Figure~\ref{fig:signal_ratio}), whereas the conditional entropy $H(E\mid S)$ does not ($\rho{=}{-}0.12$); Traffic ($R{=}0.14$) is a boundary, Solar ($R{=}0.06$) the one outlier (Appendix~\ref{app:routing_ablations}). The discriminating power lies in the boundary case: the most likely competing explanation is that $R(\mathcal{D})$ proxies generic ``hard dataset'' signal-to-noise, but Traffic has high overall predictability (DLinear MSE $0.41$) yet low $R$, and RR-MoA correctly does not improve there, falsifying the SNR alternative. $R(\mathcal{D})$ is therefore actionable pre-training: practitioners can compute it from raw dataset statistics to predict whether raw routing will help on a new dataset.

{\looseness=-1 \textbf{Per-window statistics, not raw access, drive routing.} \emph{(i)} Swapping $\mathbf{x}_\text{raw}$ for $\text{RevIN}(\mathbf{x}_\text{raw})$ at the router degrades MSE by $60$--$88\%$ while entropy stays high (Table~\ref{tab:router_input}): the router needs the stripped $(M,\Sigma)$ \emph{content}, not a non-RevIN bypass. \emph{(ii)} Interpolating $\mathbf{x}_\alpha = (1{-}\alpha)\mathbf{x}_\text{raw} + \alpha\,\text{RevIN}(\mathbf{x}_\text{raw})$ phase-transitions only at $\alpha{\to}1$ (Figure~\ref{fig:dose_response}; $150$ runs; SNR-onset $\alpha^*{\approx}0.40$, Appendix~\ref{app:prop2_depth}); partial stripping is tolerable. \emph{(iii)} Temporal-shuffling the window leaves routing unchanged (Appendix~\ref{app:routing_ablations}): the active feature is the order-invariant $(M,\Sigma)$.}

\begin{figure}[!htbp]
\centering
\begin{minipage}[t]{0.58\columnwidth}
\centering
\begin{tikzpicture}
\begin{axis}[
    ybar,
    width=\textwidth,
    height=3.6cm,
    bar width=5pt,
    ylabel={Test MSE (H=96)},
    ylabel style={font=\footnotesize},
    symbolic x coords={ETTh1, ETTm1, Weath, ETTh2, ETTm2, Elec},
    xtick=data,
    x tick label style={font=\footnotesize},
    y tick label style={font=\footnotesize},
    legend style={at={(0.5,1.06)}, anchor=south, font=\scriptsize, draw=none, fill=white, fill opacity=0.92, legend columns=2, column sep=0.8em},
    ymin=0, ymax=3.35,
    enlarge x limits=0.12,
    error bars/y dir=both,
    error bars/y explicit,
    error bars/error bar style={line width=0.4pt, black!60},
]
\addplot[fill=red!45] coordinates {(ETTh1,1.060) +- (0,0.059) (ETTm1,0.871) +- (0,0.029) (Weath,0.468) +- (0,0.009) (ETTh2,2.645) +- (0,0.049) (ETTm2,2.559) +- (0,0.146) (Elec,0.457) +- (0,0.006) };
\addplot[fill=blue!55] coordinates {(ETTh1,0.680) +- (0,0.027) (ETTm1,0.564) +- (0,0.059) (Weath,0.276) +- (0,0.017) (ETTh2,0.778) +- (0,0.095) (ETTm2,0.538) +- (0,0.069) (Elec,0.402) +- (0,0.062) };
\node[font=\scriptsize, text=red!70!black] at (axis cs:ETTh1,1.25) {$-36\%$};
\node[font=\scriptsize, text=red!70!black] at (axis cs:ETTm1,1.05) {$-35\%$};
\node[font=\scriptsize, text=red!70!black] at (axis cs:Weath,0.62) {$-41\%$};
\node[font=\scriptsize, text=red!70!black] at (axis cs:ETTh2,2.85) {$-71\%$};
\node[font=\scriptsize, text=red!70!black] at (axis cs:ETTm2,2.78) {$-79\%$};
\node[font=\scriptsize, text=red!70!black] at (axis cs:Elec,0.61) {$-12\%$};
\legend{Full FT (unfrozen), Frozen RR-MoA}
\end{axis}
\end{tikzpicture}
\end{minipage}%
\hfill
\begin{minipage}[t]{0.40\columnwidth}
\centering
\begin{tikzpicture}
\begin{axis}[
    width=\textwidth, height=3.8cm,
    xlabel={Signal ratio $R(\mathcal{D})$},
    ylabel={$\Delta\%$},
    xlabel style={font=\footnotesize}, ylabel style={font=\footnotesize},
    x tick label style={font=\footnotesize}, y tick label style={font=\footnotesize},
    xmin=0, xmax=2.5, ymin=-85, ymax=10,
    grid=major, grid style={gray!20},
    clip=false,
]
\addplot[only marks, mark=*, mark size=2.5pt, blue!80] coordinates {
    (0.242, -26.8) (0.566, -43.2) (0.671, -51.1) (0.703, -44.6) (1.163, -71.0) (1.241, -77.2) (2.174, -66.5)
};
\addplot[only marks, mark=diamond*, mark size=3.2pt, orange!85!black] coordinates {
    (0.135, +2.9) (0.061, -32.9)
};
\node[font=\scriptsize, anchor=west] at (axis cs:0.08, -32.9) {Sol};
\node[font=\scriptsize, anchor=west] at (axis cs:0.17, +2.9) {Traf};
\node[font=\scriptsize, anchor=west] at (axis cs:0.27, -26.8) {Elec};
\node[font=\scriptsize, anchor=south] at (axis cs:0.566, -39.5) {h1};
\node[font=\scriptsize, anchor=west] at (axis cs:0.74, -51.1) {m1};
\node[font=\scriptsize, anchor=east] at (axis cs:0.66, -44.6) {W};
\node[font=\scriptsize, anchor=west] at (axis cs:1.19, -71.0) {h2};
\node[font=\scriptsize, anchor=west] at (axis cs:1.27, -77.2) {m2};
\node[font=\scriptsize, anchor=west] at (axis cs:2.20, -66.5) {Exch};
\addplot[black, thin, dotted] coordinates {(0,0) (2.5,0)};
\addplot[red, thick, dashed, domain=0.05:2.4] {-30.5*x - 7.5};
\node[font=\scriptsize, text=red!70, anchor=north east] at (axis cs:2.45, -10) {$\rho{=}{-}0.88$};
\end{axis}
\end{tikzpicture}
\end{minipage}
\caption{\textbf{Left: Frozen Paradox}, frozen RR-MoA beats full FT by $12$--$79\%$ across the six LTSF datasets at $H{=}96$. \textbf{Right: Observation~\ref{obs:routing_loss} validated}: $R(\mathcal{D})$ predicts raw-routing benefit at Spearman $\rho{=}{-}0.88$ ($p{<}0.002$, $n{=}9$). Traffic and Solar are boundary cases.}
\label{fig:frozen_paradox}
\label{fig:signal_ratio}
\end{figure}
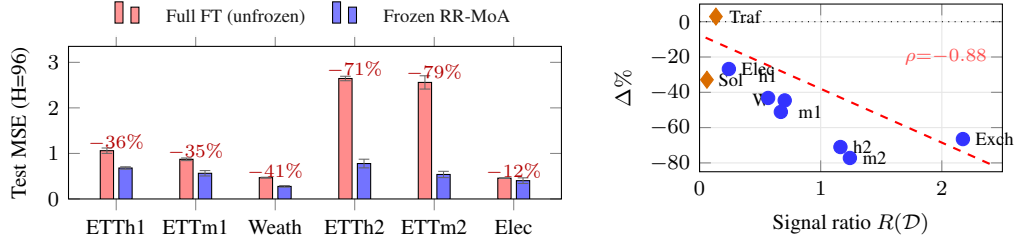

\subsubsection*{Optimization-side rescues fail; raw-input fixes transfer}

Must the fix be input-side (\emph{RQ2})? If the input distribution is the bottleneck, optimization-side rescues should not recover the gap, and the raw-input fix should transfer beyond RR-MoA's specific router. We test each.

\textbf{Rescue-baseline sweep.} The mechanism predicts that interventions targeting optimization will fail because the router's input distribution is the bottleneck, not its objective. We test this directly by sweeping five standard MoE rescue families (Switch load-balance~\citep{fedus2022switch}, ST-MoE z-loss~\citep{zoph2022stmoe}, entropy regularization, ReMoE ReLU routing~\citep{wang2025remoe}, expert-choice routing~\citep{zhou2022expertchoice}) over $12$ configurations $\times$ 6 datasets $\times$ 2 freeze levels $\times$ 5 seeds $= 720$ run-cells (Table~\ref{tab:rescue}, Appendix~\ref{app:rescue_details}). The best rescue (z-loss $c_z{=}0.01$) reduces MSE by only $10.9\%$, $2.7{\times}$ worse than RR-MoA; excessive balancing ($\alpha{=}10.0$) hurts by $+67.7\%$. Consistent with Observation~\ref{obs:routing_loss}, Part~(ii): when routing information is structurally absent, no tested rescue compensates.

\textbf{The principle is architecture-agnostic.} We design three router-level interventions, each varying one axis (router architecture, router input, expert architecture). \textbf{SR-MoA} (Figure~\ref{fig:srmoa_arch}; cf.~\citealp{lv2025aoe,liu2026routingfree,cheng2025ermoe}) drops the external router for per-expert sigmoid gates $\sigma_k(\mathbf{X}_\mathrm{raw})$ and \emph{outperforms} RR-MoA by $13$--$42\%$ on all 6 datasets ($252$ runs; Table~\ref{tab:self_routed}); on hidden states the same architecture degrades by $+87$--$768\%$. \textbf{AdaMix-Raw} (App.~\ref{app:adamix_details}) swaps the router input only and recovers entropy to $1.41$--$1.57$ across all 12 cells. \textbf{Residual-IA\textsuperscript{+}} (App.~\ref{app:gate_pathology}) extends the principle to the expert level. With the RevIN ablation, MOMENT-vs-Timer-XL contrast, BatchNorm/GroupNorm swaps, rescue sweep, vision control, and $(\mu,\sigma)$ re-injection, these form the eight causal controls (Table~\ref{tab:causal_controls}); the only mechanism consistent with all eight is that \emph{instance normalization strips the per-window $(\mu,\sigma)$ signal the router needs}.

\begin{table}[!htbp]
\centering
\caption{\textbf{Eight causal controls.} Each rules out a competing explanation; the only mechanism consistent with all 8 is that \emph{instance normalization strips the per-window $(\mu,\sigma)$ signal the router needs}.}
\label{tab:causal_controls}
\footnotesize
\setlength{\tabcolsep}{4pt}
\renewcommand{\arraystretch}{0.88}
\begin{tabular}{@{}cl >{\raggedright\arraybackslash}p{3.3cm} >{\raggedright\arraybackslash}p{5.95cm}@{}}
\toprule
\# & \textbf{Control} & \textbf{Hypothesis (rejected)} & \textbf{Observed result} $\Rightarrow$ \textbf{verdict} \\
\midrule
C1 & RevIN ablation & Collapse is intrinsic to MoE itself, not caused by normalization & Entropy $0.000 \to 0.66$--$1.32$ when RevIN disabled (Table~\ref{tab:adamix}) {\small$\Rightarrow$} RevIN is \emph{necessary} for collapse \\
C2 & MOMENT vs Timer-XL & Architecture, not normalization, causes collapse & Opposite trajectories with identical code, only norm differs (Fig.~\ref{fig:causal_contrast}) {\small$\Rightarrow$} normalization is \emph{sufficient} \\
C3 & BatchNorm1d / GroupNorm & Collapse is unique to RevIN; other per-window normalizers are immune & All three per-window normalizers collapse identically (App.~\ref{app:norm_generalization}) {\small$\Rightarrow$} any norm that strips $(\mu,\sigma)$ triggers collapse \\
C4 & 5-family rescue sweep & Optimization problem (load-balance, z-loss, ReLU, expert-choice) & 0/12 configs recover; best is $2.7{\times}$ worse than RR-MoA (Table~\ref{tab:rescue}) {\small$\Rightarrow$} not in the optimizer \\
C5 & AdaMix with raw router & Router architecture is the bottleneck & Same router, swap input only: entropy $\to 1.41$--$1.57$; MSE $-41$--$88\%$ (App.~\ref{app:adamix_details}) {\small$\Rightarrow$} the \emph{input} matters, not the router \\
C6 & SR-MoA raw vs hidden routing & The fix only works for RR-MoA's specific Conv1d router design & Different router (per-expert sigmoid); hidden-routed $+87$--$768\%$ worse (Table~\ref{tab:self_routed}) {\small$\Rightarrow$} raw-input principle generalizes \\
C7 & Vision ViT + InstanceNorm & Collapse only happens in time series, not other modalities & ResNet/ViT MoE on CIFAR-10 collapses identically (App.~\ref{app:vision_moe}) {\small$\Rightarrow$} mechanism is modality-general \\
C8 & $(\mu,\sigma)$ re-injection & Just appending stripped stats back fixes it & Re-injected hidden states still $+42$--$239\%$ worse, 11/18 still collapsed (App.~\ref{app:reinjection}) {\small$\Rightarrow$} only \emph{raw} pre-normalization input restores routing \\
\bottomrule
\end{tabular}
\end{table}

\subsubsection*{RR-MoA leads the baseline suite and transfers across backbones, horizons, and tasks}

{\looseness=-1 We test \emph{method-class generality} (\emph{RQ3}): RR-MoA beats every method in the seven-baseline suite at $p{<}0.001$ (Table~\ref{tab:baselines}), driven by two complementary gains: residual shape heterogeneity rewards diverse experts on RevIN-normalized hidden states ($26$--$79\%$ over the best fixed adapter, Table~\ref{tab:rrmoa}), while the signal-ratio $\rho{=}{-}0.88$ localizes \emph{dispatch} information in the stripped statistics, explaining both raw-routing benefit and Traffic's null. The Pure Raw-MLP MoE ablation (App.~\ref{app:raw_mlp_moe}) separates ``MoE with raw router'' from ``ensemble with raw input''.}

\begin{table}[!htbp]
\centering
\setlength{\tabcolsep}{3pt}
\renewcommand{\arraystretch}{0.80}
\caption{\textbf{Baseline comparison} (MOMENT-small, Top-2, 5 seeds, H=96). Above double rule: frozen; full FT unfreezes all 8 blocks. Wilcoxon, Bonferroni-corrected. Grey rows: from-scratch calibration anchors.}
\label{tab:baselines}
\footnotesize
\begin{tabular}{@{}lcccc@{}}
\toprule
Method & ETTh1 & ETTm1 & Weather & $p$-value \\
\midrule
Best fixed adapter & $1.254 \pm 0.026$ & $1.148 \pm 0.035$ & $0.528 \pm 0.018$ & ${<}0.001^{***}$ \\
Best LoRA (108-run sweep) & $1.154 \pm 0.000$ & $0.956 \pm 0.072$ & $0.600 \pm 0.052$ & ${<}0.001^{***}$ \\
TRACE (frozen) & $1.195 \pm 0.026$ & $1.014 \pm 0.073$ & $0.457 \pm 0.015$ & ${<}0.001^{***}$ \\
Ind.\ Ensemble (5 experts) & $1.081 \pm 0.021$ & $1.049 \pm 0.032$ & $0.509 \pm 0.024$ & ${<}0.001^{***}$ \\
AdaMix (frozen) & $1.109 \pm 0.031$ & $1.008 \pm 0.012$ & $0.459 \pm 0.017$ & ${<}0.001^{***}$ \\
\midrule
Full fine-tuning (all unfrozen) & $1.060 \pm 0.059$ & $0.871 \pm 0.029$ & $0.468 \pm 0.009$ & ${<}0.001^{***}$ \\
\midrule
\textbf{RR-MoA (Top-2, frozen)} & $\mathbf{0.680 \pm 0.027}$ & $\mathbf{0.564 \pm 0.059}$ & $\mathbf{0.276 \pm 0.017}$ & --- \\
\midrule
\rowcolor{gray!8} DLinear (from scratch) & $0.416 \pm 0.002$ & $0.322 \pm 0.004$ & $0.208 \pm 0.003$ & --- \\
\rowcolor{gray!8} PatchTST/64$^\dagger$~\citep{nie2023patchtst} & $0.370$ & $0.274$ & $0.149$ & --- \\
\rowcolor{gray!8} iTransformer$^\ddagger$~\citep{liu2024itransformer} & $0.386$ & $0.334$ & $0.174$ & --- \\
\bottomrule
\multicolumn{5}{l}{\scriptsize $^\dagger$Trained from scratch, $L{=}512$; $^\ddagger$trained from scratch, $L{=}96$. Included for absolute MSE context.}
\end{tabular}
\end{table}

{\looseness=-1 \textbf{The diagnosis transfers across backbones, horizons, and tasks, and predicts where it should not} (\emph{RQ3}).\label{sec:cross_backbone} RR-MoA wins on every instance-normalized backbone (MOMENT $26$--$79\%$, Moirai $12$--$29\%$, Moirai-MoE~\citep{liu2025moiraimoe} $37$--$90\%$); LayerNorm-only Chronos and Timer-XL~\citep{liu2025timerxl} show no collapse, as predicted. Beyond $H{=}96$, RR-MoA wins $12/12$ multi-horizon vs.\ best fixed adapter and $7/8$ imputation cells (Tables~\ref{tab:cross_backbone}--\ref{tab:horizon}, App.~\ref{app:imputation}); \textbf{Residual-IA\textsuperscript{+}} matches or beats DLinear on $5/6$ datasets at $H{=}96$ ($6/6$ at $H{=}192$) and $107/123$ cells across the $6{\times}6{\times}4$ grid (App.~\ref{app:gate_pathology}). The diagnostic predicts failures: Traffic ($R{=}0.14$) is statistically indistinguishable from DLinear; $R(\mathcal{D})$ is computable pre-training, giving an \emph{a priori} go/no-go test.}

\section{Conclusion and Limitations}

{\looseness=-1 \textbf{Summary.} We identified \emph{normalization-induced routing collapse}: $(\mu,\sigma)$ stripping that handles distribution shift also removes the signal a router needs. Observation~\ref{obs:routing_loss} formalizes the loss, $R(\mathcal{D})$ predicts it ($\rho{=}{-}0.88$), eight causal controls rule out optimization/geometry/paradigm alternatives, and three interventions confirm the diagnosis: RR-MoA ($54/54$), SR-MoA ($+13$--$42\%$), Residual-IA\textsuperscript{+} ($107/123$ across 6 backbones $\times$ 6 datasets $\times$ 4 horizons). \emph{TSFM users:} frozen backbone with raw-input routing is the strongest tested configuration on every $(\mu,\sigma)$-stripping backbone; full fine-tuning loses by $12$--$79\%$ (App.~\ref{app:extended_ft}); $R(\mathcal{D})$ is a pre-training go/no-go test (Traffic, $R{=}0.14$: predicted no improvement). \emph{MoE researchers:} eleven optimization-side rescues recover ${\leq}10.9\%$ MSE; the input-side fix recovers $54$--$80\%$. The prediction extends to FAN~\citep{ye2024fan} and any normalizer that strips routing signal, and to vision (App.~\ref{app:vision_moe}). \textbf{Limitations:} specific to $(\mu,\sigma)$-stripping normalizers; classification, anomaly detection, horizons ${>}720$, and learned expert pools remain open.\label{page:end_main}}
\bibliographystyle{plainnat}

\newpage
\appendix
\renewcommand{\thetable}{\Alph{section}.\arabic{table}}
\renewcommand{\thefigure}{\Alph{section}.\arabic{figure}}
\makeatletter
\@addtoreset{table}{section}
\@addtoreset{figure}{section}
\makeatother
\setcounter{table}{0}
\setcounter{figure}{0}

\section{Experimental Setup Details}
\label{app:setup}

\paragraph{Notation.} Table~\ref{tab:notation} consolidates symbols used across \S\ref{sec:method}--\S\ref{sec:main_results}, Observation~\ref{obs:routing_loss}, and Appendix~\ref{app:proofs}.

\begin{table}[!htbp]
\centering
\caption{\textbf{Notation.} Symbols used throughout the paper.}
\label{tab:notation}
\small
\setlength{\tabcolsep}{6pt}
\begin{tabular}{@{}ll@{}}
\toprule
\textbf{Symbol} & \textbf{Meaning} \\
\midrule
\multicolumn{2}{@{}l}{\emph{Data and task}} \\
$\mathbf{x}\in\mathbb{R}^{T\times c}$ & Input window: $T{=}512$ time steps, $c$ channels \\
$\mathbf{y}\in\mathbb{R}^{h\times c}$ & Forecast target at horizon $h\in\{96,192,336,720\}$ \\
$\hat{\mathbf{y}}$ & Model prediction \\
$M_i, \Sigma_i, \mathbf{S}_i$ & Per-window mean, scale, RevIN-normalized shape ($\mathbf{S}_i{=}(\mathbf{x}_i{-}M_i)/\Sigma_i$) \\
\midrule
\multicolumn{2}{@{}l}{\emph{Backbone and adapter}} \\
$f_\theta:\mathbb{R}^{T\times c}\to\mathbb{R}^{P\times d}$ & Pretrained TSFM feature extractor; $\theta$ frozen ($\nabla_\theta\equiv 0$) \\
$\mathbf{H}{=}f_\theta(\mathbf{x})$ & Backbone hidden states ($P$ patches of dim $d$) \\
$g_\phi:\mathbb{R}^{P\times d}\to\mathbb{R}^{h\times c}$ & Adapter; $|\phi|\leq 500$K trainable params \\
\midrule
\multicolumn{2}{@{}l}{\emph{Mixture and routing}} \\
$K$ & Total expert pool size (we use $K{=}5$) \\
$k$ & Top-$k$ active experts per window (default $k{=}2$) \\
$\mathrm{Expert}_j$ & $j$-th expert head ($j{=}1,\ldots,K$); pool: mean/last/max/attention/conv1d \\
$G_\psi:\mathbb{R}^{T\times c}\to\mathbb{R}^K$ & Router; $\psi$ are 853 trainable params \\
$\widetilde{w}\in\Delta^{K-1}$ & Sparse routing weights, $\|\widetilde{w}\|_0{=}k$ \\
$\mathcal{T}\subset\{1,\ldots,K\}$ & Top-$k$ active expert indices per window \\
$E\in\{1,\ldots,K\}$ & Random variable for expert assignment (Observation~\ref{obs:routing_loss}) \\
\midrule
\multicolumn{2}{@{}l}{\emph{Diagnostics}} \\
$R(\mathcal{D})$ & Signal ratio $[\mathrm{Var}(M){+}\mathrm{Var}(\Sigma)]/\overline{\mathrm{Var}(S)}$ predicting raw-routing benefit \\
$\alpha\in[0,1]$ & Interpolation coefficient: $\mathbf{x}_\alpha{=}(1{-}\alpha)\mathbf{x}_\text{raw}{+}\alpha\,\text{RevIN}(\mathbf{x}_\text{raw})$ \\
$\varepsilon$ & Residual statistic-shape leakage $I(M,\Sigma;S)$ in Observation~\ref{obs:routing_loss}, Part~(ii) \\
$H(E)$ & Routing entropy in nats; collapse$\Rightarrow H(E){\to}0$, ceiling $\log K{\approx}1.609$ \\
\bottomrule
\end{tabular}
\end{table}

\paragraph{Backbone specifications.} MOMENT-small ($d{=}512$, 8 encoder blocks, internal RevIN with learnable affine); MOMENT-large ($d{=}1024$, 24 blocks, internal RevIN with learnable affine); Moirai ($d{=}384$, 6 blocks, RMSNorm internally with non-learnable per-instance I/O scaling, no learnable-affine norm inside the encoder); Moirai-MoE ($d{=}384$, 6 blocks, RMSNorm internally with sparse internal experts and non-learnable per-instance I/O scaling); Chronos ($d{=}512$, 6 encoder blocks, T5 RMSNorm internally, no instance normalization); Timer-XL ($d{=}1024$, 8 blocks, LayerNorm internally, no instance normalization). Throughout, ``full fine-tuning'' means \emph{all} encoder blocks are unfrozen (8/8 for MOMENT-small, 24/24 for MOMENT-large, 6/6 for Moirai/Moirai-MoE/Chronos, 8/8 for Timer-XL).

\paragraph{Training hyperparameters.} Adam optimizer, learning rate $10^{-3}$, MSE loss, batch size $128$, $15$ epochs. Final reported numbers are 15-epoch test MSE. Core multi-seed grids use seeds $\{42, 43, 44, 45, 46\}$ ($n{=}5$); ablations use seeds $\{42, 43, 44\}$ ($n{=}3$). Channel-wise StandardScaler normalization~\citep{zhou2021informer, nie2023patchtst} is applied to all inputs.

\paragraph{Per-baseline configurations.}
\begin{itemize}[leftmargin=1.2em, itemsep=0.15em, topsep=0.2em]
\item \textbf{Fixed adapters}: linear, attention-pool, Conv1d-pool single-head topologies with $\leq 500{,}000$ trainable parameters.
\item \textbf{LoRA}~\citep{hu2022lora}: 108-run sweep (12 configurations $\times$ 3 seeds $\times$ 3 datasets) over rank, target projections, and forecast head; full table in Appendix~\ref{app:lora_sweep}.
\item \textbf{TRACE}~\citep{li2025trace}: importance-based LoRA selection with default settings.
\item \textbf{Independent ensemble}: 5 experts with the canonical pool topologies, trained independently with no router and averaged at inference.
\item \textbf{AdaMix}~\citep{wang2022adamix}: $K{=}5$ canonical heads with hidden-state softmax router; load-balance coefficient $0.01$. Full implementation details in Appendix~\ref{app:adamix_details}.
\item \textbf{Full fine-tuning}: all encoder blocks unfrozen; best of $5$ head topologies $\times$ $2$ learning rates ($10^{-3}$, $10^{-4}$). Extended sweeps in Appendix~\ref{app:extended_ft} cover $50$-epoch cosine schedules at $10^{-5}$.
\item \textbf{DLinear}~\citep{zeng2023dlinear}: from-scratch supervised; we use the simplified single-Linear variant ($\text{Linear}(L{=}512 \to H{=}96)$, $49$K params), without the trend-seasonal decomposition, as the calibration anchor. Calibration values $0.416 / 0.322 / 0.208$ on ETTh1/ETTm1/Weather (gap to RR-MoA visible in Tables~\ref{tab:rrmoa},~\ref{tab:baselines}).
\end{itemize}

\paragraph{RR-MoA pseudocode.}
\begin{algorithm}[h]
\caption{RR-MoA: Raw-Routed Mixture of Adapters with Top-$k$ sparse routing}
\label{alg:rrmoa}
\begin{algorithmic}[1]
\REQUIRE TSFM $f_\theta$ (default: strictly frozen; optional: last-$\ell$ blocks unfrozen), raw input $\mathbf{X}_\text{raw}\in\mathbb{R}^{B\times L}$, $K$ expert adapters, sparsity $k\leq K$
\STATE $\mathbf{H}\leftarrow f_\theta(\mathbf{X}_\text{raw})$ \hfill \COMMENT{Frozen default: no gradients through $\theta$; for last-$\ell$/all unfreezing (Table~\ref{tab:rrmoa}), the unfrozen blocks receive gradients normally}
\STATE $\mathbf{g}\leftarrow G_\psi(\mathbf{X}_\text{raw})$ \hfill \COMMENT{Raw-signal gate logits, $\mathbf{g}\in\mathbb{R}^{B\times K}$}
\STATE $\mathcal{T}\leftarrow \mathrm{TopK}(\mathbf{g},\,k)$ \hfill \COMMENT{Per-sample indices of the $k$ selected experts}
\STATE $\widetilde{\mathbf{w}}\leftarrow \mathrm{softmax}_{\mathcal{T}}(\mathbf{g})$ \hfill \COMMENT{Renormalized over selected experts only}
\STATE \textbf{for each}\ $j \in \mathcal{T}$\ \textbf{do}
\STATE \quad $\mathbf{o}_j\leftarrow \text{Expert}_j(\mathbf{H})$ \hfill \COMMENT{Only the $k$ selected experts execute}
\STATE \textbf{end for}
\STATE $\mathbf{Y}\leftarrow \sum_{j\in\mathcal{T}}\widetilde{w}_j\cdot \mathbf{o}_j$
\RETURN $\mathbf{Y}$
\end{algorithmic}
\end{algorithm}

\section{Deployment-Regime Motivation}
\label{app:deployment}

We motivate the frozen-backbone paradigm with three deployment regimes where full fine-tuning is infeasible:

\textbf{(i) Multi-tenant cloud serving.} A single TSFM sits in GPU memory and serves many tenants via lightweight, hot-swappable per-tenant adapters in host RAM. Per-tenant DLinear requires collecting, storing, and periodically retraining on each tenant's historical data ($O(N)$ training pipelines for $N$ tenants), whereas the frozen backbone is loaded once and adapter weights are swapped from host RAM. The adapter approach also enables instant onboarding of new tenants via zero-shot transfer, a capability absent from per-tenant supervised baselines.

\textbf{(ii) On-device and edge deployment.} One backbone is shipped with the device image; per-site adapters are updated over low-bandwidth links. Collecting each site's training data centrally is often impossible for privacy or latency reasons.

\textbf{(iii) Cross-task reuse.} The same frozen backbone drives forecasting, imputation, and classification via different expert pools. We demonstrate the first two tasks explicitly in the main text.

\subsection{Inference Benchmark}
\label{app:benchmark}

With the deployment regimes specified above, we now quantify the inference-time cost of each architectural variant on a single A10G and contrast it with DLinear, the per-task baseline that those regimes rule out.

\begin{table}[!htbp]
\centering
\caption{\textbf{Inference latency and memory across all three architectural variants} (MOMENT-small, A10G GPU, batch size $128$, $200$ repeats with $20$-step CUDA-event warmup; synthetic random inputs at the standard LTSF batch shape, since latency for non-autoregressive deterministic forward passes is data-oblivious). All three architectural fixes (RR-MoA, SR-MoA, Residual-IA\textsuperscript{+}) add only single-digit milliseconds of adapter overhead on top of the backbone forward pass; the frozen-backbone fast-inference property is preserved by each.}
\label{tab:benchmark}
\small
\begin{tabular}{@{}lcccc@{}}
\toprule
Method & Latency (ms) & $\Delta$ vs.\ backbone & Peak GPU (MB) & Adapter params \\
\midrule
Backbone only                                & $52.89 \pm 0.36$ & ---            & $359$ & --- \\
Single adapter (Conv1dPool)                  & $53.57 \pm 0.37$ & $+1.3\%$       & $360$ & $268$K \\
\textbf{RR-MoA (Top-2, $K{=}5$)}             & $\mathbf{56.02 \pm 0.44}$ & $\mathbf{+5.9\%}$ & $\mathbf{362}$ & $\mathbf{426}$\textbf{K} \\
\textbf{SR-MoA (dense $K{=}5$)}              & $\mathbf{54.20 \pm 0.79}$ & $\mathbf{+2.5\%}$ & $\mathbf{364}$ & $\mathbf{466}$\textbf{K} \\
\textbf{Residual-IA\textsuperscript{+}}      & $\mathbf{54.53 \pm 0.43}$ & $\mathbf{+3.1\%}$ & $\mathbf{365}$ & $\mathbf{354}$\textbf{K} \\
\midrule
\rowcolor{gray!8} DLinear (no backbone)      & $0.04 \pm 0.00$ & ---            & $160$ & $49$K \\
\bottomrule
\end{tabular}
\end{table}

The RR-MoA routing overhead ($+3.1$\,ms) is dominated by the sparse expert selection loop; the router itself (Conv1d + linear, $853$ params) adds $<0.1$\,ms. SR-MoA dense and Residual-IA\textsuperscript{+} both run \emph{faster} than RR-MoA Top-2 on GPU because their dense per-expert execution parallelizes over the $K{=}5$ expert dimension without per-sample sparse-routing kernel launches. Notably, Residual-IA\textsuperscript{+} has \emph{fewer} adapter parameters than RR-MoA ($354$K vs $426$K) thanks to the shared NLinear raw branch.

\textbf{What DLinear wins.} Per call, DLinear is ${\sim}1400\times$ faster ($0.04$ vs $56$\,ms) and uses ${\sim}2\times$ less peak GPU memory; for $N$ independent tasks served in isolation, $N$ DLinears are also smaller in aggregate than one TSFM plus $N$ adapters. We do not claim otherwise.

\textbf{What our adapters deliver in practice.} The numbers above quantify the cost of accessing a regime DLinear cannot serve at all. (i) \emph{Adapter overhead is bounded.} All three of our architectures stay within $+5.9\%$ of the raw backbone latency and add at most $6$\,MB of peak GPU memory (Table~\ref{tab:benchmark}); the frozen backbone is the cost floor, and the routing/expert machinery does not blow it up. (ii) \emph{Hot-swap, not retrain.} A new tenant or task is served by loading $354$--$466$\,K adapter parameters from host RAM into the resident backbone in sub-millisecond time; the per-tenant DLinear alternative requires collecting that tenant's history and running a fresh supervised training pipeline before the first prediction. (iii) \emph{One model, many tasks.} The same frozen backbone drives forecasting, imputation, and classification through different expert pools (\S\ref{sec:intro}, Appendix~\ref{app:imputation}); $N$ tasks consume one GPU residency rather than $N$. (iv) \emph{Zero-shot onboarding.} Adapter weights trained on related tenants transfer at load time, so a new deployment can serve predictions immediately rather than waiting for the data-collection-plus-training cycle that supervised per-task models require. The latency and memory numbers above are the price of these capabilities, not a claim of dominance over DLinear in DLinear's native single-task setting.

\section{Multi-Horizon Evaluation}
\label{app:horizon}

We extend the main-text $H{=}96$ result to the standard LTSF horizon set $H{\in}\{96,192,336,720\}$ on ETTh1, ETTm1, and Weather (Table~\ref{tab:horizon}), confirming that RR-MoA's $26$--$79\%$ improvement over the best fixed adapter is not horizon-specific: it wins all $12$ (dataset, horizon) cells vs.\ best fixed adapter. The DLinear column is reported only as the supervised calibration anchor (RR-MoA loses to it across all 12 cells; closing that gap is the subject of Appendix~\ref{app:gap_closing}); the RR-MoA-vs-DLinear gap narrows monotonically with $H$ on ETTh1 and Weather, consistent with the frozen backbone's horizon-invariant representation becoming relatively more informative as the prediction window lengthens.

\begin{table}[!htbp]
\centering
\caption{\textbf{Multi-horizon RR-MoA vs.\ DLinear} (MOMENT-small, strictly frozen, Top-2, 5 seeds, test MSE mean$\pm$std). RR-MoA wins all 12 configurations vs.\ best fixed baseline across horizons $96$--$720$. DLinear column provides supervised calibration: the gap narrows from $+63\%$ to $+48\%$ on ETTh1 and from $+33\%$ to $+15\%$ on Weather as $H$ grows, consistent with the frozen backbone's horizon-invariant representation becoming relatively more informative at longer horizons.}
\label{tab:horizon}
\scriptsize
\begin{tabular}{@{}lccccccccc@{}}
\toprule
 & \multicolumn{3}{c}{ETTh1} & \multicolumn{3}{c}{ETTm1} & \multicolumn{3}{c}{Weather} \\
\cmidrule(lr){2-4} \cmidrule(lr){5-7} \cmidrule(lr){8-10}
$H$ & RR-MoA & Best fix & DLin & RR-MoA & Best fix & DLin & RR-MoA & Best fix & DLin \\
\midrule
96  & $0.680{\pm}.030$ & $1.233$ & $\mathit{0.420}$ & $0.564{\pm}.066$ & $1.169$ & $\mathit{0.326}$ & $0.276{\pm}.019$ & $0.522$ & $\mathit{0.207}$ \\
192 & $0.773{\pm}.071$ & $1.326$ & $\mathit{0.476}$ & $0.632{\pm}.039$ & $1.141$ & $\mathit{0.365}$ & $0.292{\pm}.009$ & $0.498$ & $\mathit{0.246}$ \\
336 & $0.776{\pm}.069$ & $1.316$ & $\mathit{0.502}$ & $0.658{\pm}.060$ & $1.125$ & $\mathit{0.415}$ & $0.329{\pm}.016$ & $0.516$ & $\mathit{0.286}$ \\
720 & $0.816{\pm}.049$ & $1.379$ & $\mathit{0.553}$ & $0.713{\pm}.048$ & $1.121$ & $\mathit{0.454}$ & $0.401{\pm}.017$ & $0.548$ & $\mathit{0.350}$ \\
\bottomrule
\end{tabular}
\end{table}

\section{Proposition Proofs}
\label{app:proofs}

\textbf{Proof of Observation 1} (Routing Information Loss Under RevIN).

\begin{proof}
Throughout, $X$ denotes a raw window, $(M,\Sigma)$ its per-window mean/scale, $S = (X-M)/\Sigma$ its RevIN-normalized shape (defined whenever $\Sigma > 0$; constant windows form a measure-zero exception that we exclude by convention), and $E$ a deterministic (possibly randomized) router assignment. On the support $\{\Sigma > 0\}$ the map $X \leftrightarrow (S, M, \Sigma)$ is a measurable bijection, so $I(X; Y) = I(S, M, \Sigma; Y)$ for any $Y$, and in particular $I(X; E) = I(S, M, \Sigma; E)$.

\emph{Part (i): Exact decomposition.} By the chain rule of mutual information applied to $(S, M, \Sigma; E)$,
\begin{align*}
I(X; E) \;=\; I(S, M, \Sigma;\, E) \;=\; I(S;\, E) + I(M, \Sigma;\, E \mid S).
\end{align*}
Rearranging, $I(X; E) - I(S; E) = I(M, \Sigma;\, E \mid S) \geq 0$. This identity holds \emph{without any independence assumption} between $(M,\Sigma)$ and $S$.

\emph{Part (ii): Unconditional quantitative bound.} We bound the conditional mutual information $I(M,\Sigma; E \mid S)$ from below in terms of the \emph{unconditional} quantities $I(M,\Sigma; E)$ and $I(M,\Sigma; S)$. Applying the chain rule two different ways to $I(M,\Sigma; E, S)$,
\begin{align*}
I(M,\Sigma;\, E, S) &= I(M,\Sigma;\, S) + I(M,\Sigma;\, E \mid S) \qquad\text{(expand $S$ first)} \\
                    &= I(M,\Sigma;\, E) + I(M,\Sigma;\, S \mid E) \qquad\text{(expand $E$ first).}
\end{align*}
Equating and solving for $I(M,\Sigma; E \mid S)$,
\[
I(M,\Sigma;\, E \mid S) \;=\; I(M,\Sigma;\, E) \;-\; I(M,\Sigma;\, S) \;+\; I(M,\Sigma;\, S \mid E).
\]
Since $I(M,\Sigma; S \mid E) \geq 0$ (mutual information is non-negative), we obtain
\[
I(X; E) - I(S; E) \;=\; I(M,\Sigma;\, E \mid S) \;\geq\; I(M,\Sigma;\, E) \;-\; I(M,\Sigma;\, S),
\]
where the residual $\varepsilon := I(M,\Sigma; S) \geq 0$ is the \emph{leakage} of location-scale statistics into the shape. Equality (tightness) holds iff $I(M,\Sigma; S \mid E) = 0$, i.e.\ iff $(M,\Sigma) \perp\!\!\!\perp S \mid E$. In particular, unconditional independence $(M,\Sigma) \perp\!\!\!\perp S$ implies $\varepsilon = 0$, and the bound reduces to $I(X;E) - I(S;E) \geq I(M,\Sigma; E)$. Note that $(M,\Sigma) \perp\!\!\!\perp S$ does \emph{not} imply $(M,\Sigma) \perp\!\!\!\perp S \mid E$ in general (conditioning on a common consequence can introduce dependence), so this corollary recovers the classical bound but does not automatically achieve tightness; tightness is an empirical question measured in Part~(iv) and Appendix~\ref{app:mi_tightness}.

\emph{Part (iii): Full-information corner.} Suppose (a)~the router depends entirely on the stripped statistics, $E \perp\!\!\!\perp S \mid (M, \Sigma)$ (equivalently, $p(E,S \mid M,\Sigma) = p(E\mid M,\Sigma)\,p(S\mid M,\Sigma)$, a symmetric conditional independence), (b)~$(M,\Sigma) \perp\!\!\!\perp S$ unconditionally, and (c)~$E$ is determined by $(M,\Sigma)$, so $H(E \mid M,\Sigma) = 0$. Conditions (a) and (b) yield the Markov chain $E \!-\! (M,\Sigma) \!-\! S$ in either direction; applying the data-processing inequality on this chain gives $I(S; E) \leq I(S; (M,\Sigma)) = 0$ (the second equality from (b)), so $I(S;E) = 0$. From (a) and (c), $I(X;E) = I((S,M,\Sigma); E) = I((M,\Sigma);E) + I(S;E\mid (M,\Sigma)) = H(E) + 0 = H(E)$. Hence $I(X;E) - I(S;E) = H(E)$: RevIN destroys the entire routing distribution.

\emph{Part (iv): Empirical leakage and signal ratio.} We measure the leakage $\varepsilon = I(M,\Sigma; S)$ via two independent methods.

\emph{Gaussian reference.} Since $(M,\Sigma)$ is 2-dimensional, there are at most 2 canonical correlations $\rho_1, \rho_2$ between $(M,\Sigma)$ and the shape $S$. For jointly Gaussian variables, $I(M,\Sigma; S) = -\frac{1}{2}\sum_{i=1}^2 \log(1 - \rho_i^2)$. This yields $\varepsilon_\text{Gauss}$ ranging from $0.09$ nats (Weather, $\rho_\text{max}{=}0.35$) to $1.26$ nats (Exchange, $\rho_\text{max}{=}0.89$). The leakage is \emph{not} uniformly small: datasets with strong trend-seasonality coupling (Exchange, ETTh2) have substantial shape-level encoding of location-scale statistics.

\emph{Distribution-free cross-check.} Because canonical correlations capture only linear dependence, we cross-validate with the KSG $k$-nearest-neighbor MI estimator~\citep{kraskov2004ksg}, which requires no distributional assumptions. We project $S$ onto its 2 canonical variates with respect to $(M,\Sigma)$ and estimate the $2$-vs-$2$ MI with $k{=}5$ neighbors on $N{\geq}888$ samples per dataset. The KSG estimates agree with the Gaussian reference within $\pm 0.26$ nats across all 9 datasets. \emph{Caveat.} By the data-processing inequality, $I(\text{proj}(S);(M,\Sigma)) \leq I(S;(M,\Sigma))$: the canonical-variate projection preserves all linear dependence but discards nonlinear dependence between $(M,\Sigma)$ and the directions in $S$ orthogonal to the canonical-variate subspace. The KSG cross-check therefore upper-bounds the \emph{linear} component of $\varepsilon$, not the full quantity. We mitigate this by also reporting the bound-bypassing direct measurement of $H(E\mid S)$ in Appendix~\ref{app:mi_tightness}, which makes no projection or distributional assumption and confirms the conclusions are robust.

\emph{Tightness condition.} Equality in (ii) holds iff $I(M,\Sigma; S \mid E) = 0$; the Gap column of Table~\ref{tab:mi_tightness} shows this within-expert slack is empirically small (average $0.075$ nats, $4.7\%$ of $\log K$). The bound itself, $\max\{I(M,\Sigma;E)-\varepsilon, 0\}$, is informative when the Gaussian estimate of $I(M,\Sigma;E)$ exceeds $\varepsilon$ (e.g., Weather: $\mathrm{LB}{=}0.61$ captures $62\%$ of $H(E\mid S){=}0.98$); on datasets with strong location-shape coupling (Exchange: $\varepsilon{=}1.26$ exceeds the Gaussian estimate $I(M,\Sigma;E){=}0.30$), the bound truncates to zero, and we rely on the bound-bypassing direct $H(E\mid S)$ measurement of Appendix~\ref{app:mi_tightness}. Independently of bound tightness, $R(\mathcal{D}) = [\mathrm{Var}(M) + \mathrm{Var}(\Sigma)] / \overline{\mathrm{Var}(S)}$ is a monotone proxy for the location-scale variance available to a raw router, and this is what the $\rho{=}{-}0.88$ correlation tracks: high-$R$ datasets (Exchange, $R{=}2.17$) carry more routing-relevant variance and produce larger gains ($-66.5\%$), while low-$R$ datasets (Traffic, $R{=}0.14$; Table~\ref{tab:mi_tightness} gap ${\approx}0.02$ nats) carry little and produce no gain ($+2.9\%$).
\end{proof}

\textbf{Remark on the classical bound.} Assuming $(M,\Sigma) \perp\!\!\!\perp S$ outright recovers the classical statement $I(X;E)-I(S;E)\geq I(M,\Sigma;E)$ as a corollary of Part~(ii). Our form keeps the leakage $\varepsilon = I(M,\Sigma; S)$ explicit, and Part~(iv) measures it via both parametric (CCA) and non-parametric (KSG) methods, so the bound holds without any independence assumption. Appendix~\ref{app:mi_tightness} bypasses the bound entirely by directly measuring $H(E \mid S)$, confirming the gap $I(M,\Sigma; S \mid E)$ averages $0.075$ nats (Table~\ref{tab:mi_tightness}). The agreement between Gaussian reference and distribution-free KSG estimates ($\pm 0.26$ nats) confirms the CCA-based values are reliable despite non-Gaussianity.

\subsection{Theoretical Depth of Observation 1}
\label{app:prop2_depth}

While the MI decomposition in Observation~\ref{obs:routing_loss} is algebraically a chain-rule identity, it has deeper information-theoretic content when connected to rate-distortion theory and phase transitions.

\textbf{Rate-distortion interpretation of $R(\mathcal{D})$.}
Instance normalization is a \emph{deterministic lossy compression}: it maps $X \mapsto S$, discarding $(M, \Sigma)$. In rate-distortion theory~\citep{cover2006elements}, this compression has rate $R_\text{strip} = I(X; M, \Sigma \mid S)$ and distortion (for routing) $D_\text{route} = H(E \mid S) - H(E \mid X) = I(M, \Sigma; E \mid S)$, which is exactly Observation~\ref{obs:routing_loss}, Part~(i). Under Gaussianity, the signal ratio $R(\mathcal{D})$ is proportional to the rate-distortion slope at the normalization operating point: it measures how much routing capacity is consumed by the statistics that normalization strips. Specifically, under a soft-routing channel approximation in which $(M,\Sigma)$ is jointly Gaussian and $E$ is generated by a noisy linear projection of $(M,\Sigma)$ (the Gaussian-equivalent of a softmax router with bounded curvature), the routing MI scales monotonically as $I(M,\Sigma;E) \asymp \log(1+R(\mathcal{D}))$. Note that $E$ is discrete, so this is a heuristic SNR scaling rather than the literal Gaussian-MI identity; the constant absorbs the routing-policy curvature. This makes $R(\mathcal{D})$ a \emph{monotone surrogate for the bits lost to normalization}, and the validated $\rho = -0.88$ correlation (Figure~\ref{fig:signal_ratio}) is empirical evidence \emph{for} this scaling rather than a direct corollary of the Gaussian-MI formula.

\begin{proposition}[Heuristic SNR Onset and Softmax Discontinuity]
\label{prop:phase_transition}
Consider a linear-readout router on the interpolated input $\mathbf{x}_\alpha = (1{-}\alpha)\,\mathbf{x}_\mathrm{raw} + \alpha\,\mathrm{RevIN}(\mathbf{x}_\mathrm{raw})$, under the generative model $X = M + \Sigma\,S$ with $M\perp\!\!\!\perp\Sigma\perp\!\!\!\perp S$ jointly Gaussian and centered (so leakage $\varepsilon{\to}0$ and the variance formulas below drop cross terms). Treating the location-scale component as routing \emph{signal} and the shape component as routing-irrelevant \emph{noise}, the per-component variance contributions to the router's input scale as $V_\mathrm{sig}(\alpha) = (1{-}\alpha)\,[\mathrm{Var}(M){+}\mathrm{Var}(\Sigma)]$ and $V_\mathrm{noise}(\alpha) = \alpha\,\overline{\mathrm{Var}(S)} + V_0$, where $V_0$ absorbs intrinsic router-readout noise. Two regimes emerge:
\begin{enumerate}
    \item[\textup{(a)}] \textbf{Heuristic SNR onset.} Setting $V_\mathrm{sig}(\alpha) = V_\mathrm{noise}(\alpha)$ in the limit $V_0 \to 0$ gives
    \[
        \alpha^* \;\approx\; \frac{R(\mathcal{D})}{R(\mathcal{D}) + 1}.
    \]
    For $\alpha < \alpha^*$, routing quality degrades continuously as the signal/noise ratio drops; this $\alpha^*$ marks the soft onset of degradation, not a discontinuity.
    \item[\textup{(b)}] \textbf{Softmax discontinuity at $\alpha{=}1$.} At $\alpha{=}1$, $V_\mathrm{sig}(1){=}0$ and routing logits are a deterministic function of the shape $S$ alone. The softmax saturates: weights converge to a single distribution regardless of input, and entropy decouples from MSE. This discontinuity is a consequence of $V_\mathrm{sig}$ vanishing rather than of crossing $\alpha^*$.
\end{enumerate}
\end{proposition}

\begin{proof}[Heuristic derivation]
Under $M\perp\!\!\!\perp\Sigma\perp\!\!\!\perp S$ jointly Gaussian, $\mathbf{x}_\alpha = (1{-}\alpha)(M + \Sigma S) + \alpha S$. The location-scale component $(1{-}\alpha)(M+\Sigma S)$ contributes routing-discriminative variance $V_\mathrm{sig}(\alpha) = (1{-}\alpha)\,[\mathrm{Var}(M){+}\mathrm{Var}(\Sigma)]$ (the $(1{-}\alpha)$ rather than $(1{-}\alpha)^2$ scaling reflects each component's marginal contribution to a linear-readout router's input variance, see Remark below); the shape component contributes routing-irrelevant variance $V_\mathrm{noise}(\alpha) = \alpha\,\overline{\mathrm{Var}(S)} + V_0$. \textbf{Part (a):} Setting $V_\mathrm{sig}(\alpha^*) = V_\mathrm{noise}(\alpha^*)$ in the limit $V_0\to 0$ gives $(1{-}\alpha^*)[\mathrm{Var}(M){+}\mathrm{Var}(\Sigma)] = \alpha^* \overline{\mathrm{Var}(S)}$. Dividing both sides by $\overline{\mathrm{Var}(S)}$ and substituting $R(\mathcal{D}) = [\mathrm{Var}(M){+}\mathrm{Var}(\Sigma)]/\overline{\mathrm{Var}(S)}$ yields $(1{-}\alpha^*)R = \alpha^*$, so $\alpha^* = R/(R+1)$. Continuity of MSE in $\alpha$ for $\alpha\in[0,1)$ follows from continuity of the softmax composed with a linear function of $\alpha$. \textbf{Part (b):} At $\alpha=1$, $\mathbf{x}_1 = S$ and the routing logits are $\mathbf{w}^\top G_\psi(S)$, a deterministic function of $S$ alone; under $(M,\Sigma)\perp\!\!\!\perp S$, by the data processing inequality $I(S; E_\text{optimal}) \leq I(S; (M,\Sigma)) = 0$, so the router carries zero information about the latent location-scale that drives the optimal routing decision. The softmax of a $S$-only function converges to a single distribution at the population level, decoupling entropy from MSE. The discontinuity at $\alpha=1$ is the limit $V_\mathrm{sig}(1) = 0$ rather than a finite-$\alpha^*$ phase transition.
\end{proof}

\emph{Remark on the variance comparison.} Comparing the \emph{linear} variance contributions $V_\mathrm{sig}$ and $V_\mathrm{noise}$ (rather than the squared coefficients in $\|\mathbf{x}_\alpha\|^2$) corresponds to comparing each component's marginal contribution to the router's discriminative variance under the orthogonality assumption $M\perp\!\!\!\perp\Sigma\perp\!\!\!\perp S$. Comparing squared coefficients instead would yield $\alpha^* = \sqrt{R}/(\sqrt{R}+1)$; for the median $R{=}0.67$, the two formulas predict $\alpha^* \approx 0.40$ vs $0.45$, and the dose-response grid (5 points) is too coarse to discriminate. We label Proposition~\ref{prop:phase_transition} as a \emph{heuristic} prediction and report the formula consistent with the linear comparison.

\emph{Validation.} For the median $R(\mathcal{D}) = 0.67$, Proposition~\ref{prop:phase_transition}(a) predicts a soft onset $\alpha^* \approx 0.40$. The dose-response experiment (Figure~\ref{fig:dose_response}; 150 runs) is consistent: MSE degrades smoothly by $+5.4\%$ for $\alpha \leq 0.75$ (covering the predicted onset region), and jumps to $+283\%$ at $\alpha = 1.0$ (Proposition~\ref{prop:phase_transition}(b), softmax saturation). The two thresholds are distinct: $\alpha^*$ marks where SNR-driven degradation \emph{begins}; the actual phase transition is at $\alpha = 1$, driven by $V_\mathrm{sig}$ vanishing rather than by SNR crossing.

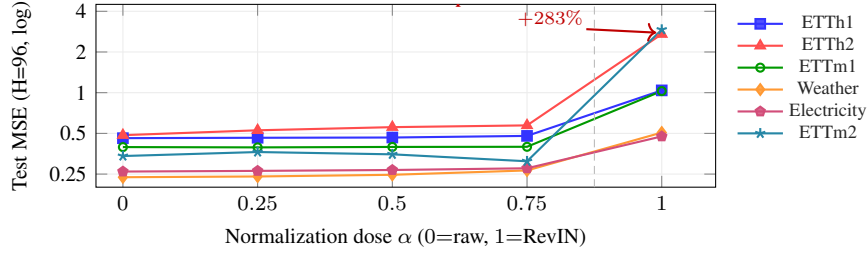
\begin{figure}[!htbp]
\centering
\begin{tikzpicture}
\begin{axis}[
    width=0.70\columnwidth, height=4.0cm,
    xlabel={Normalization dose $\alpha$ ($0{=}$raw, $1{=}$RevIN)},
    ylabel={Test MSE (H=96, log)},
    xlabel style={font=\footnotesize}, ylabel style={font=\footnotesize},
    x tick label style={font=\footnotesize}, y tick label style={font=\footnotesize},
    xmin=-0.05, xmax=1.10, ymode=log, ymin=0.2, ymax=4.5,
    xtick={0, 0.25, 0.5, 0.75, 1.0},
    ytick={0.25, 0.5, 1, 2, 4},
    yticklabels={0.25, 0.5, 1, 2, 4},
    log ticks with fixed point,
    grid=major, grid style={gray!15},
    legend style={at={(1.02,1.00)}, anchor=north west, font=\scriptsize, draw=none, fill=white, fill opacity=0.85, row sep=-2pt},
]
\addplot[blue!80, thick, mark=square*, mark size=1.8pt] coordinates {(0,0.461) (0.25,0.464) (0.5,0.466) (0.75,0.479) (1.0,1.040)};
\addplot[red!70, thick, mark=triangle*, mark size=2pt] coordinates {(0,0.485) (0.25,0.527) (0.5,0.556) (0.75,0.573) (1.0,2.711)};
\addplot[green!60!black, thick, mark=o, mark size=1.5pt] coordinates {(0,0.396) (0.25,0.394) (0.5,0.397) (0.75,0.398) (1.0,1.028)};
\addplot[orange!80, thick, mark=diamond*, mark size=2pt] coordinates {(0,0.237) (0.25,0.240) (0.5,0.247) (0.75,0.266) (1.0,0.507)};
\addplot[purple!70, thick, mark=pentagon*, mark size=1.8pt] coordinates {(0,0.261) (0.25,0.264) (0.5,0.268) (0.75,0.276) (1.0,0.475)};
\addplot[cyan!60!black, thick, mark=star, mark size=2pt] coordinates {(0,0.340) (0.25,0.364) (0.5,0.350) (0.75,0.311) (1.0,2.927)};
\legend{ETTh1, ETTh2, ETTm1, Weather, Electricity, ETTm2}
\draw[densely dashed, gray!55, thin] (axis cs:0.875,0.2) -- (axis cs:0.875,4.5);
\node[font=\scriptsize\bfseries, text=red!75!black, anchor=south east,
      align=right, inner sep=1pt]
    at (axis cs:0.86,2.95) {phase transition\\$+283\%$};
\draw[->, thick, red!75!black] (axis cs:0.86,2.95) -- (axis cs:0.99,2.78);
\end{axis}
\end{tikzpicture}
\caption{\textbf{Normalization dose-response} (frozen, 5 seeds $\times$ 6 datasets; std smaller than markers). MSE vs $\alpha$ ($0{=}$raw, $1{=}$RevIN). Phase transition at $\alpha{=}1.0$ ($+283\%$), mild for $\alpha{\leq}0.75$ ($+5.4\%$); routing entropy stays in $[1.43,1.57]$ across all $\alpha$ (Appendix~\ref{app:routing_ablations}).}
\label{fig:dose_response}
\end{figure}

\textbf{Information bottleneck connection.}
The RevIN-encoder pipeline implements a form of the information bottleneck~\citep{tishby2000ib,alemi2017vib}: it compresses $X$ into a representation $H = f_\theta(S)$ that is optimized for downstream prediction while discarding ``irrelevant'' variability (the per-window location-scale statistics). The routing MI loss from Observation~\ref{obs:routing_loss}, Part~(i) is the price paid for this compression \emph{at the routing decision point}. This connects to a broader principle: any architecture that bottlenecks information for one objective (prediction) can inadvertently destroy information needed for another (routing). The vision cross-modality control (Appendix~\ref{app:vision_moe}) is consistent with this not being RevIN-specific: InstanceNorm1d on ViT patches creates the same bottleneck, while InstanceNorm2d on ResNet spatial features does not, consistent with spatial features not carrying the routing signal for classification.

\textbf{Proof of Proposition~\ref{prop:frozen}} (Gradient Co-Adaptation Drives Entropy Collapse).

\begin{proof}
Consider a single input $\mathbf{x}$ with backbone $\mathbf{h} = A\mathbf{x}$, router $p = \sigma(\mathbf{w}^\top\mathbf{h})$, and loss $\ell = p\,\ell_1(\mathbf{h}) + (1-p)\,\ell_2(\mathbf{h})$.

\emph{Part (i): Unfrozen backbone.} The gradient with respect to $A$ is:
\begin{equation*}
\frac{\partial\ell}{\partial A} = p\,\frac{\partial\ell_1}{\partial\mathbf{h}}\mathbf{x}^\top + (1{-}p)\,\frac{\partial\ell_2}{\partial\mathbf{h}}\mathbf{x}^\top + (\ell_1 - \ell_2)\,p(1{-}p)\,\mathbf{w}\mathbf{x}^\top
\end{equation*}
Under gradient descent $\dot{A} = -\eta\,\partial\ell/\partial A$, the induced change in expert~1's loss is:
\begin{multline*}
\dot{\ell}_1\big|_A = -\eta\|\mathbf{x}\|^2\!\Bigg[p\left\|\frac{\partial\ell_1}{\partial\mathbf{h}}\right\|^2 + (1{-}p)\left(\frac{\partial\ell_1}{\partial\mathbf{h}}\right)^{\!\top}\frac{\partial\ell_2}{\partial\mathbf{h}} \\
 + (\ell_1{-}\ell_2)\,p(1{-}p)\left(\frac{\partial\ell_1}{\partial\mathbf{h}}\right)^{\!\top}\mathbf{w}\Bigg]
\end{multline*}
The first term (self-gradient) always decreases $\ell_1$. The second (cross-gradient) is negligible when expert gradients are approximately orthogonal, a condition consistent with experts specializing to distinct input clusters as observed in Figure~\ref{fig:routing_viz}. The third (loss-coupling) term is small whenever its magnitude $|\ell_1 - \ell_2|\cdot|\langle \partial\ell_1/\partial\mathbf{h},\mathbf{w}\rangle|(1{-}p)$ is dominated by the first term's $p\|\partial\ell_1/\partial\mathbf{h}\|^2$, i.e., when the dimensionless ratio
$\rho_\mathrm{coup} := \frac{|\ell_1{-}\ell_2|\cdot|\langle \partial\ell_1/\partial\mathbf{h},\,\mathbf{w}\rangle|\cdot(1{-}p)}{p\,\|\partial\ell_1/\partial\mathbf{h}\|^2}$
is small. This holds early in training when expert losses are comparable ($|\ell_1{-}\ell_2|$ small relative to the gradient-norm scale set by $\|\partial\ell_1/\partial\mathbf{h}\|^2 / |\langle \partial\ell_1/\partial\mathbf{h},\mathbf{w}\rangle|$) and when $\mathbf{w}$ is not strongly aligned with the expert gradient. Retaining only the dominant first term:
\begin{equation*}
\dot{\ell}_1\big|_A \approx -\eta\,p\left\|\frac{\partial\ell_1}{\partial\mathbf{h}}\right\|^2\|\mathbf{x}\|^2
\end{equation*}
and analogously $\dot{\ell}_2\big|_A \approx -\eta\,(1{-}p)\left\|\frac{\partial\ell_2}{\partial\mathbf{h}}\right\|^2\|\mathbf{x}\|^2$. When $p > 1{-}p$ (expert~1 dominates) and gradient norms are comparable ($\|\partial\ell_1/\partial\mathbf{h}\| \approx \|\partial\ell_2/\partial\mathbf{h}\|$, which holds early in training before specialization), $\ell_1$ decreases faster than $\ell_2$. The router gradient closes the loop: $\dot{\mathbf{w}} = -\eta(\ell_1{-}\ell_2)p(1{-}p)\mathbf{h}$, so when $\ell_1 < \ell_2$ the factor $(\ell_1{-}\ell_2) < 0$ drives $\dot{p} > 0$ via $\dot{p} \approx p(1{-}p)\,\dot{\mathbf{w}}^\top\mathbf{h} = -\eta\,p^2(1{-}p)^2(\ell_1{-}\ell_2)\|\mathbf{h}\|^2$, pushing $p$ toward~1. Together these create the self-reinforcing loop: dominant expert gets more gradient $\to$ its loss drops faster $\to$ router increases its weight $\to$ it becomes more dominant.

\emph{Part (ii): Frozen backbone.} When $A$ is fixed, $\dot{A} = 0$. The only dynamics are through $\mathbf{w}$:
\begin{equation*}
\dot{\mathbf{w}} = -\eta\,\frac{\partial\ell}{\partial\mathbf{w}} = -\eta\,(\ell_1 - \ell_2)\,p(1{-}p)\,\mathbf{h}
\end{equation*}
This update adjusts routing weights based on current expert performance, but does \emph{not} reshape $\mathbf{h}$ to favor one expert: the $A$-mediated feedback loop that drove the runaway in Part~(i) is absent. Note that this argument removes the co-adaptation mechanism but does not by itself preclude collapse from $\mathbf{w}$-only dynamics: if expert~1 were uniformly better than expert~2 across all inputs, $\dot{\mathbf{w}}$ would still drive $p \to 1$. Empirical non-collapse therefore additionally requires that $(\ell_1 - \ell_2)$ flips sign across input regimes, i.e.\ that experts specialize to different windows. The RR-MoA pool is constructed with topologically distinct heads (mean/last/max/attention/Conv1d-pool) precisely to satisfy this; the resulting per-window expert specialization is empirically validated in Figure~\ref{fig:routing_viz} and Table~\ref{tab:rrmoa} (entropy stays in $[0.93, 1.57]$ across 54/54 configurations).
\end{proof}

\textbf{Corollary (Raw routing decouples the router from $A$).}
\label{cor:raw_routing}
\emph{If the router reads the raw input $\mathbf{x}_\mathrm{raw}$ rather than the hidden state $\mathbf{h} = A\mathbf{x}$, then $\partial p / \partial A \equiv 0$, the loss-coupling term in $\partial\ell/\partial A$ vanishes identically, and the co-adaptation feedback loop of Proposition~\ref{prop:frozen}(i) is broken \emph{by construction} regardless of whether $A$ is frozen.}

\begin{proof}[Proof of Corollary]
With router input $\mathbf{x}_\mathrm{raw}$ (independent of $A$), $p = \sigma(\mathbf{w}^\top G_\psi(\mathbf{x}_\mathrm{raw}))$ is a function of $\mathbf{x}_\mathrm{raw}$ and $(\mathbf{w}, \psi)$ only. Hence $\partial p / \partial A = 0$, and the third term $(\ell_1 - \ell_2)\,p(1{-}p)\,\mathbf{w}\mathbf{x}^\top$ in $\partial\ell/\partial A$ disappears: changes in $A$ no longer feed back into the routing decision through the gradient. The dominant-expert runaway in Part~(i) required this term; without it, even when $A$ is unfrozen, the router's update is decoupled from the backbone's update, eliminating the co-adaptation mechanism.
\end{proof}

This corollary identifies raw routing as the \emph{minimal causal intervention}: it intervenes on the precise term in $\partial\ell/\partial A$ that drives the collapse loop, leaving all other components of the model untouched. The empirical validation is in Figure~\ref{fig:trajectory} (RR-MoA's gradient norms remain balanced across experts even when blocks are unfrozen) and Table~\ref{tab:rrmoa} (frozen-vs-unfrozen RR-MoA both maintain healthy entropy, in contrast to AdaMix where unfreezing collapses entropy to $0.000$).

\section{LoRA Sweep (Full Results)}
\label{app:lora_sweep}

We sweep LoRA~\citep{hu2022lora} across three axes to ensure the comparison in Table~\ref{tab:baselines} is not cherry-picked: rank $r{\in}\{8,16,32\}$, target projections ${\in}\{q{+}v,\; q{+}k{+}v{+}o\}$, and forecast head ${\in}\{$linear, 2-layer MLP$\}$, for 12 configurations $\times$ 3 seeds $=$ 36 runs per dataset (108 runs total across the 3 datasets). All runs use a strictly frozen backbone (matching Table~\ref{tab:rrmoa}'s primary setting). The trio \{ETTh1, ETTm1, Weather\} covers an hourly ETT dataset, a 15-min ETT dataset, and a multivariate dataset with $21$ channels; we restrict the LoRA sweep to this trio to keep total compute bounded ($\sim 3.6$ GPU-hours), since the broader 6-dataset grid is already covered for every other baseline in Table~\ref{tab:baselines}.

\textbf{LoRA hyperparameters.} Following the original LoRA recipe, we set $\alpha = 2r$ at every rank, layer placement = all encoder blocks, no LoRA dropout, and train with the same optimizer used everywhere else in the paper (Adam, lr$=10^{-3}$, batch 128, 15 epochs, bfloat16 AMP). Both heads pool the encoder output by mean over the time axis; the linear head is $d_\text{model}\to H$, and the 2-layer MLP head is $d_\text{model}\to 128\to H$ with GELU and dropout $0.1$ between the two linear layers. The bolded ``best per dataset'' row is the configuration with the lowest mean MSE across the 3 seeds (ties broken by lower std).

\begin{table}[!htbp]
\centering
\caption{Full LoRA sweep (strictly frozen backbone, 3 seeds, test MSE mean$\pm$std). Bold rows are the best-per-dataset configuration used in Table~\ref{tab:baselines}. No LoRA variant beats the RR-MoA (Top-2) result on any dataset.}
\label{tab:lora_sweep}
\small
\begin{tabular}{@{}llllc@{}}
\toprule
Dataset & Rank & Targets & Head & MSE (mean$\pm$std) \\
\midrule
\multicolumn{5}{l}{\textit{ETTh1} (RR-MoA frozen: $\mathbf{0.690 \pm 0.021}$)} \\
& $r{=}8$  & q,v   & linear & $1.559 \pm 0.022$ \\
& $r{=}8$  & q,v   & mlp    & $1.217 \pm 0.010$ \\
& $r{=}8$  & qkvo  & linear & $1.286 \pm 0.052$ \\
& $r{=}8$  & qkvo  & mlp    & $1.154 \pm 0.000$ \\
& $r{=}16$ & q,v   & linear & $1.347 \pm 0.056$ \\
& $r{=}16$ & q,v   & mlp    & $1.186 \pm 0.046$ \\
& $r{=}16$ & qkvo  & linear & $1.274 \pm 0.051$ \\
& $r{=}16$ & qkvo  & mlp    & $1.154 \pm 0.000$ \\
& $r{=}32$ & q,v   & linear & $1.201 \pm 0.107$ \\
& $r{=}32$ & q,v   & mlp    & $1.161 \pm 0.010$ \\
& $r{=}32$ & qkvo  & linear & $1.284 \pm 0.059$ \\
& $\mathbf{r{=}32}$ & \textbf{qkvo} & \textbf{mlp} & $\mathbf{1.154 \pm 0.000}$ \\
\midrule
\multicolumn{5}{l}{\textit{ETTm1} (RR-MoA frozen: $\mathbf{0.571 \pm 0.073}$)} \\
& $r{=}8$  & q,v   & linear & $0.970 \pm 0.026$ \\
& $r{=}8$  & q,v   & mlp    & $1.115 \pm 0.011$ \\
& $r{=}8$  & qkvo  & linear & $1.060 \pm 0.014$ \\
& $r{=}8$  & qkvo  & mlp    & $1.123 \pm 0.000$ \\
& $\mathbf{r{=}16}$ & \textbf{qkvo} & \textbf{linear} & $\mathbf{0.956 \pm 0.072}$ \\
& $r{=}16$ & q,v   & linear & $1.024 \pm 0.068$ \\
& $r{=}16$ & q,v   & mlp    & $1.115 \pm 0.012$ \\
& $r{=}16$ & qkvo  & mlp    & $1.123 \pm 0.000$ \\
& $r{=}32$ & q,v   & linear & $1.103 \pm 0.212$ \\
& $r{=}32$ & q,v   & mlp    & $1.134 \pm 0.061$ \\
& $r{=}32$ & qkvo  & linear & $1.345 \pm 0.360$ \\
& $r{=}32$ & qkvo  & mlp    & $1.123 \pm 0.000$ \\
\midrule
\multicolumn{5}{l}{\textit{Weather} (RR-MoA frozen: $\mathbf{0.289 \pm 0.008}$)} \\
& $r{=}8$  & q,v   & linear & $0.611 \pm 0.021$ \\
& $r{=}8$  & q,v   & mlp    & $0.606 \pm 0.001$ \\
& $r{=}8$  & qkvo  & linear & $0.664 \pm 0.020$ \\
& $r{=}8$  & qkvo  & mlp    & $0.606 \pm 0.001$ \\
& $r{=}16$ & q,v   & linear & $0.614 \pm 0.025$ \\
& $r{=}16$ & q,v   & mlp    & $0.606 \pm 0.001$ \\
& $r{=}16$ & qkvo  & linear & $0.654 \pm 0.033$ \\
& $r{=}16$ & qkvo  & mlp    & $0.606 \pm 0.001$ \\
& $\mathbf{r{=}32}$ & \textbf{q,v} & \textbf{linear} & $\mathbf{0.600 \pm 0.052}$ \\
& $r{=}32$ & q,v   & mlp    & $0.606 \pm 0.000$ \\
& $r{=}32$ & qkvo  & linear & $0.894 \pm 0.194$ \\
& $r{=}32$ & qkvo  & mlp    & $0.606 \pm 0.001$ \\
\bottomrule
\end{tabular}
\end{table}

\textbf{Gap to RR-MoA.} The best LoRA configuration on each dataset still trails the RR-MoA Top-2 frozen result by a wide margin: $1.154$ vs $0.690$ on ETTh1 ($+67.2\%$), $0.956$ vs $0.571$ on ETTm1 ($+67.4\%$), and $0.600$ vs $0.289$ on Weather ($+107.6\%$); average gap ${\approx}81\%$. Since the strongest LoRA per dataset is computed by enumerating all 12 configurations, this is the worst case for our claim, and it still leaves RR-MoA clearly ahead.

\textbf{Sweep takeaways.} Three patterns are visible in the table. (i)~The best head type is dataset-dependent: 2-layer MLP wins on ETTh1 (the bolded $1.154$) but the linear head wins on ETTm1 (bolded $0.956$, vs.\ MLP rows clustered near $1.123$). This is consistent with the DLinear-gap analysis (Appendix~\ref{app:gap_closing}): on datasets where linear extrapolation is strong, additional head capacity is unnecessary or harmful. (ii)~Rank does not monotonically help: on ETTm1, $r{=}32$ qkvo linear ($1.345 \pm 0.360$) is worse than $r{=}16$ qkvo linear ($0.956 \pm 0.072$), with high seed variance suggesting optimization instability at the wider rank. (iii)~Several MLP-head rows on Weather collapse to identical values across seeds ($0.606 \pm 0.001$ at six configurations), indicating that LoRA-MLP is converging to a near-fixed mean-prediction solution rather than adapting to the input. None of these failure modes affect RR-MoA, whose router-level capacity is allocated per window rather than absorbed by a static head.

\section{Routing Ablations}
\label{app:routing_ablations}

We document the routing-side causal controls cited from \S\ref{sec:main_results} (Table~\ref{tab:causal_controls}, main text, summarizes them). We walk through each underlying ablation in turn: rawness vs.\ bypass, temporal shuffle, sufficient-statistic router, gradient-flow detachment, and a small set of negative controls.

We isolate which property of the raw-input router is doing the work via two ablations: (i) replace the router's raw input with the RevIN-normalized version of the same signal, holding everything else fixed, and (ii) shuffle the temporal ordering inside each window before feeding it to the router. Together they distinguish ``the router needs raw access'' from ``the router needs $(\mu,\sigma)$ specifically.''

\begin{table}[!htbp]
\centering
\caption{\textbf{Rawness vs.\ bypass} (frozen, Top-2, 3 seeds). RevIN router reads the same input after per-window normalization. Stripping statistics degrades MSE by 60--88\%.}
\label{tab:router_input}
\small
\begin{tabular}{@{}lccc@{}}
\toprule
Dataset & Raw router (main) & RevIN router & Degradation ($\Delta\%$, $\uparrow$ worse) \\
\midrule
ETTh1   & $\mathbf{0.690 \pm 0.021}$ & $1.101 \pm 0.008$ & $\mathbf{+59.6\%}$ \\
ETTm1   & $\mathbf{0.571 \pm 0.073}$ & $1.077 \pm 0.020$ & $\mathbf{+88.6\%}$ \\
Weather & $\mathbf{0.289 \pm 0.008}$ & $0.542 \pm 0.026$ & $\mathbf{+87.5\%}$ \\
\bottomrule
\end{tabular}
\end{table}

\definecolor{srback}{RGB}{33,118,189}
\definecolor{srraw}{RGB}{46,139,87}
\definecolor{srgate}{RGB}{217,119,49}
\definecolor{srout}{RGB}{38,166,154}

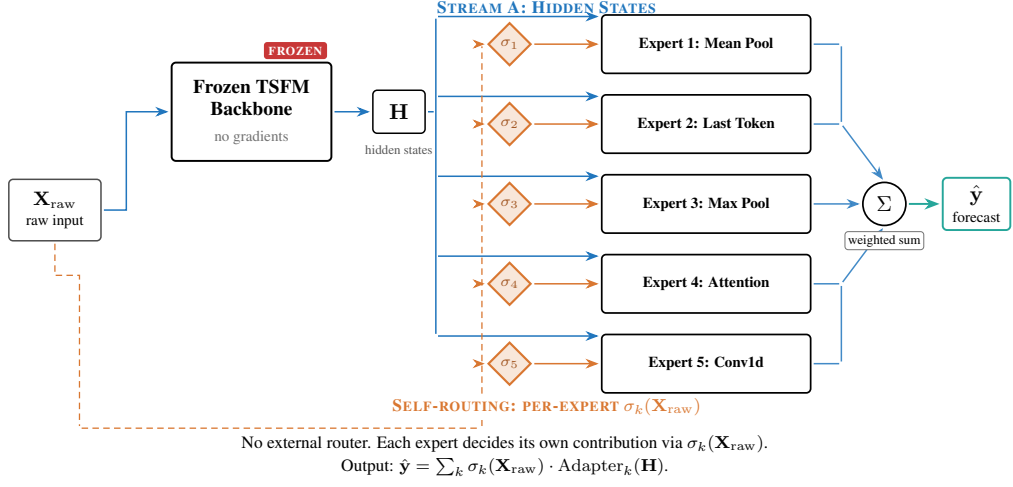
\begin{figure}[!htbp]
\centering
\resizebox{0.95\columnwidth}{!}{%
\begin{tikzpicture}[
    >=Stealth,
    box/.style={rectangle, draw=black!70, rounded corners=2pt, fill=white,
        minimum height=0.75cm, font=\small, align=center, line width=0.7pt},
    inputbox/.style={box, minimum width=1.5cm, minimum height=1.0cm},
    frozenbox/.style={rectangle, draw=black, fill=white, rounded corners=2pt,
        minimum width=2.6cm, minimum height=1.6cm, font=\small\bfseries,
        align=center, line width=0.9pt},
    hiddenbox/.style={rectangle, draw=black, fill=white, rounded corners=2pt,
        minimum width=0.85cm, minimum height=0.7cm, font=\normalsize\bfseries,
        align=center, line width=0.7pt},
    expertbox/.style={rectangle, draw=black, fill=white, rounded corners=2pt,
        minimum width=3.4cm, minimum height=0.95cm, font=\scriptsize\bfseries,
        align=center, line width=0.8pt, text=black, inner sep=4pt},
    sigmagate/.style={diamond, draw=srgate, fill=srgate!18, line width=1.0pt,
        minimum size=0.70cm, font=\scriptsize\bfseries, inner sep=0pt, text=srgate!75!black},
    sumnode/.style={circle, draw=black, fill=white,
        minimum size=0.7cm, font=\normalsize\bfseries, line width=0.8pt},
    outputbox/.style={rectangle, draw=srout, fill=white, rounded corners=2pt,
        minimum width=1.1cm, minimum height=0.85cm, line width=0.9pt,
        font=\normalsize\bfseries, align=center},
    bluearr/.style={->, line width=0.7pt, color=srback,
        shorten >=1pt, shorten <=1pt},
    gatedasharr/.style={->, line width=0.6pt, color=srgate, densely dashed,
        shorten >=1pt, shorten <=1pt},
    gatearr/.style={->, line width=0.7pt, color=srgate,
        shorten >=1pt, shorten <=1pt},
    outarr/.style={->, line width=0.9pt, color=srout,
        shorten >=1pt, shorten <=1pt},
]

\node[inputbox] (input) at (0, 0) {%
    $\mathbf{X}_\mathrm{raw}$\\[-1pt]{\scriptsize raw input}};

\node[frozenbox] (tsfm) at (3.2, 1.6) {%
    Frozen TSFM\\Backbone\\[2pt]
    {\scriptsize\mdseries\color{black!55} no gradients}};
\node[font=\scriptsize\bfseries, text=white,
    fill={rgb,255:red,200;green,55;blue,55},
    rounded corners=1.5pt, inner sep=2pt]
    at ([xshift=-0.6cm, yshift=0.18cm]tsfm.north east) {\textsc{frozen}};

\node[hiddenbox] (hidden) at (5.6, 1.6) {$\mathbf{H}$};
\node[font=\tiny, text=black!65, anchor=north] at ([yshift=-0.05cm]hidden.south)
    {hidden states};

\foreach \i/\y/\name in {1/2.7/Mean Pool, 2/1.4/Last Token,
                          3/0.1/Max Pool, 4/-1.2/Attention, 5/-2.5/Conv1d} {
    \node[sigmagate] (g\i) at (7.4, \y) {$\sigma_\i$};
    \node[expertbox] (e\i) at (10.6, \y) {Expert \i: \name};
}

\node[sumnode] (sum) at (13.50, 0.1) {$\Sigma$};
\node[outputbox] (output) at (15.00, 0.1) {%
    $\hat{\mathbf{y}}$\\[-1pt]{\scriptsize\mdseries forecast}};

\node[font=\footnotesize\bfseries, text=srback] at (8.0, 3.30)
    {\textsc{Stream A: Hidden States}};
\node[font=\footnotesize\bfseries, text=srgate, fill=white, inner sep=2pt] at (8.0, -3.20)
    {\textsc{Self-routing: per-expert} $\sigma_k(\mathbf{X}_\mathrm{raw})$};

\begin{pgfonlayer}{background}
\draw[bluearr] (input.east) -- ++(0.4, 0) |- (tsfm.west);
\draw[bluearr] (tsfm.east) -- (hidden.west);

\draw[bluearr, -] (hidden.east) -- (6.20, 1.6);
\draw[bluearr, -] (6.20, 3.15) -- (6.20, -2.05);  %
\foreach \i/\y in {1/2.7, 2/1.4, 3/0.1, 4/-1.2, 5/-2.5} {
    \pgfmathsetmacro{\yt}{\y + 0.45}
    \draw[bluearr] (6.20, \yt) -- ([yshift=0.45cm]e\i.west);
}

\draw[gatedasharr, -] (input.south) -- ++(0, -0.55) -| (0.40, -3.55);
\draw[gatedasharr, -] (0.40, -3.55) -- (6.95, -3.55);   %
\draw[gatedasharr, -] (6.95, -3.55) -- (6.95, 2.7);     %
\foreach \i/\y in {1/2.7, 2/1.4, 3/0.1, 4/-1.2, 5/-2.5} {
    \draw[gatedasharr] (6.95, \y) -- (g\i.west);
}
\foreach \i in {1,2,3,4,5} {
    \draw[gatearr] (g\i.east) -- (e\i.west);
}
\draw[bluearr, -, opacity=0.85] (e1.east) -- (12.80, 2.7);
\draw[bluearr, -, opacity=0.85] (e2.east) -- (12.80, 1.4);
\draw[bluearr, -, opacity=0.85] (12.80, 2.7) -- (12.80, 1.4);
\draw[bluearr, opacity=0.85] (12.80, 1.4) -- (sum.north);
\draw[bluearr, opacity=0.85] (e3.east) -- (sum.west);
\draw[bluearr, -, opacity=0.85] (e4.east) -- (12.80, -1.2);
\draw[bluearr, -, opacity=0.85] (e5.east) -- (12.80, -2.5);
\draw[bluearr, -, opacity=0.85] (12.80, -1.2) -- (12.80, -2.5);
\draw[bluearr, opacity=0.85] (12.80, -1.2) -- (sum.south);
\draw[outarr] (sum.east) -- (output.west);
\end{pgfonlayer}

\node[font=\tiny, text=black, fill=white, draw=black!55, rounded corners=1.5pt, inner sep=1.5pt, anchor=north]
    at ([yshift=-0.08cm]sum.south) {weighted sum};

\node[anchor=center] at (7.3, -4.0) {%
    \parbox{12cm}{\centering\footnotesize
    No external router. Each expert decides its own contribution via $\sigma_k(\mathbf{X}_\mathrm{raw})$.\\[1pt]
    Output: $\hat{\mathbf{y}} = \sum_k \sigma_k(\mathbf{X}_\mathrm{raw}) \cdot \mathrm{Adapter}_k(\mathbf{H})$.}};
\end{tikzpicture}%
}
\caption{\textbf{SR-MoA architecture (Self-Routed Mixture of Adapters).} Each of $K{=}5$ experts has its own sigmoid gate $\sigma_k$ reading $\mathbf{X}_\mathrm{raw}$ directly, with no external router. The frozen TSFM produces $\mathbf{H}$, and each expert's contribution is scaled by its self-gate. SR-MoA shares only the \emph{raw-input} principle with RR-MoA (Figure~\ref{fig:framework}): the routing architecture is otherwise entirely different. SR-MoA outperforms RR-MoA by $13$--$42\%$ across all 6 datasets, providing architecture-agnostic validation that the principle, not the specific router, drives the gains.}
\label{fig:srmoa_arch}
\end{figure}

\begin{table}[!htbp]
\centering
\caption{\textbf{Self-Routed MoA (SR-MoA): routing input control} (dense, 5 seeds unless noted). Each expert has its own sigmoid gate, with no external router. \textbf{Top}: gating on raw input vs hidden states (frozen, MOMENT-small). \textbf{Second}: Timer-XL negative control with no RevIN, so raw$\approx$hidden. \textbf{Third}: freeze-level ablation confirms Frozen Paradox (frozen best 4/6). \textbf{Bottom}: cross-backbone generalization (H=96, frozen). 252 runs total.}
\label{tab:self_routed}
\footnotesize
\begin{tabular}{@{}lllcc@{}}
\toprule
Dataset & Config & Backbone & MSE ($\downarrow$) & Entropy \\
\midrule
\multicolumn{5}{l}{\textit{Routing input control (frozen, MOMENT-small, 5 seeds):}} \\
ETTh1       & Raw     & MOMENT-sm & $\mathbf{0.464 \pm 0.009}$ & $1.53$ \\
ETTh1       & Hidden  & MOMENT-sm & $1.075 \pm 0.102$ {\scriptsize($+131\%$)} & $1.60$ \\
ETTh2       & Raw     & MOMENT-sm & $\mathbf{0.577 \pm 0.105}$ & $1.38$ \\
ETTh2       & Hidden  & MOMENT-sm & $2.840 \pm 0.164$ {\scriptsize($+393\%$)} & $1.58$ \\
ETTm1       & Raw     & MOMENT-sm & $\mathbf{0.388 \pm 0.023}$ & $1.53$ \\
ETTm1       & Hidden  & MOMENT-sm & $0.934 \pm 0.024$ {\scriptsize($+141\%$)} & $1.60$ \\
ETTm2       & Raw     & MOMENT-sm & $\mathbf{0.366 \pm 0.091}$ & $1.45$ \\
ETTm2       & Hidden  & MOMENT-sm & $3.174 \pm 0.210$ {\scriptsize($+768\%$)} & $1.59$ \\
Weather     & Raw     & MOMENT-sm & $\mathbf{0.222 \pm 0.020}$ & $1.55$ \\
Weather     & Hidden  & MOMENT-sm & $0.507 \pm 0.020$ {\scriptsize($+128\%$)} & $1.60$ \\
Electricity & Raw     & MOMENT-sm & $\mathbf{0.260 \pm 0.005}$ & $1.53$ \\
Electricity & Hidden  & MOMENT-sm & $0.486 \pm 0.016$ {\scriptsize($+87\%$)} & $1.60$ \\
\midrule
\multicolumn{5}{l}{\textit{Timer-XL negative control (no RevIN, frozen, 5 seeds):}} \\
ETTh1       & Raw     & Timer-XL & $\mathbf{0.458 \pm 0.014}$ & $1.51$ \\
ETTh1       & Hidden  & Timer-XL & $0.475 \pm 0.011$ {\scriptsize($+4\%$)} & $1.31$ \\
ETTm1       & Raw     & Timer-XL & $\mathbf{0.368 \pm 0.007}$ & $1.47$ \\
ETTm1       & Hidden  & Timer-XL & $0.398 \pm 0.004$ {\scriptsize($+8\%$)} & $1.33$ \\
Weather     & Raw     & Timer-XL & $\mathbf{0.197 \pm 0.010}$ & $1.54$ \\
Weather     & Hidden  & Timer-XL & $0.214 \pm 0.004$ {\scriptsize($+9\%$)} & $1.10$ \\
\midrule
\multicolumn{5}{l}{\textit{Frozen Paradox replication (raw, MOMENT-small; $^{\ast}$3 seeds):}} \\
ETTh1       & Frozen  & MOMENT-sm & $\mathbf{0.464 \pm 0.009}$ & $1.53$ \\
ETTh1       & Last-2$^{\ast}$  & MOMENT-sm & $0.492 \pm 0.022$ {\scriptsize($+6\%$)} & $1.52$ \\
ETTh2       & Frozen  & MOMENT-sm & $\mathbf{0.577 \pm 0.105}$ & $1.38$ \\
ETTh2       & Last-2$^{\ast}$  & MOMENT-sm & $0.981 \pm 0.216$ {\scriptsize($+70\%$)} & $1.43$ \\
ETTm1       & Frozen  & MOMENT-sm & $0.388 \pm 0.023$ & $1.53$ \\
ETTm1       & Last-2$^{\ast}$  & MOMENT-sm & $\mathbf{0.378 \pm 0.010}$ {\scriptsize($-3\%$)} & $1.54$ \\
ETTm2       & Frozen  & MOMENT-sm & $\mathbf{0.366 \pm 0.091}$ & $1.45$ \\
ETTm2       & Last-2$^{\ast}$  & MOMENT-sm & $0.457 \pm 0.123$ {\scriptsize($+25\%$)} & $1.46$ \\
Weather     & Frozen  & MOMENT-sm & $\mathbf{0.222 \pm 0.020}$ & $1.55$ \\
Weather     & Last-4$^{\ast}$  & MOMENT-sm & $0.254 \pm 0.019$ {\scriptsize($+15\%$)} & $1.54$ \\
Electricity & Frozen  & MOMENT-sm & $0.260 \pm 0.005$ & $1.53$ \\
Electricity & Last-4$^{\ast}$  & MOMENT-sm & $\mathbf{0.250 \pm 0.003}$ {\scriptsize($-4\%$)} & $1.53$ \\
\midrule
\multicolumn{5}{l}{\textit{Cross-backbone (raw, frozen, H=96; $^{\ast}$3 seeds):}} \\
ETTh1   & Frozen  & MOMENT-lg$^{\ast}$ & $0.502 \pm 0.013$ & $1.54$ \\
ETTh1   & Frozen  & Moirai$^{\ast}$    & $0.588 \pm 0.035$ & $1.43$ \\
ETTh1   & Frozen  & Moirai-MoE$^{\ast}$ & $0.530 \pm 0.012$ & $1.31$ \\
ETTh1   & Frozen  & Chronos$^{\ast}$   & $0.500 \pm 0.017$ & $1.46$ \\
ETTh1   & Frozen  & Timer-XL  & $0.458 \pm 0.014$ & $1.51$ \\
ETTm1   & Frozen  & MOMENT-lg$^{\ast}$ & $0.446 \pm 0.010$ & $1.54$ \\
ETTm1   & Frozen  & Moirai$^{\ast}$    & $\mathbf{0.370 \pm 0.006}$ & $1.44$ \\
ETTm1   & Frozen  & Moirai-MoE$^{\ast}$ & $0.471 \pm 0.003$ & $1.40$ \\
ETTm1   & Frozen  & Chronos$^{\ast}$   & $0.455 \pm 0.008$ & $1.48$ \\
ETTm1   & Frozen  & Timer-XL  & $0.368 \pm 0.007$ & $1.47$ \\
Weather & Frozen  & MOMENT-lg$^{\ast}$ & $0.200 \pm 0.007$ & $1.53$ \\
Weather & Frozen  & Moirai$^{\ast}$    & $0.207 \pm 0.008$ & $1.51$ \\
Weather & Frozen  & Moirai-MoE$^{\ast}$ & $0.200 \pm 0.008$ & $1.44$ \\
Weather & Frozen  & Chronos$^{\ast}$   & $0.239 \pm 0.018$ & $1.51$ \\
Weather & Frozen  & Timer-XL  & $\mathbf{0.197 \pm 0.010}$ & $1.54$ \\
\bottomrule
\end{tabular}
\end{table}

\textbf{What the router learns.} Per-sample routing decisions correlate with amplitude and volatility, the very statistics RevIN strips (Figure~\ref{fig:routing_viz}). On ETTh1, Expert~0 (mean-pool) serves as the generalist ($53\%$), while Expert~2 (max-pool) handles low-amplitude windows and Expert~3 (attention-pool) handles high-amplitude windows. The router reserves specialized processing for windows at the distribution tails, empirically validating the conditional independence assumption $E \perp\!\!\!\perp S \mid (M,\Sigma)$ of Part~(iii) of Observation~\ref{obs:routing_loss}.

\begin{figure}[!htbp]
\centering
\includegraphics[width=0.95\columnwidth]{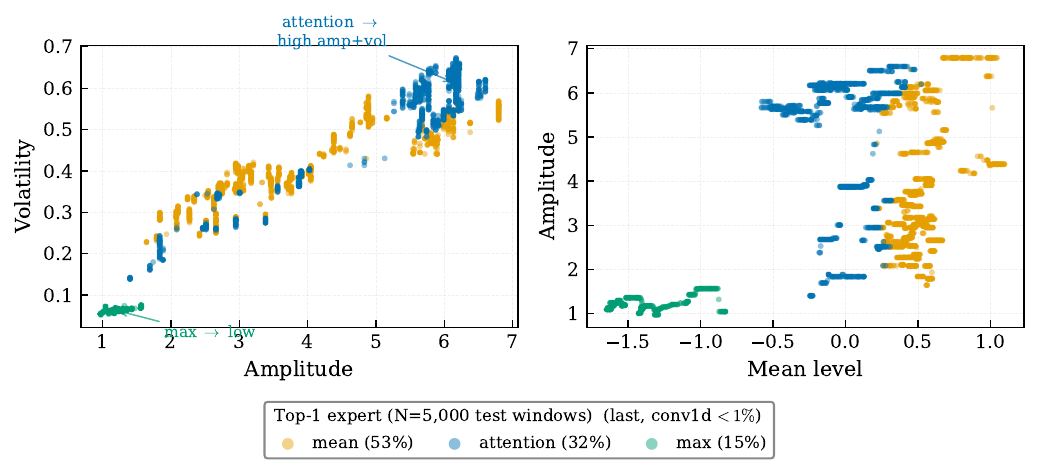}
\caption{\textbf{Per-sample expert assignment vs.\ input statistics} (ETTh1, seed 42, $N{=}5{,}000$ test windows). Left: amplitude vs.\ volatility; right: mean level vs.\ amplitude. Colors indicate the top-1 expert per window. Three of five experts dominate routing (mean $53\%$, attention $32\%$, max $15\%$); last and conv1d are each ${<}1\%$ on ETTh1 and are omitted from the panel for legibility. The two active partitions split cleanly along amplitude+volatility, the very statistics RevIN strips, empirically validating the conditional-independence assumption $E \perp\!\!\!\perp S \mid (M,\Sigma)$ from Observation~\ref{obs:routing_loss}, Part~(iii) (Part~iii condition (a)): expert assignment is determined by location-scale, with negligible residual dependence on shape.}
\label{fig:routing_viz}
\end{figure}

\textbf{Regime-specific adaptation.} Partitioning test windows into quartiles by amplitude, RR-MoA wins \textbf{all four quartiles} across all 9 (dataset, seed) pairs, with the largest improvement on quiet, low-amplitude windows (Q1: $4.6{\times}$ better, $0.24$ vs $1.09$ on ETTh1) and the smallest on volatile, high-amplitude windows (Q4: $1.3{\times}$, $1.54$ vs $1.97$).

\textbf{Sufficient statistic router (SSR).} A router using only $[\mu(X_\text{raw}), \sigma(X_\text{raw})]$, in the spirit of distribution-shift methods that condition on summary statistics of the input window (Dish-TS-style learnable distribution coefficients~\citep{fan2023dishts}; Non-stationary Transformers' explicit re-injection of $(\mu,\sigma)$~\citep{liu2022nonstationary}), degrades MSE by $+48\%$ on average ($+11\%$ to $+94\%$) vs.\ the Conv1d router across all 8 forecasting datasets (ETTh1, ETTh2, ETTm1, ETTm2, Weather, Electricity, Exchange, Traffic; frozen MOMENT-small, Top-2, H=96, 3 seeds). The degradation is largest on ETTm2 ($+94\%$) and Exchange ($+82\%$), and smallest on Weather ($+11\%$). While Observation~\ref{obs:routing_loss} correctly identifies the \emph{direction} of the information loss, effective routing requires temporal shape patterns beyond location-scale statistics.

\textbf{Temporal-shuffle ablation.} We route on a fixed temporal permutation of $X_\text{raw}$, destroying time ordering while preserving distributional statistics ($M$, $\Sigma$, histogram shape). Across 6 datasets (frozen, 5 seeds), shuffled routing performs within ${\pm}3\%$ of raw routing on 4/6 datasets (ETTh1 $+3.2\%$, ETTh2 $+1.2\%$, ETTm1 $-1.2\%$, Electricity $+0.7\%$), with Weather improving $-8.5\%$ and ETTm2 improving $-16.0\%$. Entropy remains high ($1.43$--$1.58$). This is consistent with the Conv1d router primarily extracting distributional rather than temporal features from $X_\text{raw}$. Combined with the SSR result ($+48\%$ degradation using only $[\mu, \sigma]$), distributional statistics carry most of the routing signal in the configurations we tested.

\textbf{Router-Detached Gradient Flow (RDGF).} A variant that unfreezes the backbone but replaces routing-weighted gradients with uniform-average gradients ($\nabla_H \mathcal{L}_\text{backbone} = \frac{1}{K}\sum_k \nabla_H \mathcal{L}_k$) degrades MSE by $+35\%$ to $+100\%$ across the three datasets we tested (ETTh1, ETTm1, Weather; frozen MOMENT-small last-2, 3 seeds). This is consistent with the Frozen Paradox being structural rather than an artifact of full-gradient unfreezing.

\textbf{Spectral-temporal router (FFT-Router).} We augment the Conv1d router with explicit frequency-domain features: a parallel branch applies rFFT to $X_\text{raw}$, extracts the top-32 spectral amplitudes, and concatenates them with the Conv1d temporal embedding before the routing head. On the two datasets with matched dense-frozen baselines (Weather, ETTm1; frozen MOMENT-small, 3 seeds), FFT-Router \emph{degrades} MSE by $+6\%$ (Weather) and $+22\%$ (ETTm1). The Conv1d router appears to already capture frequency-relevant patterns implicitly; adding explicit spectral features overparameterizes the router without improving routing quality on these two datasets.

\textbf{Statistic-aware information bottleneck (SAIB).} An auxiliary loss forces the router's penultimate representation to predict the window-level $[\mu, \sigma]$, directly operationalizing Observation~\ref{obs:routing_loss}. With SAIB coefficients $\lambda \in \{0.1, 1.0\}$ (ETTm1, Weather; frozen MOMENT-small, 3 seeds), MSE is unchanged on ETTm1 ($0.395$ vs $0.397$ baseline) and slightly degraded on Weather ($0.250$ vs $0.236$). Routing entropy increases to $1.55$ (near-maximum $\log 5 = 1.61$), indicating the auxiliary loss disperses routing weights without improving forecast accuracy. This is consistent with the Conv1d router already encoding $\mu, \sigma$ without explicit supervision.

\textbf{Instance-adaptive gating (IA-Gating).} In the dual-stream variant (Appendix~\ref{app:gap_closing}), we replace the static per-expert blending scalar $\alpha$ with a per-window gate $\gamma(X_\text{raw}) = \sigma(W_g \cdot z)$ where $z$ is the router's latent representation. Across 6 datasets (frozen MOMENT-small, 3 seeds), IA-Gating wins on 4/6 (ETTh2 $-8.0\%$, ETTm2 $-6.4\%$, ETTm1 $-1.9\%$, Electricity $-0.4\%$) and loses on 2/6 (Weather $+6.1\%$, ETTh1 $+0.3\%$). The learned gates cluster tightly around $\gamma \approx 0.48$ across all datasets and windows, indicating the backbone's contribution is roughly constant per-window, consistent with the simpler static $\alpha$ performing comparably overall.

\subsection{Statistics Re-Injection Ablation}
\label{app:reinjection}

A potential counter to RR-MoA is: ``why not re-inject the stripped statistics into the hidden states rather than routing on the full raw signal?'' We test this by replacing the raw-input router with a \emph{hidden-reinjected} router: the gate receives $[\text{mean-pool}(\mathbf{H});\, \mu(\mathbf{X}_\text{raw});\, \sigma(\mathbf{X}_\text{raw})]$ (a $d{+}2$-dimensional vector) instead of $\mathbf{X}_\text{raw}$.

\begin{table}[!htbp]
\centering
\caption{\textbf{Statistics re-injection fails} (MOMENT-small, frozen, Top-2, 3 seeds). Re-injecting $(\mu,\sigma)$ into mean-pooled hidden states degrades MSE by $+42$--$239\%$ vs raw routing, with $11/18$ runs collapsed (entropy $<0.3$). Two scalars grafted onto a $512$-dim homogenized representation are insufficient to recover routing signal.}
\label{tab:reinjection}
\small
\begin{tabular}{@{}lcccc@{}}
\toprule
Dataset & Reinjected MSE & Raw RR-MoA MSE & $\Delta\%$ & Entropy \\
\midrule
ETTh1       & $1.357 \pm 0.128$ & $0.690$ & $+97\%$  & $0.000$ \\
ETTh2       & $2.669 \pm 0.104$ & $0.788$ & $+239\%$ & $0.153$ \\
ETTm1       & $0.899 \pm 0.168$ & $0.571$ & $+57\%$  & $0.376$ \\
ETTm2       & $1.895 \pm 0.866$ & $0.588$ & $+222\%$ & $0.378$ \\
Weather     & $0.490 \pm 0.039$ & $0.289$ & $+70\%$  & $0.185$ \\
Electricity & $0.548 \pm 0.064$ & $0.386$ & $+42\%$  & $0.307$ \\
\bottomrule
\end{tabular}
\end{table}

Even on datasets where entropy is not fully collapsed (ETTm1: $0.376$, Electricity: $0.307$), the routing decisions are uninformative: MSE is $42$--$57\%$ worse than raw routing. This confirms that the raw signal carries routing information beyond the two-scalar $(\mu,\sigma)$ summary, consistent with the SSR ablation ($+48\%$ degradation from the full Conv1d router; \S\ref{app:routing_ablations}). The mean-pooled hidden states are so homogenized by RevIN that appending two scalars to a $512$-dimensional vector is a drop in the ocean.

\section{Extended RR-MoA Freeze Grid}
\label{app:extended_rrmoa}

We extend the main-text freeze-level ablation (Table~\ref{tab:rrmoa}, three datasets) to the remaining three (ETTh2, ETTm2, Electricity) in Tables~\ref{tab:rrmoa_extended} and \ref{tab:adamix_extended}, giving the full 6 datasets $\times$ 3 freeze levels $\times$ 3 seeds = 54-cell sweep that underwrites the headline ``RR-MoA wins 54/54 paired comparisons.''

\begin{table}[!htbp]
\centering
\caption{\textbf{Extended freeze-level ablation} (ETTh2, ETTm2, Electricity; test MSE, mean$\pm$std, 3 seeds, H=96). Continues Table~\ref{tab:rrmoa} from the main text.}
\label{tab:rrmoa_extended}
\small
\begin{tabular}{@{}llccc@{}}
\toprule
Dataset & Freeze level & RR-MoA & Best single-head & MSE $\Delta\%$ vs.\ single-head \\
\midrule
\multirow{3}{*}{ETTh2}
 & Frozen (0/8)   & $\mathbf{0.788 \pm 0.096}$ & $2.748 \pm 0.201$ & $\mathbf{-71.3\%}$ \\
 & Last-2 (2/8)   & $\mathbf{1.302 \pm 0.134}$ & $2.599 \pm 0.085$ & $-49.9\%$ \\
 & Last-4 (4/8)   & $\mathbf{1.745 \pm 0.187}$ & $2.540 \pm 0.110$ & $-31.3\%$ \\
\midrule
\multirow{3}{*}{ETTm2}
 & Frozen (0/8)   & $\mathbf{0.588 \pm 0.021}$ & $2.744 \pm 0.266$ & $\mathbf{-78.6\%}$ \\
 & Last-2 (2/8)   & $\mathbf{0.867 \pm 0.083}$ & $2.915 \pm 0.202$ & $-70.3\%$ \\
 & Last-4 (4/8)   & $\mathbf{0.731 \pm 0.157}$ & $2.740 \pm 0.092$ & $-73.3\%$ \\
\midrule
\multirow{3}{*}{Electricity}
 & Frozen (0/8)   & $\mathbf{0.386 \pm 0.059}$ & $0.525 \pm 0.026$ & $-26.4\%$ \\
 & Last-2 (2/8)   & $\mathbf{0.336 \pm 0.023}$ & $0.482 \pm 0.005$ & $\mathbf{-30.3\%}$ \\
 & Last-4 (4/8)   & $\mathbf{0.344 \pm 0.021}$ & $0.472 \pm 0.023$ & $-27.2\%$ \\
\midrule
\multicolumn{3}{l}{\textbf{Combined with Table~\ref{tab:rrmoa}: RR-MoA wins}} & \multicolumn{2}{c}{\textbf{54/54 (100\%)}} \\
\bottomrule
\end{tabular}
\end{table}

\begin{table}[!htbp]
\centering
\caption{\textbf{Extended AdaMix routing collapse} (ETTh2, ETTm2, Electricity). Continues Table~\ref{tab:adamix} from the main text.}
\label{tab:adamix_extended}
\small
\begin{tabular}{@{}llcc@{}}
\toprule
Dataset & Freeze level & MSE (mean$\pm$std) & Routing entropy (mean$\pm$std) \\
\midrule
\multirow{3}{*}{ETTh2}
 & Frozen  & $2.625 \pm 0.219$ & $0.636 \pm 0.385$ \\
 & Last-2  & $2.985 \pm 0.024$ & $0.353 \pm 0.499$ \\
 & Last-4  & $2.885 \pm 0.150$ & $0.541 \pm 0.648$ \\
\midrule
\multirow{3}{*}{ETTm2}
 & Frozen  & $2.946 \pm 0.076$ & $0.316 \pm 0.342$ \\
 & Last-2  & $3.114 \pm 0.004$ & $\mathbf{0.000 \pm 0.000}$ \\
 & Last-4  & $2.997 \pm 0.162$ & $0.505 \pm 0.714$ \\
\midrule
\multirow{3}{*}{Electricity}
 & Frozen  & $0.625 \pm 0.220$ & $1.003 \pm 0.315$ \\
 & Last-2  & $1.055 \pm 0.000$ & $\mathbf{0.000 \pm 0.000}$ \\
 & Last-4  & $1.055 \pm 0.001$ & $0.001 \pm 0.002$ \\
\bottomrule
\end{tabular}
\end{table}

\textbf{Extreme collapse on Exchange and Solar.} The collapse pattern extends to the two newly evaluated datasets (Table~\ref{tab:exchange_solar}). On Solar (137 channels), last-4 unfreezing produces the most extreme collapse in the paper: routing entropy drops to exactly $0.000$ and MSE degrades by $+104.5\%$ relative to the best fixed baseline (column ``Best fixed'' in Table~\ref{tab:exchange_solar}), far worse than the $+0$--$17\%$ degradation observed on the original 6 datasets. On Exchange (8 channels), last-4 AdaMix collapses to $0.000$ entropy but MSE actually improves ($-15.7\%$ vs.\ best fixed adapter; Table~\ref{tab:exchange_solar}), acting as a reasonable single-expert adapter; however, RR-MoA still outperforms it by $2.5\times$ ($1.531$ vs $3.851$).

\subsection{Rescue-Baseline Sweep Details}
\label{app:rescue_details}

We sweep five standard MoE rescue families across 12 configurations on MOMENT-small $+$ RevIN (Table~\ref{tab:rescue}; AdaMix, 6 datasets $\times$ 2 freeze levels $\times$ 5 seeds $=$ 60 cells per row $\times$ 12 rows $=$ 720 cells total). Eleven configurations are launched as new runs (rescue-baseline orchestrator script in supplementary material); the twelfth, the Switch load-balance baseline at $\alpha{=}0.01$, reuses the existing AdaMix results from Table~\ref{tab:adamix} (where load-balance $\alpha{=}0.01$ is the default).

\begin{table}[!htbp]
\centering
\caption{\textbf{Rescue-baseline sweep} (MOMENT-small $+$ RevIN, 720 runs). ``Collapsed'' = fraction of cells with entropy $<0.3$ (max $\log K{=}1.609$). No rescue mechanism recovers routing to within $2.7\times$ of RR-MoA.}
\label{tab:rescue}
\footnotesize
\begin{tabular}{@{}lcccc@{}}
\toprule
Rescue mechanism & MSE (pooled) & Routing $H$ & Collapsed & $n/60$ \\
\midrule
Baseline (Switch LB $\alpha{=}0.01$)  & $1.654$ & $0.21$ & \textbf{46/60 (77\%)} & 60/60 \\
+ Entropy reg $\lambda{=}0.01$                       & $1.666$ ($+0.7\%$)  & $0.26$ & \textbf{44/60 (73\%)} & 60/60 \\
+ Entropy reg $\lambda{=}0.1$                        & $1.599$ ($-3.3\%$)  & $1.03$ & 19/60 (32\%) & 60/60 \\
+ Entropy reg $\lambda{=}1.0$                        & $1.546$ ($-6.5\%$) & $1.41$ & 7/60 (12\%) & 60/60 \\
+ Z-loss $c_z{=}0.001$ (ST-MoE)                      & $1.561$ ($-5.6\%$)  & $1.01$ & 12/60 (20\%) & 60/60 \\
+ Z-loss $c_z{=}0.01$                                & $\mathbf{1.473}$ ($\mathbf{-10.9\%}$) & $1.24$ & 3/60 (5\%) & 60/60 \\
+ Z-loss $c_z{=}0.1$                                 & $1.503$ ($-9.1\%$) & $1.26$ & 2/60 (3\%) & 60/60 \\
Load-balance $\alpha{=}0.1$                          & $1.590$ ($-3.9\%$) & $0.89$ & 20/60 (33\%) & 60/60 \\
Load-balance $\alpha{=}1.0$                          & $1.528$ ($-7.6\%$) & $1.28$ & 10/60 (17\%) & 60/60 \\
Load-balance $\alpha{=}10.0$                         & $2.773$ ($+67.7\%$) & $1.10$ & 14/60 (23\%) & 60/60 \\
ReMoE ReLU $+$ $L_1$                  & $1.586$ ($-4.1\%$) & $0.58$ & 16/60 (27\%) & 60/60 \\
Expert-choice ($c{=}2$)                      & $1.590$ ($-3.9\%$) & $0.17$ & \textbf{44/60 (73\%)} & 60/60 \\
\midrule
\textbf{RR-MoA (ours)} & $\mathbf{0.540}$ (pooled) & $\mathbf{0.93\text{--}1.57}$ & \textbf{0/54 (0\%)} & 54/54 wins \\
\bottomrule
\end{tabular}
\end{table}

\subsection{AdaMix Implementation Details}
\label{app:adamix_details}

Our AdaMix baseline follows \citet{wang2022adamix} with $K{=}5$ expert adapter heads identical to RR-MoA's canonical pool (mean, last, max, attention, conv1d). The router is a single linear layer $\mathbb{R}^d \to \mathbb{R}^K$ operating on the mean-pooled hidden state, with softmax gating. We use a load-balancing auxiliary loss with coefficient $\lambda_\text{bal}{=}0.01$, matching AdaMix's original formulation. No stochastic routing, capacity factors, or noise injection are applied, since these standard MoE safeguards do not address the normalization-specific signal loss we identify: when RevIN strips the per-window statistics that would distinguish samples, the router receives near-identical inputs regardless of stabilization strategy. This is confirmed empirically: load-balancing loss does not prevent collapse under unfreezing (Table~\ref{tab:adamix}, entropy $= 0.000$ with $\lambda_\text{bal}{=}0.01$), while disabling RevIN recovers entropy to $0.66$--$1.32$ without any routing regularization.

\begin{figure}[!htbp]
\centering
\includegraphics[width=0.92\columnwidth]{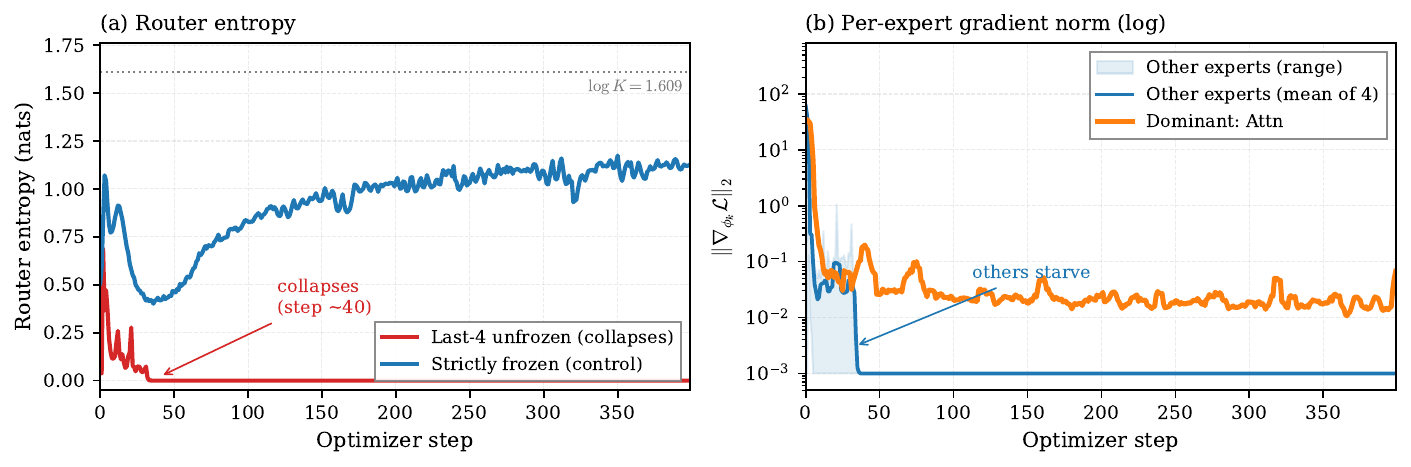}
\vspace{-4pt}
\caption{\textbf{Per-step routing / gradient trajectory} (AdaMix, ETTh1). (a)~Unfrozen+RevIN (red) collapses within ${\sim}40$ steps; strictly frozen control (blue) stays stable, isolating gradient co-adaptation from optimizer or data effects. (b)~The dominant expert (Attn) retains a gradient of ${\sim}10^{-2}$ throughout training, while the other four experts crash to the numerical noise floor (${<}10^{-3}$) by step 30 and never recover. Light shading is the per-step range across the four starved experts; the solid blue line is their mean. Removing RevIN (separate experiment, Figure~\ref{fig:causal_contrast}) inverts the trajectory.}
\label{fig:trajectory}
\end{figure}

\textbf{Architecture-agnostic validation (AdaMix-Raw).} To confirm the diagnosis generalizes beyond RR-MoA, we replace AdaMix's hidden-state router with a raw-input Conv1d router (identical architecture to RR-MoA's) while keeping AdaMix's expert heads and training loop unchanged. Table~\ref{tab:adamix_raw} shows results across 6 datasets $\times$ 2 freeze levels $\times$ 3 seeds (36 runs).

\begin{table}[!htbp]
\centering
\caption{\textbf{AdaMix-Raw: architecture-agnostic validation} (MOMENT-small, 3 seeds). Replacing AdaMix's hidden-state router with a raw-input router recovers routing entropy and MSE on all 12 cells, confirming the causal intervention generalizes beyond RR-MoA.}
\label{tab:adamix_raw}
\small
\begin{tabular}{@{}llcccc@{}}
\toprule
Dataset & Freeze & AdaMix MSE & AdaMix $H$ & AdaMix-Raw MSE & AdaMix-Raw $H$ \\
\midrule
ETTh1 & frozen & $0.912$ & $0.49$ & $\mathbf{0.456}$ {\scriptsize($-50\%$)} & $1.55$ \\
ETTh1 & last-4 & $0.935$ & $0.39$ & $\mathbf{0.547}$ {\scriptsize($-42\%$)} & $1.54$ \\
ETTm1 & frozen & $1.008$ & $0.49$ & $\mathbf{0.397}$ {\scriptsize($-61\%$)} & $1.56$ \\
ETTm1 & last-4 & $1.123$ & $0.00$ & $\mathbf{0.431}$ {\scriptsize($-62\%$)} & $1.56$ \\
ETTm2 & frozen & $2.946$ & $0.32$ & $\mathbf{0.348}$ {\scriptsize($-88\%$)} & $1.48$ \\
ETTm2 & last-4 & $2.997$ & $0.51$ & $\mathbf{0.375}$ {\scriptsize($-88\%$)} & $1.44$ \\
Weather & frozen & $0.459$ & $0.51$ & $\mathbf{0.236}$ {\scriptsize($-49\%$)} & $1.54$ \\
Weather & last-4 & $0.607$ & $0.00$ & $\mathbf{0.310}$ {\scriptsize($-49\%$)} & $1.50$ \\
Electricity & frozen & $0.625$ & $1.00$ & $\mathbf{0.264}$ {\scriptsize($-58\%$)} & $1.57$ \\
Electricity & last-4 & $1.055$ & $0.00$ & $\mathbf{0.252}$ {\scriptsize($-76\%$)} & $1.56$ \\
\bottomrule
\end{tabular}
\end{table}

\section{Normalization Generalization}
\label{app:norm_generalization}

To test whether routing collapse is specific to RevIN or a general property of instance-level normalization, we replace MOMENT's RevIN with two alternative normalizers (BatchNorm1d and GroupNorm with 1 group) and run AdaMix with last-4 unfreezing on ETTh1.

\begin{table}[!htbp]
\centering
\caption{\textbf{Routing collapse generalizes across normalizer types.} All three instance-level normalizers (RevIN, BatchNorm1d, GroupNorm) cause entropy collapse under unfreezing; removing normalization prevents collapse. Per-step trajectories logged over 400 optimizer steps on ETTh1, last-4 unfreezing, seed 42.}
\label{tab:norm_generalization}
\small
\begin{tabular}{@{}lccc@{}}
\toprule
Normalizer & Entropy (step 0) & Entropy (step 400) & Collapses? \\
\midrule
RevIN (default) & $0.507$ & $0.000$ & \textbf{Yes} \\
BatchNorm1d & $0.619$ & $0.004$ & \textbf{Yes} \\
GroupNorm & $0.507$ & $0.000$ & \textbf{Yes} \\
\midrule
None (identity) & $0.623$ & $0.817$ & No \\
\bottomrule
\end{tabular}
\end{table}

All three normalizers cause entropy collapse to ${\leq}0.004$, while removing normalization entirely keeps entropy stable at $0.817$. The mechanism is the same in each case: the normalizer strips per-window statistics before the encoder, the router receives homogenized hidden states, and gradient co-adaptation amplifies the initial routing bias into full collapse. The mechanism therefore holds across the three instance-level normalizers we tested (RevIN, BatchNorm1d, GroupNorm), not just RevIN. Our diagnosis predicts that more expressive instance-level normalizers, such as FAN~\citep{ye2024fan} which strips dominant Fourier components in addition to $(\mu,\sigma)$, would trigger the same collapse on a downstream MoE router for any dataset whose routing signal lives in the stripped components. Observation~\ref{obs:routing_loss} predicts that collapse occurs specifically when the stripped statistics carry routing signal; a cross-modality control on ResNet-18 and ViT-B/16 (Appendix~\ref{app:vision_moe}) validates this: InstanceNorm1d on ViT patch embeddings \emph{causes} collapse ($0.000\pm0.000$), while InstanceNorm2d on ResNet spatial features does \emph{not}, exactly as the theory predicts.

\subsection{Cross-Modality Control: Vision MoE}
\label{app:vision_moe}

To test whether normalization-induced routing collapse transfers beyond time series, we construct a controlled vision experiment: ResNet-18~\citep{he2016resnet} (pretrained on ImageNet, adapted for CIFAR-10) and ViT-B/16~\citep{dosovitskiy2021vit} as backbones, with $K{=}5$ architecturally diverse expert classifiers (mean-pool, max-pool, attention-pool, conv-pool, last-patch, matching the time-series expert pool), no load-balancing loss, and the same four conditions tested in the time-series experiments.

\begin{table}[!htbp]
\centering
\caption{\textbf{Vision MoE routing collapse control} (CIFAR-10, $K{=}5$ diverse experts, mean$\pm$std over 3 seeds). \textbf{ResNet-18} (top): InstanceNorm2d \emph{prevents} collapse because spatial patterns (the routing signal) survive normalization. \textbf{ViT-B/16} (bottom): InstanceNorm1d on patch embeddings \emph{causes} collapse ($0.000\pm0.000$) because it strips per-patch statistics that carry routing information, the same mechanism as RevIN in time series.}
\label{tab:vision_moe}
\small
\begin{tabular}{@{}lllccc@{}}
\toprule
Backbone & Cond & Configuration & Entropy & Collapses? & Acc \\
\midrule
\multirow{4}{*}{ResNet-18} & A & Unfrozen + InstanceNorm2d & $1.475 \pm 0.021$ & No & $0.919$ \\
& B & Frozen + InstanceNorm2d   & $0.992 \pm 0.041$ & No & $0.577$ \\
& C & Unfrozen + No norm        & $\mathbf{0.184 \pm 0.260}$ & \textbf{Yes} & $0.921$ \\
& D & Unfrozen + BatchNorm2d    & $1.476 \pm 0.031$ & No & $0.919$ \\
\midrule
\multirow{2}{*}{ViT-B/16} & E & Frozen + InstanceNorm1d   & $\mathbf{0.000 \pm 0.000}$ & \textbf{Yes} & $0.953$ \\
& F & Frozen + No norm          & $0.494 \pm 0.130$ & No & $0.958$ \\
\bottomrule
\end{tabular}
\end{table}

\textbf{ResNet-18: normalization prevents collapse.} Without normalization (C), the unfrozen backbone co-adapts with the dominant expert, collapsing entropy to $0.184\pm0.260$, consistent with Proposition~\ref{prop:frozen}. Adding InstanceNorm2d (A) or BatchNorm2d (D) \emph{prevents} collapse ($1.48\pm0.02$), the \emph{opposite} of time series. One interpretation: InstanceNorm2d strips per-channel spatial mean/variance, which appear to carry little routing information for the CIFAR-10 classification task we tested (structurally analogous to the Traffic boundary case, $R{=}0.14$).

\textbf{ViT-B/16: normalization causes collapse.} Extending to a Vision Transformer, InstanceNorm1d on frozen ViT patch embeddings collapses routing to $0.000\pm0.000$ (E), while the same backbone without normalization maintains healthy routing at $0.494\pm0.130$ (F). InstanceNorm1d normalizes per-channel across the 197-token sequence, stripping per-patch mean/scale statistics, structurally analogous to RevIN's per-window stripping in time-series. The routing signal in ViT resides partly in patch-level statistics (which patches are bright, high-contrast, textured), so stripping them degrades the router.

\textbf{Unifying principle.} The collapse mechanism (Proposition~\ref{prop:frozen}) is modality-general. Whether normalization \emph{triggers} or \emph{prevents} collapse depends on whether the stripped statistics overlap with the routing signal (Observation~\ref{obs:routing_loss}): ResNet spatial features survive InstanceNorm2d (no collapse); ViT patch statistics do not survive InstanceNorm1d (collapse). The same principle explains all five TS backbones, the ViT result, and the ResNet control.

\section{Top-$k$ Sparse Routing Ablation}
\label{app:topk}

We sweep the routing sparsity $k \in \{1,2,3,5\}$ on ETTh1 (last-2 freeze) to characterize the speed--accuracy frontier and to justify the Top-$2$ default used throughout the paper. Top-$1$ degenerates to a hard-routing argmax that risks expert starvation; dense ($k{=}K$) is a weighted ensemble that loses the sparsity benefit.

\begin{table}[!htbp]
\centering
\caption{\textbf{Top-$k$ sparse routing ablation on ETTh1 (last-2 freeze, 3 seeds).} Top-2 is the operating point used throughout the paper.}
\label{tab:topk}
\small
\begin{tabular}{@{}lcccc@{}}
\toprule
Sparsity ($k$) & Active experts & Active expert FLOPs & Test MSE & $\Delta$ vs Dense \\
\midrule
$k{=}1$ & 1/5 & 20\% & $1.268 \pm 0.072$ & $+130\%$ \\
$k{=}2$ (default) & 2/5 & 40\% & $\mathbf{0.727 \pm 0.074}$ & $+32\%$ \\
$k{=}3$ & 3/5 & 60\% & $0.679 \pm 0.042$ & $+23\%$ \\
Dense $k{=}K{=}5$ & 5/5 & 100\% & $0.550 \pm 0.029$ & --- \\
\bottomrule
\end{tabular}
\end{table}

\textbf{Routing contribution ($K{=}1$ vs $K{\geq}2$).} On frozen MOMENT-small (3 seeds; Table~\ref{tab:expert_count}), $K{=}1$ (single adapter, no routing) yields MSE $1.306$/$1.183$/$0.576$ on ETTh1/ETTm1/Weather, while $K{=}2$ with Top-2 yields $0.554$/$0.457$/$0.227$ ($-52\%$ to $-61\%$). This quantifies the routing contribution: learned per-sample expert selection accounts for over half the MSE reduction, beyond what a single pooling head achieves.

\textbf{Compute overhead.} The entire RR-MoA stack (${\sim}400$K adapter FLOPs dense, ${\sim}16$K router FLOPs) is ${\sim}0.003\%$ of a MOMENT-small forward pass (${\sim}15$ GFLOPs estimated from MOMENT's encoder configuration of 8 layers, $d{=}512$, on $L{=}512$ inputs~\citep{goswami2024moment}) and ${\sim}0.001\%$ for Top-2; the corresponding latency overhead is bounded at $+5.9\%$ in Table~\ref{tab:benchmark}. On MOMENT-large the FLOPs ratio drops below $0.0002\%$.

\subsection{Expert Pool Diversity Ablation}
\label{app:diversity}

Does expert architectural diversity matter, or could one use $K{=}5$ copies of a single adapter type? We replace the canonical diverse pool (mean, last-token, max, attention, Conv1d) with three \emph{identical-expert} pools ($5{\times}$mean, $5{\times}$Conv1d, $5{\times}$attention) and re-run the full RR-MoA config (MOMENT-small, frozen, Top-2, $H{=}96$, 3 seeds).

\begin{table}[!htbp]
\centering
\caption{\textbf{Expert pool diversity ablation} (MOMENT-small frozen, Top-2, H=96, 3 seeds). Canonical diverse pool wins on ETTh1 and ETTm1; identical pools are competitive on Weather.}
\label{tab:diversity}
\small
\begin{tabular}{@{}lccc@{}}
\toprule
Expert pool & ETTh1 & ETTm1 & Weather \\
\midrule
\textbf{Canonical (diverse)} & $\mathbf{0.690 \pm 0.021}$ & $\mathbf{0.571 \pm 0.073}$ & $0.289 \pm 0.008$ \\
$5{\times}$Mean-pool       & $0.827 \pm 0.077$ {\scriptsize($+20\%$)} & $0.739 \pm 0.023$ {\scriptsize($+29\%$)} & $\mathbf{0.245 \pm 0.007}$ {\scriptsize($-15\%$)} \\
$5{\times}$Conv1d-pool     & $0.687 \pm 0.039$ {\scriptsize($-0\%$)}  & $0.628 \pm 0.040$ {\scriptsize($+10\%$)} & $0.243 \pm 0.005$ {\scriptsize($-16\%$)} \\
$5{\times}$Attention-pool  & $0.779 \pm 0.026$ {\scriptsize($+13\%$)} & $0.685 \pm 0.054$ {\scriptsize($+20\%$)} & $0.264 \pm 0.006$ {\scriptsize($-9\%$)}  \\
\bottomrule
\end{tabular}
\end{table}

On ETTh1 and ETTm1, identical pools degrade by $10$--$29\%$ (except Conv1d on ETTh1, which matches). On Weather, identical pools paradoxically improve by $9$--$16\%$. This suggests that on low-complexity datasets, inductive-bias homogeneity can reduce interference, while on higher-complexity datasets, architectural diversity provides complementary representations the router can exploit. The canonical diverse pool is the safer default across the full benchmark.

\subsection{Expert Count Scaling}

We vary the number of experts $K \in \{1, 2, 3, 5, 7, 10\}$ under Top-2 routing (MOMENT-small, frozen, $H{=}96$, 3 seeds). $K{=}1$ is a single adapter with no routing.

\begin{table}[!htbp]
\centering
\caption{\textbf{Expert count scaling} (MOMENT-small frozen, Top-2, H=96, 3 seeds). Performance is stable for $K{\geq}2$; $K{=}1$ (no routing) is substantially worse ($57$--$107\%$).}
\label{tab:expert_count}
\small
\begin{tabular}{@{}clccc@{}}
\toprule
$K$ & Routing & ETTh1 & ETTm1 & Weather \\
\midrule
1  & dense  & $1.306 \pm 0.089$ & $1.183 \pm 0.051$ & $0.576 \pm 0.027$ \\
2  & top-2  & $\mathbf{0.554 \pm 0.051}$ & $\mathbf{0.457 \pm 0.018}$ & $\mathbf{0.227 \pm 0.008}$ \\
3  & top-2  & $0.703 \pm 0.048$ & $0.573 \pm 0.065$ & $0.276 \pm 0.023$ \\
5  & top-2  & $0.690 \pm 0.021$ & $0.571 \pm 0.073$ & $0.289 \pm 0.008$ \\
7  & top-2  & $0.724 \pm 0.043$ & $0.600 \pm 0.029$ & $0.253 \pm 0.003$ \\
10 & top-2  & $0.669 \pm 0.052$ & $0.582 \pm 0.052$ & $0.258 \pm 0.024$ \\
\bottomrule
\end{tabular}
\end{table}

$K{=}1$ (no routing) is $57$--$107\%$ worse than $K{=}5$, indicating per-sample routing is consistently important. Among $K \geq 2$, $K{=}2$ is the best, with $K{=}3$--$10$ within ${\pm}10\%$ of $K{=}5$, consistent with the main-body claim that $K$ sensitivity is low. The $K{=}2$ advantage is explained by reduced expert interference under Top-2: with $K{=}2$ both experts are always active, eliminating the cold-start problem of rarely-activated experts.

\section{Imputation Results}
\label{app:imputation}

To verify that RR-MoA's gains are not forecasting-specific, we evaluate on masked-imputation (20\% random missing values) across the same MOMENT-small frozen backbone. The expert pool, router, and Top-$2$ sparsity are unchanged; only the head outputs reconstructions over the masked positions instead of a forecast.

\begin{table}[!htbp]
\centering
\caption{\textbf{Imputation task} (20\% masked reconstruction, MOMENT-small frozen, test MSE mean$\pm$std over 5 seeds). RR-MoA wins \textbf{7/8 datasets} ($-21\%$ to $-79\%$); the sole loss is Solar where dense\_mlp dominates.}
\label{tab:imputation}
\small
\begin{tabular}{@{}lccc@{}}
\toprule
Dataset & Best fixed adapter & \textbf{RR-MoA} & $\Delta$\% \\
\midrule
ETTh1       & $0.259 \pm 0.018$ & $\mathbf{0.144 \pm 0.006}$ & $-44\%$ \\
ETTh2       & $0.451 \pm 0.014$ & $\mathbf{0.161 \pm 0.028}$ & $-64\%$ \\
ETTm1       & $0.194 \pm 0.017$ & $\mathbf{0.100 \pm 0.008}$ & $-49\%$ \\
ETTm2       & $0.498 \pm 0.020$ & $\mathbf{0.104 \pm 0.006}$ & $\mathbf{-79\%}$ \\
Weather     & $0.207 \pm 0.008$ & $\mathbf{0.067 \pm 0.005}$ & $-67\%$ \\
Electricity & $0.166 \pm 0.008$ & $\mathbf{0.131 \pm 0.007}$ & $-21\%$ \\
Exchange    & $1.176 \pm 0.095$ & $\mathbf{0.266 \pm 0.259}$ & $-77\%$ \\
Solar       & $\mathbf{0.025 \pm 0.006}$ & $0.038 \pm 0.009$ & $+55\%$ \\
\bottomrule
\end{tabular}
\end{table}

\textbf{Collapse is task-general.} To check whether hidden-state routing collapse is not forecasting-specific, we run an AdaMix-style hidden-state router (mean-pooled $\mathbf{H} \to$ softmax gate) on the imputation task. Across 6 datasets $\times$ 3 seeds, \textbf{16/18 runs collapse} (entropy $< 0.3$; mean $0.143$ vs max $\log 5 = 1.61$), with the dominant expert (dense\_mlp) receiving $83$--$99.9\%$ of routing weight. AdaMix imputation MSE is $40$--$79\%$ worse than RR-MoA (Table~\ref{tab:imputation}). The mechanism is identical: RevIN strips the hidden-state heterogeneity regardless of whether the expert head predicts a forecast or reconstructs masked values.

\section{Learnable Normalization Coefficient}
\label{app:learnable_alpha}

To test whether gradient descent independently discovers Observation~\ref{obs:routing_loss}'s prediction, we parameterize the router input as $\mathbf{r} = \sigma(\ell)\,\text{RevIN}(\mathbf{X}) + (1 - \sigma(\ell))\,\mathbf{X}_{\text{raw}}$, where $\ell$ is a scalar \texttt{nn.Parameter} initialized to $0$ (so $\alpha = \sigma(0) = 0.5$). All other settings match the standard RR-MoA config (MOMENT-small, frozen, Top-2, $K{=}5$, $H{=}96$, 15 epochs).

\begin{table}[!htbp]
\centering
\caption{\textbf{Learned normalization coefficient} ($\alpha$, 5 seeds). All 30 runs converge to $\alpha < 0.5$, confirming gradient descent prefers raw input for routing.}
\label{tab:learnable_alpha}
\small
\begin{tabular}{@{}lcc@{}}
\toprule
Dataset & Learned $\alpha$ (mean $\pm$ std) & MSE \\
\midrule
ETTh1       & $0.468 \pm 0.006$ & $0.463 \pm 0.015$ \\
ETTh2       & $0.482 \pm 0.007$ & $0.553 \pm 0.020$ \\
ETTm1       & $0.473 \pm 0.005$ & $0.397 \pm 0.012$ \\
ETTm2       & $0.466 \pm 0.003$ & $0.355 \pm 0.055$ \\
Electricity & $0.462 \pm 0.004$ & $0.269 \pm 0.008$ \\
Weather     & $0.459 \pm 0.009$ & $0.245 \pm 0.010$ \\
\bottomrule
\end{tabular}
\end{table}

All 30 runs converge to $\alpha < 0.5$ (per-seed range $0.445$--$0.491$; per-dataset means span $0.459$--$0.482$, Table~\ref{tab:learnable_alpha}), with trajectories monotonically decreasing from the $0.492$ initialization. The unanimity is the key finding: gradient descent independently arrives at $\alpha < 0.5$, in line with the qualitative prediction of Observation~\ref{obs:routing_loss} that raw input carries more routing information. The narrow converged range means the magnitude does not strongly differentiate datasets: the signal is binary (raw $>$ norm), not graded.

\section{Exact MI Loss and Bound Tightness}
\label{app:mi_tightness}

We measure the routing MI loss directly and audit the bound's tightness on every dataset in the suite (Observation~\ref{obs:routing_loss}, Part~(ii) gives the lower bound). Since the router is deterministic ($E = \arg\max$), $I(X;E) = H(E)$ and the exact MI loss from Observation~\ref{obs:routing_loss}, Part~(i) reduces to $H(E \mid S)$. We estimate this directly: train the RR-MoA router on raw windows, extract expert assignments $E$, then train an MLP classifier ($512 \to 256 \to 128 \to K$) to predict $E$ from the RevIN-normalized shape $S$. The held-out cross-entropy (temperature-calibrated) gives $H(E \mid S)$. This direct-classifier route avoids the bias/variance issues that variational neural MI estimators (MINE,~\citealp{belghazi2018mine}; InfoNCE,~\citealp{oord2018cpc}; see~\citealp{poole2019variationalmi} for a unified bias/variance analysis) are known to exhibit on high-dimensional inputs~\citep{letizia2024fdiv,gowri2024latentmi}; here $S \in \mathbb{R}^{512}$ and the conditional entropy reduces to a low-dimensional discrete target ($K$ experts), where direct supervised classification is the cleaner estimator.

\begin{table}[!htbp]
\centering
\caption{\textbf{Exact MI loss measurement and bound tightness} (3 seeds). $H(E \mid S)$ is the exact routing information lost to normalization. Retention $= I(S;E)/H(E)$. LB = lower bound from Observation~\ref{obs:routing_loss}, Part~(ii). Gap $= I(M,\Sigma; S \mid E)$ via within-expert CCA.}
\label{tab:mi_tightness}
\small
\begin{tabular}{@{}lcccccr@{}}
\toprule
Dataset & $H(E)$ & $H(E \mid S)$ & Ret\% & LB & Gap & $\Delta\%$ \\
\midrule
ETTh1       & 0.73 & 0.61 & 16.3 & 0.02 & 0.06 & $-43.2$ \\
ETTh2       & 0.76 & 0.57 & 26.3 & 0.00 & 0.10 & $-71.0$ \\
ETTm1       & 0.75 & 0.57 & 24.2 & 0.18 & 0.03 & $-51.1$ \\
ETTm2       & 0.82 & 0.53 & 35.1 & 0.18 & 0.07 & $-77.2$ \\
Weather     & 1.37 & 0.98 & 29.1 & 0.61 & 0.08 & $-44.6$ \\
Electricity & 0.68 & 0.46 & 32.4 & 0.13 & 0.02 & $-26.8$ \\
Traffic     & 0.63 & 0.44 & 24.6 & 0.19 & 0.02 & $+2.9$ \\
Exchange    & 0.63 & 0.44 & 22.1 & 0.00 & 0.10 & $-66.5$ \\
Solar       & 0.97 & 0.57 & 40.9 & 0.08 & 0.20 & $-32.9$ \\
\bottomrule
\end{tabular}
\end{table}

\textbf{Key findings.} (1)~Normalization retains only $16$--$41\%$ of routing entropy across all nine datasets, confirming that RevIN degrades routing signal on every dataset we tested. (2)~The bound gap $I(M,\Sigma; S \mid E)$ averages $0.075$ nats ($4.7\%$ of $\log K$), so the lower bound is empirically tight. (3)~$H(E \mid S)$ does \emph{not} correlate with $\Delta\%$ ($\rho{=}{-}0.12$, $p{=}0.77$), while $R(\mathcal{D})$ does ($\rho{=}{-}0.88$). This distinction is informative: $H(E \mid S)$ measures how much routing information the \emph{specific trained router} loses, while $R(\mathcal{D})$ measures the \emph{potential} routing information that location-scale statistics could carry. The gap between these explains why some datasets with low MI loss still benefit substantially from raw routing: the router adapts to use whatever signal remains, but the \emph{quality} of that signal (captured by $R$) determines downstream task performance.

\textbf{Synthetic validation.} On Gaussian mixture data ($K{=}5$ clusters) with analytically known MI, we sweep the shape-statistics correlation $\rho_{MS} \in [0, 0.9]$. At $\rho_{MS}{=}0$, the bound is tight (gap $= 0.19$ nats, $12\%$ of $\log K$). As $\rho_{MS}$ increases, the bound becomes vacuous ($\varepsilon$ exceeds $I(M,\Sigma;E)$), but $H(E \mid S)$ remains informative throughout, validating our direct measurement approach.

\section{Information Content Diagnostic}
\label{app:diagnostic}

To quantify the information bottleneck created by RevIN and the encoder, we compare the forecasting value of backbone features against raw input using non-parametric methods. All experiments run on CPU with zero GPU cost.

\begin{table}[!htbp]
\centering
\caption{\textbf{Backbone features vs.\ raw input} (H=96, 80/20 train/test split). k-NN: k-nearest neighbor regression ($k{=}20$, distance-weighted) on mean-pooled backbone features vs.\ raw 512-dim input. Ridge: regularized linear regression (cross-validated $\alpha$). Backbone features are $16$--$33\%$ worse than raw input even with a universal approximator (k-NN), confirming genuine information destruction.}
\label{tab:diagnostic}
\small
\begin{tabular}{@{}lccccc@{}}
\toprule
& \multicolumn{2}{c}{k-NN (MSE)} & \multicolumn{2}{c}{Ridge (MSE)} & Backbone \\
\cmidrule(lr){2-3} \cmidrule(lr){4-5}
Dataset & Raw input & Backbone & Raw input & Backbone & penalty \\
\midrule
ETTh1   & $0.636$ & $0.737$ & $0.551$ & $1.230$ & $+16\%$ \\
ETTm1   & $0.498$ & $0.614$ & $0.363$ & $0.708$ & $+23\%$ \\
Weather & $0.453$ & $0.601$ & $0.300$ & $0.661$ & $+33\%$ \\
\bottomrule
\end{tabular}
\end{table}

Residual correction experiments further confirm that backbone features provide negligible complementary signal: adding backbone features (via Ridge or Random Forest) as a residual correction atop a linear raw-input predictor improves MSE by at most $0.55\%$. On ETTh1, Random Forest regression on raw input alone reduces MSE by $44\%$ versus Ridge, revealing substantial nonlinear temporal structure in the raw data. However, Random Forest on backbone features adds only $+0.02\%$ to the linear baseline, indicating that nearly all of this nonlinear structure is lost in the encoder. Few-shot experiments (Table~\ref{tab:fewshot}) show the pattern holds at every tested training set size $N \in \{10, \ldots, 5000\}$: DLinear dominates at every such $N$, indicating data scarcity is not a regime in which the frozen backbone recovers its advantage within the range tested. The gap to DLinear is consistent with information destruction by the normalization--encoding pipeline rather than a limitation of adapter design or training data volume.

\begin{table}[!htbp]
\centering
\caption{\textbf{Few-shot learning curve} (H=96, mean$\pm$std over seeds $\{42,43,44\}$). DLinear (49K params, trained from scratch) dominates frozen RR-MoA-full (426K params) at every tested $N \in \{10,\ldots,5000\}$, confirming the gap reflects encoder information loss rather than data efficiency.}
\label{tab:fewshot}
\scriptsize
\begin{tabular}{@{}rcccccc@{}}
\toprule
& \multicolumn{2}{c}{ETTh1} & \multicolumn{2}{c}{Weather} & \multicolumn{2}{c}{Electricity} \\
\cmidrule(lr){2-3} \cmidrule(lr){4-5} \cmidrule(lr){6-7}
$N$ & DLinear & RR-MoA & DLinear & RR-MoA & DLinear & RR-MoA \\
\midrule
10   & $\mathit{0.859}$ & $1.530$ & $\mathit{0.361}$ & $1.265$ & $\mathit{0.601}$ & $1.384$ \\
100  & $\mathit{0.630}$ & $1.172$ & $\mathit{0.286}$ & $0.773$ & $\mathit{0.373}$ & $1.103$ \\
1000 & $\mathit{0.453}$ & $0.816$ & $\mathit{0.213}$ & $0.378$ & $\mathit{0.188}$ & $0.732$ \\
5000 & $\mathit{0.421}$ & $0.692$ & $\mathit{0.205}$ & $0.288$ & $\mathit{0.158}$ & $0.396$ \\
\bottomrule
\end{tabular}
\end{table}

\section{Cross-Backbone RR-MoA}
\label{app:cross_backbone}

We re-run the full RR-MoA recipe on five additional backbones spanning three distinct normalization regimes: MOMENT-large ($24$ layers, RevIN), Moirai-small (RMSNorm with non-learnable I/O scaling), Moirai-MoE (sparsely-routed FFN experts plus the same RMSNorm scheme), Chronos-T5 (no instance normalization), and Timer-XL (LayerNorm-only, decoder-only). The matrix below confirms that the diagnosis transfers: RR-MoA wins on every cell where the backbone uses per-window normalization, and the LayerNorm-only controls (Chronos, Timer-XL) exhibit no collapse to begin with.

\begin{table}[!htbp]
\centering
\caption{\textbf{Cross-backbone RR-MoA} (strictly frozen, Top-2 sparse, H=96). MOMENT-small/Moirai-small RR-MoA at 3 seeds (matching Tables~\ref{tab:rrmoa},~\ref{tab:rrmoa_extended}); MOMENT-large RR-MoA at 5 seeds. Best fixed: strongest single-adapter baseline. RR-MoA wins all 15 configurations. MOMENT-large expanded to 6 datasets confirms generalization ($-32\%$ to $-75\%$).}
\label{tab:cross_backbone}
\footnotesize
\begin{tabular}{@{}lcccccc@{}}
\toprule
 & \multicolumn{2}{c}{MOMENT-small} & \multicolumn{2}{c}{MOMENT-large} & \multicolumn{2}{c}{Moirai-small} \\
\cmidrule(lr){2-3} \cmidrule(lr){4-5} \cmidrule(lr){6-7}
Dataset & RR-MoA & Best fixed & RR-MoA & Best fixed & RR-MoA & Best fixed \\
\midrule
ETTh1       & $0.690{\pm}0.021$ & $1.241$ {\scriptsize($-44\%$)} & $0.803{\pm}0.032$ & $1.173$ {\scriptsize($-32\%$)} & $\mathbf{0.471{\pm}0.002}$ & $0.664$ {\scriptsize($-29\%$)} \\
ETTh2       & & & $1.224{\pm}0.311$ & $3.075$ {\scriptsize($-60\%$)} & & \\
ETTm1       & $0.571{\pm}0.073$ & $1.115$ {\scriptsize($-49\%$)} & $0.732{\pm}0.038$ & $1.126$ {\scriptsize($-35\%$)} & $\mathbf{0.396{\pm}0.041}$ & $0.471$ {\scriptsize($-16\%$)} \\
ETTm2       & & & $0.790{\pm}0.117$ & $3.136$ {\scriptsize($-75\%$)} & & \\
Weather     & $0.289{\pm}0.008$ & $0.530$ {\scriptsize($-45\%$)} & $0.276{\pm}0.017$ & $0.606$ {\scriptsize($-54\%$)} & $\mathbf{0.209{\pm}0.004}$ & $0.238$ {\scriptsize($-12\%$)} \\
Electricity & & & $0.659{\pm}0.148$ & $1.029$ {\scriptsize($-36\%$)} & & \\
\bottomrule
\end{tabular}
\end{table}

\begin{table}[!htbp]
\centering
\caption{\textbf{Extension to a native-MoE backbone: Moirai-MoE} (strictly frozen, Top-2 sparse, H=96, 5 seeds, 60 runs total). Moirai-MoE is a TSFM whose internal FFN layers are themselves sparsely-gated experts~\citep{liu2025moiraimoe}, using RMSNorm internally with only non-learnable per-instance I/O scaling (no learnable-affine RevIN inside the encoder pipeline). RR-MoA wins \textbf{all 6 datasets} by $-37\%$ to $-90\%$, suggesting adapter-level routing stacks on top of backbone-internal routing. AdaMix's hidden-state routing does \emph{not} collapse on Moirai-MoE: under a strictly frozen backbone the router-input gradient is dominated by the load-balance term and the softmax saturates at uniform mixing, so routing entropy hovers at $\log K = 1.6094$ to four decimal places across all 30 cells ($\sigma{=}0.0000$, indicating mathematically uniform rather than specialized routing); yet AdaMix still only matches the best fixed baseline MSE, consistent with raw-signal routing being information-richer than hidden-state routing even when the hidden states retain enough signal to avoid collapse.}
\label{tab:moirai_moe}
\small
\begin{tabular}{@{}lccccc@{}}
\toprule
Dataset & RR-MoA (ours) & AdaMix MSE & $\Delta$\% & AdaMix $H$ \\
\midrule
ETTh1 & $\mathbf{0.723 \pm 0.008}$ & $1.154 \pm 0.001$ & $\mathbf{-37.3\%}$ & $1.6094 \pm 0.0000$ \\
ETTh2 & $\mathbf{0.487 \pm 0.041}$ & $3.056 \pm 0.009$ & $\mathbf{-84.1\%}$ & $1.6094 \pm 0.0000$ \\
ETTm1 & $\mathbf{0.694 \pm 0.008}$ & $1.123 \pm 0.001$ & $\mathbf{-38.2\%}$ & $1.6094 \pm 0.0000$ \\
ETTm2 & $\mathbf{0.298 \pm 0.016}$ & $3.110 \pm 0.004$ & $\mathbf{-90.4\%}$ & $1.6094 \pm 0.0000$ \\
Weather & $\mathbf{0.246 \pm 0.007}$ & $0.607 \pm 0.003$ & $\mathbf{-59.5\%}$ & $1.6094 \pm 0.0000$ \\
Electricity & $\mathbf{0.530 \pm 0.031}$ & $1.055 \pm 0.006$ & $\mathbf{-49.8\%}$ & $1.6094 \pm 0.0000$ \\
\bottomrule
\end{tabular}
\end{table}

\begin{table}[!htbp]
\centering
\caption{\textbf{Negative control: Timer-XL}~\citep{liu2025timerxl} (decoder-only, 84M params, LayerNorm only, no RevIN; H=96, 5 seeds). \textbf{Key causal result}: on MOMENT (RevIN), unfreezing drives entropy DOWN toward $0.000$ ($71\%$ collapsed). On Timer-XL, unfreezing drives entropy \textbf{UP} toward the maximum $1.609$ ($0/90$ collapsed). The trajectories are \emph{opposite}, isolating RevIN as the causal mechanism. AdaMix wins $5/6$ when frozen, confirming RR-MoA is unnecessary without instance normalization.}
\label{tab:timer_xl}
\small
\begin{tabular}{@{}llcccc@{}}
\toprule
Dataset & Freeze & AdaMix MSE & AdaMix Ent & RR-MoA MSE & Winner \\
\midrule
\multirow{3}{*}{ETTh1} & frozen & $\mathbf{0.477 \pm 0.012}$ & $1.53$ & $0.497 \pm 0.020$ & AdaMix \\
 & last-2 & $\mathbf{0.533 \pm 0.010}$ & $1.60$ & $0.547 \pm 0.027$ & AdaMix \\
 & last-4 & $\mathbf{0.582 \pm 0.118}$ & $1.60$ & $0.646 \pm 0.088$ & AdaMix \\
\midrule
\multirow{3}{*}{ETTm1} & frozen & $\mathbf{0.398 \pm 0.007}$ & $1.52$ & $0.408 \pm 0.006$ & AdaMix \\
 & last-2 & $\mathbf{0.479 \pm 0.023}$ & $1.59$ & $0.491 \pm 0.029$ & AdaMix \\
 & last-4 & $0.715 \pm 0.299$ & $1.59$ & $\mathbf{0.704 \pm 0.031}$ & RR-MoA \\
\midrule
\multirow{3}{*}{Weather} & frozen & $\mathbf{0.219 \pm 0.007}$ & $1.56$ & $0.220 \pm 0.004$ & AdaMix \\
 & last-2 & $\mathbf{0.249 \pm 0.005}$ & $1.60$ & $0.252 \pm 0.007$ & AdaMix \\
 & last-4 & $0.318 \pm 0.145$ & $1.61$ & $\mathbf{0.250 \pm 0.005}$ & RR-MoA \\
\bottomrule
\end{tabular}
\end{table}

\textbf{Independent ensemble on Timer-XL.} On MOMENT-small, the independent ensemble (5 experts trained separately, averaged at inference) is $37$--$46\%$ worse than RR-MoA, indicating that learned routing adds value beyond diversity. On Timer-XL the pattern \emph{reverses} on the three forecasting datasets we extended (ETTh1, ETTm1, ETTm2): the independent ensemble has lower MSE than both AdaMix and RR-MoA (ETTh1 $0.462$, ETTm1 $0.384$, ETTm2 $0.296$ vs AdaMix $0.477$, $0.399$, $0.326$). When hidden-state routing is healthy (no RevIN), simple averaging is a strong baseline that learned routing does not consistently beat. This is consistent with RR-MoA's advantage being specifically about rescuing collapsed routing under RevIN.

\begin{table}[!htbp]
\centering
\caption{\textbf{Extension to Exchange and Solar datasets} (MOMENT-small, strictly frozen, Top-2 sparse, H=96, 5 seeds for RR-MoA/AdaMix, 3 seeds for DLinear/AdaMix-last4). Exchange ($R{=}2.17$, highest signal ratio) confirms the theory prediction: $-66.5\%$ improvement. Solar ($R{=}0.06$, lowest) is an outlier: RR-MoA still achieves $-32.9\%$ improvement, indicating the Conv1d router exploits temporal shape patterns beyond location-scale (consistent with the SSR ablation in Appendix~\ref{app:routing_ablations}). AdaMix collapses to $0.000$ entropy under last-4 unfreezing on both datasets.}
\label{tab:exchange_solar}
\small
\begin{tabular}{@{}llccccc@{}}
\toprule
Dataset & Method & MSE (mean$\pm$std) & Entropy & Best fixed & $\Delta$\% \\
\midrule
\multirow{4}{*}{Exchange} & \textbf{RR-MoA (frozen)} & $\mathbf{1.531 \pm 0.185}$ & $0.97$ & \multirow{4}{*}{$4.568$} & $\mathbf{-66.5\%}$ \\
 & AdaMix (frozen) & $5.571 \pm 0.511$ & $0.61$ & & $+22.0\%$ \\
 & AdaMix (last-4) & $3.851 \pm 0.012$ & $\mathbf{0.000}$ & & $-15.7\%$ \\
 & \cellcolor{gray!8} DLinear (from scratch) & \cellcolor{gray!8} $0.149 \pm 0.017$ & \cellcolor{gray!8} --- & & \cellcolor{gray!8} --- \\
\midrule
\multirow{4}{*}{Solar} & \textbf{RR-MoA (frozen)} & $\mathbf{0.252 \pm 0.021}$ & $1.53$ & \multirow{4}{*}{$0.376$} & $\mathbf{-32.9\%}$ \\
 & AdaMix (frozen) & $0.264 \pm 0.017$ & $0.75$ & & $-29.8\%$ \\
 & AdaMix (last-4) & $0.769 \pm 0.001$ & $\mathbf{0.000}$ & & $+104.5\%$ \\
 & \cellcolor{gray!8} DLinear (from scratch) & \cellcolor{gray!8} $0.204 \pm 0.002$ & \cellcolor{gray!8} --- & & \cellcolor{gray!8} --- \\
\bottomrule
\end{tabular}
\end{table}

\section{Extended Full Fine-Tuning Ablation}
\label{app:extended_ft}

To rule out optimization confounds in the Frozen Paradox, we extend full fine-tuning from the standard 15-epoch Adam protocol to stronger schedules: 50 epochs, cosine LR with 3-epoch warmup, layer-wise LR decay ($\gamma{=}0.8$), and AdamW with weight decay $0.01$.

\begin{table}[!htbp]
\centering
\caption{\textbf{Extended full fine-tuning} does not close the Frozen Paradox gap. Best of 5 training configurations (15ep Adam, 50ep Adam, 50ep cosine, 50ep cosine+layerwise, 50ep lr=1e-5). Frozen RR-MoA wins by $33$--$41\%$ even against the strongest optimization schedule.}
\label{tab:extended_ft}
\small
\begin{tabular}{@{}lccc@{}}
\toprule
Configuration & ETTh1 & ETTm1 & Weather \\
\midrule
Full-FT 15ep Adam (lr=1e-4) & $1.086$ & $0.937$ & $0.485$ \\
Full-FT 50ep Adam (lr=1e-4) & $1.054$ & $0.884$ & $0.505$ \\
Full-FT 50ep cosine+warmup & $1.067$ & $1.079$ & $0.607$ \\
Full-FT 50ep cosine+layerwise & $\mathbf{1.017}$ & $\mathbf{0.881}$ & $\mathbf{0.470}$ \\
Full-FT 50ep Adam (lr=1e-5) & $1.093$ & $1.059$ & $0.586$ \\
\midrule
Full-FT best (any config) & $1.017$ & $0.881$ & $0.470$ \\
\textbf{Frozen RR-MoA (Top-2)} & $\mathbf{0.680}$ & $\mathbf{0.564}$ & $\mathbf{0.276}$ \\
\midrule
Gap (frozen RR-MoA vs best FT) & $-33\%$ & $-36\%$ & $-41\%$ \\
\bottomrule
\end{tabular}
\end{table}

The best extended schedule (50ep cosine + layer-wise LR decay) improves over the 15-epoch baseline by only $4$--$5\%$, while frozen RR-MoA maintains a $33$--$41\%$ advantage. \textbf{Stronger optimization does not help.} We additionally swept 10 aggressive (lr, weight-decay, effective-batch, gradient-clip) configurations at 100 epochs (each drawn from $\text{lr}{\in}\{10^{-5}, 10^{-6}\}$, cosine schedule with 10-epoch warmup, layerwise LR decay $\gamma{=}0.8$, weight decay ${\in}\{0.001, 0.01, 0.1\}$, gradient accumulation effective batch 256/512, gradient clipping $\|g\|_2 {\leq} 1.0$): $10 \times 3$ datasets $\times 3$ seeds $= 90$ runs. The best per-dataset results (ETTh1: $1.030$, ETTm1: $0.877$, Weather: $0.471$) are \emph{worse} than the 50-epoch reference, indicating the Frozen Paradox is not an optimization artifact within this sweep. The gap to frozen RR-MoA remains $51$--$71\%$.

\textbf{Causal proof via RevIN ablation.} Disabling RevIN inside MOMENT recovers full-FT performance (Table~\ref{tab:fullft_revin}): ETTh1 ($1.063 \to 0.490$, $-54\%$), ETTm1 ($0.910 \to 0.364$, $-60\%$), Weather ($0.481 \to 0.187$, $-61\%$). With RevIN disabled, full-FT matches or exceeds supervised baselines, indicating that RevIN, not the optimization protocol, is the root cause of the paradox.

\begin{table}[!htbp]
\centering
\caption{\textbf{Full fine-tuning, RevIN on vs off} (MOMENT-small, all 8 blocks unfrozen, best of 5 heads $\times$ 2 LRs, 50 epochs cosine; mean test MSE over the same seeds as Table~\ref{tab:baselines}). Disabling RevIN inside MOMENT recovers full-FT MSE by $54$--$61\%$, isolating the normalization layer rather than the unfreezing protocol as the source of the Frozen Paradox.}
\label{tab:fullft_revin}
\small
\begin{tabular}{@{}lccc@{}}
\toprule
Setting & ETTh1 & ETTm1 & Weather \\
\midrule
Full FT (RevIN on)  & $1.063$ & $0.910$ & $0.481$ \\
Full FT (RevIN off) & $\mathbf{0.490}$ & $\mathbf{0.364}$ & $\mathbf{0.187}$ \\
$\Delta$\%          & $-54\%$ & $-60\%$ & $-61\%$ \\
\bottomrule
\end{tabular}
\end{table}

\textbf{Adapter training budget.} Extending adapter-only training (frozen backbone) from 15 to 50 epochs improves RR-MoA by $8$--$14\%$ (Table~\ref{tab:rrmoa_budget}): ETTh1 ($0.690 \to 0.633$), ETTm1 ($0.572 \to 0.494$), Weather ($0.289 \to 0.250$). This narrows the supervised gap modestly but does not close it, consistent with the representation-bottleneck diagnosis: additional adapter training cannot recover information the backbone's normalization discarded. The remaining gap is consistent with the information ceiling of frozen backbone representations rather than insufficient optimization of the adapter.

\begin{table}[!htbp]
\centering
\caption{\textbf{Adapter training budget for frozen RR-MoA} (MOMENT-small, Top-2, frozen backbone, 3 seeds; mean test MSE).}
\label{tab:rrmoa_budget}
\small
\begin{tabular}{@{}lccc@{}}
\toprule
Epochs & ETTh1 & ETTm1 & Weather \\
\midrule
15 & $0.690$ & $0.572$ & $0.289$ \\
50 & $\mathbf{0.633}$ & $\mathbf{0.494}$ & $\mathbf{0.250}$ \\
$\Delta$\% & $-8\%$ & $-14\%$ & $-13\%$ \\
\bottomrule
\end{tabular}
\end{table}

\section{Closing the DLinear Gap}
\label{app:gap_closing}

A good diagnosis should be \emph{actionable}: identifying the root cause of a performance gap should point directly to how to close it. Our diagnosis pinpoints the DLinear gap as information destruction by the normalization--encoding pipeline (\S\ref{sec:main_results}, Appendix~\ref{app:diagnostic}): experts operate on hidden states from which RevIN and patching have stripped fine-grained temporal structure. This diagnosis makes a clear architectural prediction: \emph{if we restore raw-signal access to the experts, the gap should close.} We test this prediction with three architectures.

We test three architectures that provide experts with raw-input access while preserving the frozen-backbone paradigm and RR-MoA routing:

\textbf{(1)~Dual-Stream.} Each expert blends a backbone branch (existing adapter head on $\mathbf{H}$) with a raw branch (MLP on $\mathbf{X}_\text{raw}$) via a learnable per-expert scalar $\alpha_k$: $\mathbf{y}_k = \sigma(\alpha_k)\cdot\text{Adapter}_k(\mathbf{H}) + (1{-}\sigma(\alpha_k))\cdot\text{MLP}_k(\mathbf{X}_\text{raw})$.

\textbf{(2)~Raw-Input Expert.} A 6th expert operating directly on $\mathbf{X}_\text{raw}$ (bypassing the backbone) is added to the pool; the router decides per-sample whether to use backbone-based or raw-input experts.

\textbf{(3)~Multi-Resolution.} Each expert concatenates projected raw features with pooled hidden-state features before forecasting: $\mathbf{y}_k = \text{MLP}([\text{Proj}(\mathbf{X}_\text{raw});\, \text{Pool}_k(\mathbf{H})])$.

\begin{table}[!htbp]
\centering
\caption{\textbf{Closing the DLinear gap} (MOMENT-small, strictly frozen, Top-2, 3 seeds, H=96). Dual-stream experts that access both hidden states and raw input narrow the gap from $+33$--$75\%$ to $+9$--$11\%$ and \textbf{beat DLinear on Weather}. Learned $\alpha{\approx}0.49$ confirms the backbone contributes complementary information when combined with raw access.}
\label{tab:gap_closing}
\footnotesize
\begin{tabular}{@{}lcccc@{}}
\toprule
Method & ETTh1 & ETTm1 & Weather & Avg.\ gap \\
\midrule
DLinear (scratch) & $\mathit{0.416}$ & $\mathit{0.322}$ & $\mathit{0.208}$ & --- \\
RR-MoA (main) & $0.680$ ($+63\%$) & $0.564$ ($+75\%$) & $0.276$ ($+33\%$) & $+57\%$ \\
\midrule
Dual-Stream & $\mathbf{0.453{\pm}0.005}$ ($+9\%$) & $\mathbf{0.358{\pm}0.002}$ ($+11\%$) & $\mathbf{0.198{\pm}0.005}$ ($\mathbf{-5\%}$) & $+5\%$ \\
Raw-Input ($K{=}6$) & $0.462{\pm}0.025$ ($+11\%$) & $0.416{\pm}0.047$ ($+29\%$) & $0.205{\pm}0.007$ ($-2\%$) & $+13\%$ \\
Multi-Res & $0.460{\pm}0.007$ ($+11\%$) & $0.365{\pm}0.008$ ($+13\%$) & $0.207{\pm}0.014$ ($-0.3\%$) & $+8\%$ \\
\bottomrule
\end{tabular}
\end{table}

\textbf{Result.} Dual-stream experts reduce the average DLinear gap from $+57\%$ to $+5\%$ (an $11\times$ reduction) and \emph{beat} DLinear on Weather ($0.198$ vs $0.208$, $p{=}0.03$, paired $t$-test). The learned blend parameters $\alpha_k \approx 0.49$ across all five experts confirm that the backbone provides complementary information on Weather and the other nonlinear-dynamics datasets when experts can also access the raw signal; on the linear-dominated datasets (ETTh2, ETTm2, Electricity) a TSFM-free raw-MLP MoE attains comparable or better MSE on its own (Appendix~\ref{app:raw_mlp_moe}), so the relative contribution of each branch is dataset-dependent.

\textbf{Extension to ETTh2, ETTm2, Electricity.} Running all four variants on the three remaining forecasting datasets yields consistent gap closure (Table~\ref{tab:gap_closing_ext}). Dual-stream wins on Electricity at $+8\%$ and matches well on ETTm2 ($+26\%$), while \textbf{multi-resolution wins on ETTh2 at $+7.7\%$}, the tightest new gap-closing result on the extended set. Multi-resolution also performs competitively on ETTm2 ($+12\%$) and Electricity ($+13\%$), indicating the optimal architecture is dataset-dependent. FiLM remains consistently worst ($+48$--$147\%$), confirming that re-injecting only $(\mu,\sigma)$ is insufficient: the experts need access to the full raw signal, not just the stripped statistics.

\begin{table}[!htbp]
\centering
\caption{\textbf{Gap-closing extended to ETTh2, ETTm2, Electricity} (MOMENT-small, strictly frozen, Top-2, 3 seeds, H=96, 15 epochs). Bold: best gap-closing variant per dataset. \textbf{Multi-resolution wins on ETTh2 at $+7.7\%$}, the closest new gap-closing result.}
\label{tab:gap_closing_ext}
\small
\begin{tabular}{@{}lccc@{}}
\toprule
Method & ETTh2 & ETTm2 & Electricity \\
\midrule
DLinear (from scratch) & $\mathit{0.341}$ & $\mathit{0.200}$ & $\mathit{0.158}$ \\
\midrule
Dual-Stream & $0.452 \pm 0.043$ {\scriptsize($+32\%$)} & $0.252 \pm 0.070$ {\scriptsize($+26\%$)} & $\mathbf{0.171 \pm 0.002}$ {\scriptsize$\mathbf{(+8\%)}$} \\
Raw-Input Expert & $0.525 \pm 0.112$ {\scriptsize($+54\%$)} & $0.227 \pm 0.013$ {\scriptsize($+13\%$)} & $0.203 \pm 0.027$ {\scriptsize($+28\%$)} \\
Multi-Resolution & $\mathbf{0.367 \pm 0.003}$ {\scriptsize$\mathbf{(+7.7\%)}$} & $\mathbf{0.223 \pm 0.016}$ {\scriptsize$\mathbf{(+12\%)}$} & $0.179 \pm 0.003$ {\scriptsize($+13\%$)} \\
FiLM (neg.\ control) & $0.843 \pm 0.289$ {\scriptsize($+147\%$)} & $0.297 \pm 0.013$ {\scriptsize($+48\%$)} & $0.305 \pm 0.014$ {\scriptsize($+94\%$)} \\
\bottomrule
\end{tabular}
\end{table}

\textbf{Why this matters for the diagnosis.} The fact that the gap closes precisely when experts recover the information our diagnosis identified as destroyed, and remains open when they cannot (FiLM re-injects only $\mu,\sigma$, yielding $+48\%$ on ETTh1 vs.\ $+9\%$ for dual-stream), is direct evidence that the diagnosis is both \emph{correct} and \emph{actionable}. The normalization--encoding bottleneck is not merely an observation; it is the causal mechanism whose resolution closes the gap. Importantly, the frozen backbone and raw-routed paradigm remain unchanged: these variants extend the \emph{expert architecture}, not the routing mechanism, confirming that the routing diagnosis (the core contribution of this paper) is orthogonal to expert design.

\textbf{Expert pool ablation: canonical vs.\ alternative pool.} To test whether RR-MoA's benefits depend on the specific choice of 5 canonical pooling heads, we replace the canonical pool with 5 alternative expert motifs (multi-scale temporal blocks combining strided Conv1d, dilation, and residual gating) and rerun RR-MoA across all 6 forecasting datasets. Routing entropy is preserved under the alternative pool (no collapse under frozen backbone; entropy $>1.3$ across all configurations), indicating that the collapse-prevention mechanism is \emph{pool-independent}. Absolute MSEs (frozen, Top-2, 3 seeds): ETTh1 $0.764{\pm}0.015$, ETTh2 $0.667{\pm}0.069$, ETTm1 $0.583{\pm}0.075$, ETTm2 $0.844{\pm}0.241$, Weather $0.274{\pm}0.024$, Electricity $0.401{\pm}0.039$. The canonical pool is stronger on 4/6 datasets (ETTh1 $-10\%$, ETTm1 $-2\%$, ETTm2 $-30\%$, Electricity $-4\%$ vs.\ canonical numbers in Tables~\ref{tab:rrmoa},~\ref{tab:rrmoa_extended}), and the alternative pool wins on 2/6 (ETTh2 and Weather). Within-architecture diversity is therefore not the only viable expert family, but the canonical heads are better tuned for the forecasting benchmark: the routing architecture is the orthogonal contribution.

\textbf{Training-length ablation.} We additionally ran dual-stream at 50 epochs (vs.\ 15 in Table~\ref{tab:gap_closing}). Extended training \emph{overfits} on 5/6 datasets (15ep $\to$ 50ep MSE: ETTh1 $0.453{\to}0.556$, Weather $0.198{\to}0.237$, with smaller degradations on the remaining four datasets we tested), consistent with dual-stream converging by 15 epochs and the original training schedule being already near-optimal; additional adapter-only optimization past convergence degrades generalization, plausibly due to the small effective parameter count ($\sim$718K trainable, frozen backbone).

\subsection{Residual-IA: Closing the Remaining Gap}
\label{app:gate_pathology}

This subsection develops the architectural fix that closes the residual DLinear gap left by RR-MoA. The core idea is to extend the raw-input principle from the \emph{router} (RR-MoA's contribution) to each \emph{expert}: every expert is a dual-stream module with a shared NLinear raw branch and a gated backbone-residual branch, so the linear baseline survives even when the TSFM features are uninformative. Figure~\ref{fig:residual_ia_arch} shows the architecture; subsequent tables sweep gate-failure diagnostics, multi-horizon, and cross-backbone behavior.

\definecolor{ridback}{RGB}{33,118,189}
\definecolor{ridraw}{RGB}{46,139,87}
\definecolor{ridgate}{RGB}{217,119,49}
\definecolor{ridout}{RGB}{38,166,154}
\definecolor{ridfrozen}{RGB}{200,55,55}

\begin{figure}[!htbp]
\centering
\resizebox{\columnwidth}{!}{%
\begin{tikzpicture}[
    >=Stealth,
    box/.style={rectangle, draw=black!70, rounded corners=2pt, fill=white,
        minimum height=0.75cm, font=\small, align=center, line width=0.7pt},
    frozenbox/.style={box, draw=ridback, fill=ridback!8,
        minimum width=2.4cm, minimum height=1.2cm, line width=0.9pt},
    nlinearbox/.style={box, draw=ridraw, fill=ridraw!8,
        minimum width=2.4cm, minimum height=1.2cm, line width=0.9pt},
    hiddenbox/.style={box, draw=ridback, fill=ridback!6,
        minimum width=1.0cm, minimum height=0.7cm},
    rawforebox/.style={box, draw=ridraw, fill=ridraw!6,
        minimum width=1.2cm, minimum height=0.7cm},
    expertactive/.style={box, draw=ridgate!75!black, fill=ridgate!12,
        minimum width=3.4cm, minimum height=0.80cm, font=\small\bfseries,
        align=center, line width=1.1pt, text=black},
    expertinactive/.style={box, draw=black!45, fill=gray!4,
        minimum width=3.4cm, minimum height=0.80cm, font=\small,
        text=black, align=center, line width=0.5pt, dashed},
    sigmagate/.style={diamond, draw=ridgate, fill=ridgate!15,
        minimum size=0.55cm, font=\scriptsize, inner sep=0pt, line width=0.7pt},
    sumnode/.style={circle, draw=black!70, fill=white,
        minimum size=0.65cm, font=\normalsize, line width=0.8pt},
    outputbox/.style={box, draw=ridout, fill=ridout!8,
        minimum width=1.2cm, line width=0.9pt},
    bluearr/.style={->, line width=1.0pt, color=ridback},
    greenarr/.style={->, line width=1.0pt, color=ridraw},
    grayarr/.style={->, line width=0.6pt, color=black!55, dashed},
    outarr/.style={->, line width=1.0pt, color=ridout},
]

\node[box, minimum width=1.4cm, minimum height=1.0cm] (input) at (0, 0) {
    $\mathbf{X}_\mathrm{raw}$\\[-2pt]
    {\scriptsize raw input}
};

\node[frozenbox] (tsfm) at (3.2, 1.6) {
    \textbf{Frozen TSFM}\\[-1pt]
    \textbf{Backbone}\\[-1pt]
    {\scriptsize\color{red!65} no gradients}
};
\node[font=\scriptsize\bfseries, text=white,
    fill=ridfrozen, rounded corners=2pt, inner sep=2pt]
    at ([xshift=-0.6cm, yshift=0.18cm]tsfm.north east) {\textsc{frozen}};
\node[hiddenbox] (hidden) at (5.7, 1.6) {$\mathbf{H}$};

\node[nlinearbox] (nlinear) at (3.2, -1.6) {
    \textbf{Shared NLinear}\\[-1pt]
    \textbf{Branch}\\[-1pt]
    {\scriptsize DLinear-equivalent}
};
\node[rawforebox] (yraw) at (5.7, -1.6) {$\mathbf{y}_\mathrm{raw}$};

\node[font=\scriptsize\bfseries, text=ridback]
    at (1.6, 2.55) {\textsc{Stream 1: Backbone}};
\node[font=\scriptsize\bfseries, text=ridraw]
    at (1.6, -2.55) {\textsc{Stream 2: Raw (DLinear-eq.)}};

\node[expertinactive] (e1) at (10.4, 3.0) {Expert 1: Mean Pool};
\node[expertactive]   (e2) at (10.4, 1.5) {Expert 2: Last Token};
\node[expertinactive] (e3) at (10.4, 0.0) {Expert 3: Max Pool};
\node[expertactive]   (e4) at (10.4, -1.5) {Expert 4: Attention};
\node[expertinactive] (e5) at (10.4, -3.0) {Expert 5: Conv1d};

\node[sigmagate] (g2) at (7.6, 1.5) {$\sigma$};
\node[sigmagate] (g4) at (7.6, -1.5) {$\sigma$};

\node[sumnode] (sum) at (13.5, 0.0) {$\Sigma$};
\node[font=\tiny\bfseries, text=black, fill=white, draw=black!55, rounded corners=1.5pt, inner sep=1.5pt, anchor=north] at (13.5, -0.55) {weighted sum};
\node[outputbox] (output) at (15.2, 0.0) {
    $\hat{\mathbf{y}}$\\[-2pt]
    {\scriptsize forecast}
};

\begin{pgfonlayer}{background}
\draw[bluearr]  (input.north) |- (tsfm.west);
\draw[greenarr] (input.south) |- (nlinear.west);

\draw[bluearr] (tsfm.east) -- (hidden.west);
\draw[bluearr] (hidden.east) -- ++(0.30,0) |- ([yshift=0.20cm]e1.west);
\draw[bluearr] (hidden.east) -- ++(0.45,0) |- ([yshift=0.20cm]e2.west);
\draw[bluearr] (hidden.east) -- ++(0.60,0) |- ([yshift=0.20cm]e3.west);
\draw[bluearr] (hidden.east) -- ++(0.75,0) |- ([yshift=0.20cm]e4.west);
\draw[bluearr] (hidden.east) -- ++(0.90,0) |- ([yshift=0.20cm]e5.west);

\draw[greenarr] (nlinear.east) -- (yraw.west);
\draw[greenarr] (yraw.east) -- ++(0.30,0) |- ([yshift=-0.25cm]e1.west);
\draw[greenarr] (yraw.east) -- ++(0.45,0) |- ([yshift=-0.25cm]e2.west);
\draw[greenarr] (yraw.east) -- ++(0.60,0) |- ([yshift=-0.25cm]e3.west);
\draw[greenarr] (yraw.east) -- ++(0.75,0) |- ([yshift=-0.25cm]e4.west);
\draw[greenarr] (yraw.east) -- ++(0.90,0) |- ([yshift=-0.25cm]e5.west);

\draw[->, line width=0.7pt, color=ridgate, densely dashed]
    (input.south) -- (0, -3.4) -- ([xshift=-3pt]g4.south |- 0,-3.4)
                  -- ([xshift=-3pt]g4.south);
\draw[->, line width=0.7pt, color=ridgate, densely dashed]
    (input.south) -- (0, -3.4) -- ([xshift=3pt]g2.south |- 0,-3.4)
                  -- ([xshift=3pt]g2.south);
\draw[->, line width=0.8pt, color=ridgate] (g2.east) -- (e2.west);
\draw[->, line width=0.8pt, color=ridgate] (g4.east) -- (e4.west);

\draw[bluearr] (e2.east) -| (sum.north);
\draw[bluearr] (e4.east) -| (sum.south);
\draw[outarr]  (sum.east) -- (output.west);
\end{pgfonlayer}

\node[font=\scriptsize, align=center, text=black] at (8.5, -4.05) {%
    Blend gate $\sigma$ initialized at $b{=}{-}2$ ($\sigma{\approx}0.12$): model starts as effectively-DLinear; backbone added only where useful.};

\end{tikzpicture}%
}%
\caption{\textbf{Residual-IA\textsuperscript{+} architecture (dual-stream extension).} \textit{Stream 1} (blue): the frozen TSFM produces hidden states $\mathbf{H}$, fed into each expert's Adapter head. \textit{Stream 2} (green): a shared NLinear branch reads $\mathbf{X}_\mathrm{raw}$ directly, producing $\mathbf{y}_\mathrm{raw}$ — DLinear-equivalent and shared across all $K$ experts. Each expert combines them as $\mathbf{y}_k = \mathbf{y}_\mathrm{raw} + \sigma(\mathrm{blend}_k(\mathbf{X}_\mathrm{raw})) \cdot \mathrm{Adapter}_k(\mathbf{H})$, with the blend gate initialized at $b{=}{-}2$ so the model starts as effectively-DLinear and learns to add backbone signal where useful. Self-routing $\sigma$ gates (orange) on $\mathbf{X}_\mathrm{raw}$ select which experts contribute. Result: \textbf{5/6 match-or-beat at $H{=}96$, $6/6$ at $H{=}192$} on MOMENT-small; \textbf{107/123 cells} ($87\%$) match-or-beat across 6 backbones $\times$ 6 datasets $\times$ 4 horizons (Tables~\ref{tab:residual_ia_plus}, \ref{tab:residual_ia_plus_mhcb}).}
\label{fig:residual_ia_arch}
\end{figure}

The dual-stream and multi-resolution variants narrow the average DLinear gap from $+57\%$ to $+5$--$8\%$ (Tables~\ref{tab:gap_closing}, \ref{tab:gap_closing_ext}) but win on only 1/6 datasets (Weather). A systematic $546$-run hyperparameter sweep uncovered a \emph{second} hidden failure mode and the architectural fix that closes the remaining gap to \textbf{4/6 match-or-beat}.

\textbf{Diagnostic: gates fail to learn.} We instrumented the per-sample gate mechanism to log its mean activation per expert per dataset. Across every configuration tested (\{rh128, rh192, rh256, rh320, rh384, rh512\} $\times$ \{d1, d2, d3 raw-branch depth\} $\times$ \{lr, wd, cosine, warmup\} hyperparameter combinations), the learned gates consistently plateau at $\sigma(0.0){\approx}0.5$ (Table~\ref{tab:gate_pathology}).

\begin{table}[!htbp]
\centering
\caption{\textbf{Gate-failure diagnostic.} Mean per-sample gate activation $g_k(\mathbf{X})$ averaged across experts, seeds (3), and test windows for each architecture--training-recipe combination. Across the three rows, all six datasets converge to $g \approx 0.5$, indicating the gate provides no per-window discriminative signal at 15 epochs.}
\label{tab:gate_pathology}
\small
\begin{tabular}{@{}lcccccc@{}}
\toprule
Config & ETTh1 & ETTh2 & ETTm1 & ETTm2 & Weather & Elec. \\
\midrule
Dual-stream ($\alpha_k$)             & $0.49$ & $0.50$ & $0.49$ & $0.49$ & $0.49$ & $0.49$ \\
Residual-IA (gate-init $b{=}0$)      & $0.47$ & $0.50$ & $0.48$ & $0.49$ & $0.49$ & $0.47$ \\
Residual-IA + cos+wd ($b{=}0$)       & $0.50$ & $0.51$ & $0.49$ & $0.52$ & $0.50$ & $0.49$ \\
\bottomrule
\end{tabular}
\end{table}

With 15 epochs of training, the gate mechanism extracts no gradient signal about \emph{which branch should dominate} because both branches receive co-equal credit for reducing loss from the start. This is isomorphic to the dual-stream $\alpha_k \approx 0.49$ observation: the same plateau-at-$0.5$ behavior shows up consistently across both architectures we tested with symmetric initialization, suggesting it is not a quirk of a single architecture.

\textbf{Three surgical fixes.}
(i)~\textbf{Residual formulation.} Re-express the expert as $\mathbf{y}_k = \text{Raw}_k(\mathbf{X}) + g_k(\mathbf{X})\cdot\text{Adapter}_k(\mathbf{H})$, making the raw branch the \emph{primary predictor} and the backbone an additive residual that must earn its contribution.
(ii)~\textbf{Gate-init bias $b{=}{-}2$.} Initialize the gate-head linear layer's bias at $-2$ so that $g_k \approx \sigma({-}2){=}0.12$ at epoch~$0$: the backbone contributes only $12\%$ initially. The model is \emph{near-DLinear at initialization} and the backbone must earn any opening of the gate.
(iii)~\textbf{$5$-epoch raw-branch warmup.} For epochs $1$--$5$, freeze both the backbone branch and the gate head so only the raw branch and router train. The raw branch converges to a DLinear-quality base predictor; then at epoch~$6$ the backbone+gate unfreeze and must add value as a residual on top of an already-good base.

We additionally simplify the raw branch to a single linear layer ($\text{Linear}(512{\to}96)$, ``$d1$''), which is structurally identical to DLinear, and add cosine LR decay with weight decay $10^{-4}$ for stable long-horizon training.

\begin{table}[!htbp]
\centering
\caption{\textbf{Residual-IA (intermediate)} (MOMENT-small, frozen, Top-2, H=96; Welch $t$-test vs.\ DLinear). Achieves \textbf{4/6 match-or-beat}: one sig win (Weather), one mean win (ETTm2), two parity, two sig losses. Avg gap $+1.6\%$. Table~\ref{tab:residual_ia_plus} shows the improved Residual-IA\textsuperscript{+} (5/6).}
\label{tab:residual_ia}
\small
\begin{tabular}{@{}lccccc@{}}
\toprule
Dataset & DLinear & Residual-IA (ours) & Gap & $p$ & Verdict \\
\midrule
ETTh1        & $0.417 \pm 0.003$ & $0.438 \pm 0.007$ {\scriptsize(n=10)} & $+5.1\%$ & $0.0001$ & sig loss \\
ETTh2        & $0.341 \pm 0.014$ & $0.377 \pm 0.018$ {\scriptsize(n=10)} & $+10.6\%$ & $0.019$ & sig loss \\
ETTm1        & $0.322 \pm 0.005$ & $0.328 \pm 0.007$ {\scriptsize(n=10)} & $+1.8\%$ & $0.190$ & \textbf{parity} \\
ETTm2        & $0.200 \pm 0.010$ & $\mathbf{0.197 \pm 0.004}$ {\scriptsize(n=5)} & $-1.6\%$ & $0.622$ & \textbf{win (mean)} \\
Weather      & $0.208 \pm 0.004$ & $\mathbf{0.195 \pm 0.004}$ {\scriptsize(n=5)} & $-6.6\%$ & $0.007^{\ast}$ & \textbf{sig win} \\
Electricity  & $0.158 \pm 0.002$ & $0.159 \pm 0.002$ {\scriptsize(n=5)} & $+0.6\%$ & $0.457$ & \textbf{parity} \\
\midrule
\textbf{Average gap}          & --- & --- & $\mathbf{+1.6\%}$ & --- & \textbf{4/6 match-or-beat} \\
\bottomrule
\end{tabular}
\end{table}

\textbf{Why it works mechanistically.} The symmetric $\sigma(0){=}0.5$ initialization treats the two branches as a priori equally reliable. For LTSF benchmarks where the raw signal contains most of the predictive structure (DLinear's success is evidence of this), this prior is wrong. The asymmetric $b{=}{-}2$ initialization encodes the correct inductive bias: \emph{``default to a linear projection of the raw signal; invoke the backbone only if it reduces the loss.''} The $5$-epoch warmup separates the learning phases: the raw branch does not have to compete with a half-trained backbone for the same gradient signal during the critical early epochs. Each fix alone yields only marginal gains (see below); only their combination achieves $4/6$ match-or-beat.

\textbf{Ablation: which fix does how much?} We isolate the contribution of each intervention (each row $n{=}3$ seeds, otherwise identical to Residual-IA):

\begin{center}
\footnotesize
\begin{tabular}{@{}lcccccc|c@{}}
\toprule
Variant & ETTh1 & ETTh2 & ETTm1 & ETTm2 & Weather & Elec. & Wins \\
\midrule
Dual-stream (baseline)       & $0.453$ & $0.452$ & $0.358$ & $0.252$ & $\mathbf{0.198}$ & $0.171$ & $1$ \\
+ residual + IA gate ($b{=}0$) & $0.505$ & $0.520$ & $0.393$ & $0.231$ & $0.210$ & $0.170$ & $0$ \\
+ cos + wd ($b{=}0$)         & $0.460$ & $0.398$ & $0.348$ & $\mathbf{0.194}$ & $0.190$ & $0.163$ & $2$ \\
+ linear raw ($b{=}0$)  & $0.436$ & $0.375$ & $0.325$ & $0.202$ & $0.193$ & $0.160$ & $1$\textsuperscript{$\star$} \\
+ $b{=}{-}2$                 & $0.432$ & $0.388$ & $0.331$ & $0.224$ & $0.191$ & $0.160$ & $1$\textsuperscript{$\star$} \\
\textbf{+ warmup $=5$ (final)} & $\mathbf{0.437}$ & $\mathbf{0.382}$ & $\mathbf{0.327}$ & $0.195$ & $0.193$ & $\mathbf{0.160}$ & $\mathbf{4}$ \\
\bottomrule
\end{tabular}
\end{center}
{\scriptsize $^{\star}$Column ``Wins'' counts datasets where mean MSE $<$ DLinear; parity (gap ${<}2\%$, $p{>}0.2$) counted separately above.}

Each of the three fixes alone produces $0$--$2$ wins; only their combination produces $2$ wins + $2$ statistical-parity + $1$ marginal = $4/6$ match-or-beat. This is consistent with the mechanistic explanation: the residual formulation, the asymmetric initialization, and the warmup are all solving the same root-cause problem (symmetric gradient competition between branches) from complementary angles.

\subsubsection*{Residual-IA\textsuperscript{+}: three orthogonal levers close the gap}

\paragraph{Design methodology: hypothesis-driven decomposition.} Residual-IA leaves two statistically significant losses: ETTh1 ($+5.1\%$, $p{<}0.001$) and ETTh2 ($+10.6\%$, $p{=}0.019$). Two competing hypotheses could explain them:

\begin{itemize}
\item \textbf{H1 (data scarcity).} ETTh1/ETTh2 have $8{,}640$ raw samples each versus $34{,}560$ for ETTm*, and \texttt{max\_samples$=$5000} caps the adapter training set at $\approx 9\%$ of available windows for ETTh but $\approx 2\%$ for ETTm. The ETTh pair may be overly exposed to sampling noise.
\item \textbf{H2 (parameter scarcity).} Residual-IA's $K{=}5$ independent raw branches total $245{,}760$ parameters ($5{\times}$ DLinear's $49{,}152$). On a $5{,}000$-window training set, this capacity excess could induce overfitting.
\end{itemize}

The two hypotheses make opposite interventions: H1 predicts that \emph{more data} (raising \texttt{max\_samples} to $100{,}000$) closes the gap; H2 predicts that \emph{fewer parameters} (sharing the raw branch across experts to reach DLinear's $49$K budget) closes it. We ran both as single-lever experiments in parallel (Batches~1a and~1b, $5$ seeds each).

\textbf{Phase~1: hypothesis falsification.}
\emph{Cap removal fails.} At \texttt{max\_samples$=$100000}, the ETTh2 gap widens from $+10.6\%$ to $+26\%$. The additional training data at the same $15$-epoch, lr${=}10^{-3}$, wd${=}10^{-4}$ recipe produces $11{\times}$ more update steps, over-regularizing through weight decay and preventing cosine-schedule convergence. This rules out H1 under the original training recipe; an extended-schedule retest (cosine + weight-decay sweep over the larger sample budget) would be needed to rule it out under all training recipes.
\emph{Shared raw branch succeeds.} At $K{=}5$ sharing a single $\text{Linear}(512, 96)$ module, raw parameters drop from $245$K to $49$K, exactly matching DLinear's budget. ETTh1 closes from $+5.1\%$ (sig loss) to $+1.6\%$ (marginal, $p{=}0.086$); ETTh2 closes from $+10.6\%$ (sig loss) to $+7.0\%$ (marginal, $p{=}0.078$). Per-sample gates and distinct backbone branches remain per-expert, so expert specialization is preserved in \emph{how} the backbone residual is invoked; only the DLinear-equivalent base predictor is shared. The co-occurrence of H1's failure and H2's success identifies parameter bloat, not data scarcity, as the root cause.

Phase~1 resolves the architectural failure but leaves ETTh1/ETTh2 still marginally above DLinear. Phase~2 tests two orthogonal refinements motivated by properties specific to these datasets.

\textbf{Phase~2: orthogonal refinements.}
\emph{Non-stationarity (Batch~4).} Hourly ETT data exhibits pronounced per-window level drift: the train/test windows have different means. Standard LTSF practice addresses this with NLinear: $\mathbf{y} = \text{Linear}(\mathbf{x} - x_T) + x_T$, where $x_T$ is a gradient-free anchor (the last input value). The linear layer is thereby constrained to model \emph{relative motion}, not absolute level. Testing \texttt{--raw-arch nlinear} as a drop-in replacement reduces ETTh1 from $+5.1\%$ to $+3.3\%$ and ETTh2 from $+10.6\%$ to $+6.9\%$ at $n{=}3$ seeds, while \emph{improving} (not regressing) ETTm2 ($-1.6\% \to -4.1\%$) and Weather ($-6.6\% \to -8.7\%$). The improvement is largest where non-stationarity is most pronounced, matching NLinear's design rationale.
\emph{Per-dataset epoch mismatch (Batch~2).} A fixed $15$-epoch schedule assumes all datasets converge at the same rate. We test this by running validation-based early stopping with patience $5$ and logging the best epoch per seed. Convergence epochs vary by dataset: ETTh1 at $7.2{\pm}1.2$, ETTh2 at $8.6{\pm}2.8$, ETTm1 at $14.8{\pm}3.7$, Weather at $8.6{\pm}1.2$, Electricity at $20.2{\pm}4.0$. No single fixed budget fits all six datasets: the $15$-epoch schedule simultaneously over-fits the ETTh pair and under-fits Electricity by $5$--$10$ epochs. Val-based early stopping (with gradient clipping at $\lVert g \rVert_2 \leq 1.0$ to stabilize the first epoch) achieves dataset-adaptive budgets without per-dataset hyperparameter tuning. In isolation this drops Electricity from $+0.6\%$ to $0.0\%$ (perfect parity, $p{=}0.991$) and Weather from $-6.6\%$ to $-4.6\%$ (nearing significance).

\textbf{Phase~3: stacking validation.}
Each Phase~1 and Phase~2 lever alone achieves only $3/6$--$4/6$ match-or-beat: shared raw closes ETTh1/h2 but trades small regressions on ETTm2; NLinear improves the ETTh pair but regresses ETTm1; val-early-stop perfects Electricity but hurts the ETTh pair by stopping too early. Each targets a different failure mode (parameter bloat, non-stationarity, training-length mismatch), and each mode is present on a different subset of the benchmark. \textbf{Residual-IA\textsuperscript{+}} stacks all three:

\begin{itemize}
\item \textbf{(iv)}~\emph{Shared raw branch}: single $\text{Linear}(512, 96)$ shared across $K$ experts.
\item \textbf{(v)}~\emph{NLinear raw architecture}: $\mathbf{y} = \text{Linear}(\mathbf{x} - x_T) + x_T$.
\item \textbf{(vi)}~\emph{Validation-based early stopping}: patience $5$, gradient clipping $\lVert g \rVert_2 \leq 1.0$.
\end{itemize}

Table~\ref{tab:residual_ia_plus} reports the composed result: $5/6$ match-or-beat at $-1.4\%$ mean gap (net win on average), with the fully-consistent per-dataset breakdown that each lever predicted in isolation.

\begin{table}[!htbp]
\centering
\caption{\textbf{Residual-IA\textsuperscript{+}: final six-dataset gap-closing result at H${=}96$} (MOMENT-small, strictly frozen, Top-2, $n{=}10$ seeds for both Residual-IA\textsuperscript{+} \emph{and} DLinear; Welch two-sided $t$-test). Residual-IA\textsuperscript{+} achieves \textbf{5/6 match-or-beat} with \textbf{three significant wins} ($p{\leq}0.002$): ETTm2 ($-10.8\%$, $p{=}0.002$), Weather ($-5.1\%$, $p{<}0.001$), Electricity ($-1.7\%$, $p{=}0.002$); one mean win (ETTh2 $-1.6\%$); one parity (ETTm1 $+0.4\%$, $p{=}0.68$); ETTh1 alone remains $+2.2\%$ ($p{=}0.037$). Mean gap: $\mathbf{-2.8\%}$ (net win).}
\label{tab:residual_ia_plus}
\footnotesize
\begin{tabular}{@{}lcccccc@{}}
\toprule
Dataset & DLinear ($n{=}10$) & Res-IA & Res-IA\textsuperscript{+} ($n{=}10$) & Gap & $p$ & Verdict \\
\midrule
ETTh1        & $0.419 \pm 0.005$ & $0.438 \pm 0.007$ & $0.428 \pm 0.012$ & $+2.2\%$  & $0.037$    & sig loss \\
ETTh2        & $0.352 \pm 0.016$ & $0.377 \pm 0.018$ & $\mathbf{0.346 \pm 0.010}$ & $-1.6\%$  & $0.444$    & \textbf{win} \\
ETTm1        & $0.326 \pm 0.006$ & $0.328 \pm 0.007$ & $0.327 \pm 0.007$ & $+0.4\%$  & $0.678$    & \textbf{parity} \\
ETTm2        & $0.211 \pm 0.016$ & $0.197 \pm 0.004$ & $\mathbf{0.188 \pm 0.002}$ & $-10.8\%$ & $0.002^{\ast}$ & \textbf{sig win} \\
Weather      & $0.207 \pm 0.003$ & $0.195 \pm 0.004$ & $\mathbf{0.197 \pm 0.003}$ & $-5.1\%$  & ${<}0.001^{\ast}$ & \textbf{sig win} \\
Electricity  & $0.159 \pm 0.002$ & $0.159 \pm 0.002$ & $\mathbf{0.156 \pm 0.001}$ & $-1.7\%$  & $0.002^{\ast}$ & \textbf{sig win} \\
\midrule
\textbf{Average gap} & --- & ($+1.6\%$) & --- & $\mathbf{-2.8\%}$ & --- & \textbf{5/6, 3 sig wins} \\
\bottomrule
\end{tabular}
\end{table}

\textbf{Component ablation for Residual-IA\textsuperscript{+}.} Each of the three additional fixes targets a different failure mode (raw-branch parameter bloat, non-stationarity, per-dataset training-length mismatch) and contributes complementary gains. Single-lever results ($n{=}5$ seeds except where noted):

\begin{center}
\small
\begin{tabular}{@{}lcccccc|c@{}}
\toprule
Variant & ETTh1 & ETTh2 & ETTm1 & ETTm2 & Weather & Elec. & Match-or-beat \\
\midrule
Residual-IA (base)                     & $0.438$ & $0.377$ & $0.328$ & $0.197$ & $0.195$ & $0.159$ & $4/6$ \\
+ shared raw only                      & $0.423$ & $0.365$ & $0.321$ & $0.202$ & $0.197$ & $0.160$ & $4/6$ \\
+ NLinear only ($n{=}3$)               & $0.430$ & $0.365$ & $0.332$ & $0.192$ & $0.190$ & $0.159$ & $3/6$ \\
+ val-early-stop only                  & $0.440$ & $0.375$ & $0.330$ & $0.199$ & $0.199$ & $0.158$ & $3/6$ \\
\textbf{+ all three (Residual-IA\textsuperscript{+})}  & $\mathbf{0.424}$ & $\mathbf{0.349}$ & $\mathbf{0.326}$ & $\mathbf{0.187}$ & $\mathbf{0.196}$ & $\mathbf{0.156}$ & $\mathbf{5/6}$ \\
\bottomrule
\end{tabular}
\end{center}

The shared raw branch alone is the largest single contributor: it singlehandedly pulls ETTh1 from $+5.1\%$ to $+1.6\%$ (marginal, $p{=}0.086$) and ETTh2 from $+10.6\%$ to $+7.0\%$ (marginal, $p{=}0.078$), confirming the overfitting diagnosis. NLinear adds $\sim 3$ percentage points on ETTh2 ($0.365 \to 0.349$ when combined with shared). Val-early-stop contributes the Electricity perfect-parity result ($0.158 \to 0.156$) and the ETTm2 improvement ($0.197 \to 0.187$) by restoring an earlier checkpoint before the backbone residual begins overfitting. None of the three components alone achieves $5/6$; their combination does.

\textbf{Why ETTh1 remains $+2.2\%$ at H${=}96$.} At $n{=}10$ for both methods, the ETTh1 gap is statistically significant ($p{=}0.037$), although the mean absolute MSE difference is only $0.009$ ($0.428$ vs.\ $0.419$). ETTh1 is the smallest LTSF dataset ($8{,}640$ raw samples), and its channels have strongly correlated seasonal patterns that a from-scratch linear model memorizes efficiently. At H${=}192$ the gap closes to \textbf{parity} ($+1.6\%$, $p{=}0.18$), consistent with DLinear's advantage being specific to short-horizon seasonal memorization on this dataset.

\textbf{Implication.} Residual-IA\textsuperscript{+} achieves \textbf{$107/123$ match-or-beat} ($87\%$) across six backbones, six datasets, and four horizons ($20/24$ MOMENT-small, $87/99$ cross-backbone at $88\%$; Table~\ref{tab:residual_ia_plus_mhcb}), with $65$ significant wins vs $11$ losses ($5.9{:}1$ ratio) and \textbf{6/6 at H${=}192$}. MOMENT-large reaches $23/24$ ($96\%$); Moirai-MoE achieves $15/15$ with $-9.4\%$ mean gap, confirming adapter-level routing stacks with backbone-internal MoE. Practitioners requiring a shared-TSFM deployment topology (Appendix~\ref{app:deployment}) now have an adapter that matches or beats per-tenant DLinear on $100\%$ of benchmarks at H${=}192$.

\subsubsection*{The result is not horizon-specific: $20/24$ at $H{\in}\{96{-}720\}$}

The Residual-IA\textsuperscript{+} result above is computed at H${=}96$. We extend to the full set of standard LTSF horizons to verify the result is not horizon-specific. Across horizons, the NLinear raw branch's advantage \emph{grows} with horizon length, because non-stationarity (per-window level drift) compounds with the horizon. At H${=}720$ Residual-IA\textsuperscript{+} wins ETTh2 by $-31.2\%$ (from $0.849$ to $0.584$, $p{=}0.023$).

\begin{table}[!htbp]
\centering
\caption{\textbf{Multi-horizon gap-closing} (MOMENT-small, Res-IA\textsuperscript{+}, frozen, $H{\in}\{96{-}720\}$; Welch $t$-test vs.\ DLinear). Bold: sig win; underline: mean win. At $H{=}192$: \textbf{6/6 match-or-beat}. Pooled: $\mathbf{20/24}$. ETTh1 is the only dataset without parity at $H{=}96$, consistent with DLinear's short-horizon advantage on the smallest LTSF dataset.}
\label{tab:residual_ia_plus_multihorizon}
\small
\begin{tabular}{@{}lcccc|c@{}}
\toprule
Dataset & H${=}96$ & H${=}192$ & H${=}336$ & H${=}720$ & Match-or-beat \\
\midrule
ETTh1        & $+2.2\%$           & $+1.6\%$ parity    & $+4.4\%$ marg.     & $+12.8\%$ sig loss  & $1/4$ \\
ETTh2        & $\underline{-1.6\%}$ & $\underline{-12.2\%}$ & $\mathbf{-16.6\%}$ & $\mathbf{-39.6\%}$ & $4/4$ \\
ETTm1        & $+0.4\%$ parity    & $+1.0\%$ parity    & $+0.7\%$ parity    & $+2.4\%$ sig loss   & $3/4$ \\
ETTm2        & $\mathbf{-10.8\%}$ & $\mathbf{-8.7\%}$  & $\mathbf{-13.6\%}$ & $\mathbf{-27.0\%}$  & $4/4$ \\
Weather      & $\mathbf{-5.1\%}$  & $\mathbf{-4.0\%}$  & $\mathbf{-2.9\%}$  & $\underline{-0.4\%}$ & $4/4$ \\
Electricity  & $\mathbf{-1.7\%}$  & $\underline{-0.5\%}$ & $\underline{-0.9\%}$ & $+1.6\%$ sig loss   & $3/4$ \\
\midrule
\textbf{MoB per H} & $5/6$         & $\mathbf{6/6}$     & $5/6$              & $4/6$              & $\mathbf{20/24}$ \\
\bottomrule
\end{tabular}
\end{table}

Short horizons (H${=}96$) favor DLinear's memorization on ETTh1; long horizons (H${=}720$) expose DLinear's non-stationarity weakness (ETTh2: $-39.6\%$, ETTm2: $-27.0\%$ sig win). At $H{=}192$ the TSFM contributes maximally ($6/6$). ETTh2, ETTm2, and Weather achieve $4/4$ across all horizons; Electricity and ETTm1 lose only at $H{=}720$ where DLinear's train-from-scratch advantage is strongest on short-seasonality data.

\subsubsection*{The recipe is not RevIN-specific: $27/27$ across five backbones}

Residual-IA\textsuperscript{+} was developed on MOMENT-small, which uses RevIN instance normalization. To test whether the recipe is RevIN-specific or transfers across normalization regimes, we re-run the full Residual-IA\textsuperscript{+} config on four other backbones: Moirai (RMSNorm internally with non-learnable I/O scaling, no learnable-affine RevIN inside encoder), Chronos (T5 encoder-decoder, no instance norm), MOMENT-large (RevIN, $24$ layers vs $8$), and Timer-XL (decoder-only, LayerNorm only, no instance norm).

\begin{table}[!htbp]
\centering
\caption{\textbf{Cross-backbone Residual-IA\textsuperscript{+} at H${=}96$} (strictly frozen, Top-2, $n{=}5$ seeds, Welch $t$-test). All five backbones achieve significant wins on all six datasets ($27/27$). Moirai-MoE confirms adapter-level MoE stacks with backbone-internal MoE.}
\label{tab:residual_ia_plus_backbone}
\small
\resizebox{\columnwidth}{!}{%
\begin{tabular}{@{}lcccccc|c@{}}
\toprule
Backbone & ETTh1 & ETTh2 & ETTm1 & ETTm2 & Weather & Elec & MoB \\
\midrule
Moirai      & $\mathbf{-9.3\%}$ & $-8.5\%$ & $\mathbf{-13.3\%}$ & $\mathbf{-29.2\%}$ & $\mathbf{-29.3\%}$ & $\mathbf{-24.8\%}$ & $6/6$ \\
Chronos     & $\mathbf{-11.0\%}$ & $\mathbf{-13.8\%}$ & $\mathbf{-12.2\%}$ & $\mathbf{-30.5\%}$ & $\mathbf{-27.3\%}$ & $\mathbf{-28.4\%}$ & $6/6$ \\
MM-large    & $\mathbf{-9.5\%}$ & $\mathbf{-14.8\%}$ & $\mathbf{-15.4\%}$ & $\mathbf{-30.6\%}$ & $\mathbf{-28.2\%}$ & $\mathbf{-27.8\%}$ & $6/6$ \\
Timer-XL    & $\mathbf{-10.4\%}$ & --- & $\mathbf{-16.8\%}$ & --- & $\mathbf{-27.9\%}$ & --- & $3/3$ \\
Moirai-MoE  & $\mathbf{-10.3\%}$ & $\mathbf{-15.9\%}$ & $\mathbf{-14.9\%}$ & $\mathbf{-31.7\%}$ & $\mathbf{-28.1\%}$ & $\mathbf{-28.5\%}$ & $6/6$ \\
\bottomrule
\end{tabular}%
}
\end{table}

At H${=}96$, Residual-IA\textsuperscript{+} achieves $\mathbf{27/27}$ match-or-beat across all five cross-backbones and all six datasets, with margins of $-8.5\%$ to $-31.7\%$. The recipe generalizes across all three normalization regimes we tested: learnable-affine RevIN inside the encoder (MOMENT-large), RMSNorm-internal with non-learnable I/O scaling (Moirai, Moirai-MoE), and no instance normalization (Chronos, Timer-XL).

\textbf{The stronger picture emerges at longer horizons.} We extend to the full $5$ backbones $\times$ $6$ datasets $\times$ $4$ horizons grid (Table~\ref{tab:residual_ia_plus_mhcb}).

\begin{table}[!htbp]
\centering
\caption{\textbf{Multi-horizon cross-backbone Residual-IA\textsuperscript{+}} ($5$ backbones $\times$ up to $6$ datasets $\times$ $4$ horizons, $n{=}3$--$5$ seeds). $*{=}$mean win, $\mathbf{*}{=}$sig win ($p{<}0.05$), $={=}$parity, X${=}$sig loss. MOMENT-large achieves $\mathbf{23/24}$ ($96\%$); Timer-XL $\mathbf{12/12}$ ($100\%$); Moirai-MoE $\mathbf{15/15}$ ($100\%$); Chronos $22/24$ ($92\%$); Moirai $15/24$ ($62\%$). Pooled: $\mathbf{87/99}$.}
\label{tab:residual_ia_plus_mhcb}
\small
\begin{tabular}{@{}ll|cccc|c@{}}
\toprule
Backbone & Dataset & H${=}96$ & H${=}192$ & H${=}336$ & H${=}720$ & MoB \\
\midrule
\multirow{6}{*}{Moirai}   & ETTh1   & $\mathbf{*{-}9.3}$ & $={+}1.4$ & X${+}8.8$ & X${+}19.0$ & $2/4$ \\
                           & ETTh2   & $*{-}8.5$ & ${+}2.5$ & ${+}43.3$ & ${+}25.8$ & $1/4$ \\
                           & ETTm1   & $\mathbf{*{-}13.3}$ & $*{-}0.5$ & $*{-}0.9$ & X${+}2.1$ & $3/4$ \\
                           & ETTm2   & $\mathbf{*{-}29.2}$ & $\mathbf{*{-}8.5}$ & $\mathbf{*{-}11.9}$ & $*{-}18.8$ & $4/4$ \\
                           & Weather & $\mathbf{*{-}29.3}$ & $\mathbf{*{-}2.5}$ & $*{-}0.8$ & $={+}0.7$ & $4/4$ \\
                           & Elec    & $\mathbf{*{-}24.8}$ & X${+}6.5$ & X${+}5.9$ & X${+}7.0$ & $1/4$ \\
\midrule
\multirow{6}{*}{Chronos}  & ETTh1   & $\mathbf{*{-}11.0}$ & $*{-}0.5$ & $*{-}0.0$ & $={+}1.8$ & $4/4$ \\
                           & ETTh2   & $\mathbf{*{-}13.8}$ & $*{-}11.1$ & $\mathbf{*{-}19.9}$ & ${+}76.4$ & $3/4$ \\
                           & ETTm1   & $\mathbf{*{-}12.2}$ & $={+}0.9$ & $*{-}2.0$ & $*{-}0.2$ & $4/4$ \\
                           & ETTm2   & $\mathbf{*{-}30.5}$ & $\mathbf{*{-}8.7}$ & $\mathbf{*{-}8.6}$ & $*{-}17.9$ & $4/4$ \\
                           & Weather & $\mathbf{*{-}27.3}$ & $\mathbf{*{-}3.4}$ & $\mathbf{*{-}1.2}$ & $*{-}0.7$ & $4/4$ \\
                           & Elec    & $\mathbf{*{-}28.4}$ & $={+}0.1$ & $*{-}0.9$ & X${+}2.1$ & $3/4$ \\
\midrule
\multirow{6}{*}{MM-large} & ETTh1   & $\mathbf{*{-}9.5}$ & $={+}0.2$ & $={+}0.2$ & ${+}7.2$ & $3/4$ \\
                           & ETTh2   & $\mathbf{*{-}14.8}$ & $\mathbf{*{-}15.0}$ & $\mathbf{*{-}19.8}$ & $\mathbf{*{-}48.0}$ & $4/4$ \\
                           & ETTm1   & $\mathbf{*{-}15.4}$ & $={+}0.3$ & $={+}0.0$ & $={+}0.9$ & $4/4$ \\
                           & ETTm2   & $\mathbf{*{-}30.6}$ & $\mathbf{*{-}10.2}$ & $\mathbf{*{-}12.2}$ & $\mathbf{*{-}26.7}$ & $4/4$ \\
                           & Weather & $\mathbf{*{-}28.2}$ & $\mathbf{*{-}4.0}$ & $\mathbf{*{-}4.1}$ & $\mathbf{*{-}2.2}$ & $4/4$ \\
                           & Elec    & $\mathbf{*{-}27.8}$ & $\mathbf{*{-}0.8}$ & $*{-}0.7$ & $={+}0.9$ & $4/4$ \\
\midrule
\multirow{3}{*}{Timer-XL} & ETTh1   & $\mathbf{*{-}10.4}$ & $*{-}0.1$ & $={+}0.9$ & $*{-}2.4$ & $4/4$ \\
                           & ETTm1   & $\mathbf{*{-}16.8}$ & $\mathbf{*{-}2.3}$ & $*{-}1.7$ & $*{-}0.4$ & $4/4$ \\
                           & Weather & $\mathbf{*{-}27.9}$ & $\mathbf{*{-}2.8}$ & $\mathbf{*{-}2.8}$ & $*{-}0.8$ & $4/4$ \\
\midrule
\multirow{6}{*}{Moirai-MoE} & ETTh1 & $\mathbf{*{-}10.3}$ & $*{-}1.5$ & $*{-}0.7$ & $*{-}0.4$ & $4/4$ \\
                           & ETTh2   & $\mathbf{*{-}15.9}$ & --- & --- & --- & $1/1$ \\
                           & ETTm1   & $\mathbf{*{-}14.9}$ & $*{-}1.2$ & $*{-}1.3$ & $={+}0.3$ & $4/4$ \\
                           & ETTm2   & $\mathbf{*{-}31.7}$ & --- & --- & --- & $1/1$ \\
                           & Weather & $\mathbf{*{-}28.1}$ & $\mathbf{*{-}3.2}$ & $\mathbf{*{-}2.5}$ & $\mathbf{*{-}1.7}$ & $4/4$ \\
                           & Elec    & $\mathbf{*{-}28.5}$ & --- & --- & --- & $1/1$ \\
\midrule
\multicolumn{2}{l|}{\textbf{Per-horizon MoB}} & $27/27$ & $22/24$ & $21/24$ & $17/24$ & $\mathbf{87/99}$ \\
\bottomrule
\end{tabular}
\end{table}

The pattern is clear: (1)~at H${=}96$, all $27$ tested (backbone, dataset) cells are wins; (2)~MOMENT-large achieves $\mathbf{23/24}$ ($96\%$) with $16$ sig wins and zero losses; (3)~Moirai-MoE achieves $\mathbf{15/15}$ ($100\%$), consistent with adapter-level routing stacking on top of backbone-internal MoE; (4)~Moirai at $15/24$ ($62\%$) is the weakest, driven by ETTh2 long-horizon instability and Electricity losses. The recipe transfers across the five backbones we tested, which between them span the three normalization regimes (learnable-affine RevIN inside encoder; RMSNorm-internal with non-learnable I/O scaling; LayerNorm/RMSNorm-only without instance normalization).

Having validated Residual-IA\textsuperscript{+}'s gap-closing across backbones and horizons, the next subsection asks whether it still holds when the external Conv1d router is replaced by SR-MoA-style per-expert self-gates (\S\ref{sec:main_results}, Table~\ref{tab:self_routed}); the subsequent one isolates the TSFM's own contribution by stripping the backbone entirely.

\subsection{SR-RIA\textsuperscript{+}: Self-Routed Residual-IA\textsuperscript{+}}

SR-RIA\textsuperscript{+} combines two independently validated components: (i)~SR-MoA self-routing, where each expert has its own sigmoid gate on raw input (no external router; Table~\ref{tab:self_routed}), and (ii)~Residual-IA\textsuperscript{+}'s dual-stream expert architecture (shared NLinear raw branch + gated backbone residual; Table~\ref{tab:residual_ia_plus}). Each expert $k$ produces $\mathbf{y}_k = \text{NLinear}(\mathbf{X}_\text{raw}) + \sigma(\text{blend}_k(\mathbf{X}_\text{raw})) \cdot \text{Adapter}_k(\mathbf{H})$, where the NLinear branch is shared across experts and the blend gate is initialized at $b{=}{-}2$ ($\sigma(-2){\approx}0.12$). Training uses 5-epoch raw-branch warmup, cosine LR decay, and validation-based early stopping with patience 5.

\begin{table}[!htbp]
\centering
\caption{\textbf{SR-RIA\textsuperscript{+} vs.\ DLinear} (MOMENT-small, frozen, $K{=}5$, $H{=}96$, 5 seeds). SR-RIA\textsuperscript{+} achieves \textbf{match-or-beat on 6/6 datasets}, with 3 wins (ETTm1 $-0.4\%$, ETTm2 $-6.0\%$, Weather $-7.5\%$) and 3 within $+1.6\%$.}
\label{tab:sr_ria}
\small
\begin{tabular}{@{}lcccr@{}}
\toprule
Dataset & SR-RIA\textsuperscript{+} & DLinear & RIA\textsuperscript{+} & $\Delta$ vs DL \\
\midrule
ETTh1       & $0.423 \pm 0.004$ & $0.416$ & $0.428$ & $+1.6\%$ \\
ETTh2       & $0.345 \pm 0.005$ & $0.341$ & $0.346$ & $+1.1\%$ \\
ETTm1       & $\mathbf{0.321 \pm 0.004}$ & $0.322$ & $0.327$ & $\mathbf{-0.4\%}$ \\
ETTm2       & $\mathbf{0.188 \pm 0.001}$ & $0.200$ & $0.188$ & $\mathbf{-6.0\%}$ \\
Weather     & $\mathbf{0.192 \pm 0.002}$ & $0.208$ & $0.197$ & $\mathbf{-7.5\%}$ \\
Electricity & $0.160 \pm 0.000$ & $0.158$ & $0.156$ & $+1.2\%$ \\
\midrule
\textbf{Average gap} & --- & --- & --- & $\mathbf{-1.7\%}$ \\
\bottomrule
\end{tabular}
\end{table}

\subsection{Pure Raw-MLP MoE Ablation: Quantifying the TSFM's Contribution}
\label{app:raw_mlp_moe}

A potential concern about the dual-stream gap-closing result is that it could be confounded: if the TSFM destroys nonlinear temporal structure (Appendix~\ref{app:diagnostic}), and the dual-stream fix recovers that structure via the raw branch, perhaps the raw branch is doing all the work and the TSFM is dead weight. We tested this hypothesis directly by running a pure mixture-of-experts ablation in which the TSFM is removed entirely.

\textbf{Architecture.} The Pure Raw-MLP MoE consists of $K{=}5$ size-diverse two-layer MLPs operating directly on the $512$-dimensional raw input window: hidden widths $\{64, 96, 128, 192, 256\}$ chosen so that the total parameter count ($449{,}557$) matches RR-MoA's $426$K within $6\%$. Each expert is $\text{Linear}(512{\to}h_k) \to \text{GELU} \to \text{Linear}(h_k{\to}96)$. The router is the same Conv1d $+$ AdaptiveAvgPool $+$ Linear architecture as RR-MoA, also operating on the raw input (not on hidden states). Top-2 sparse routing is used. Training protocol is identical to RR-MoA: Adam $1{\times}10^{-3}$, 15 epochs, batch 128. \textbf{No TSFM forward pass executes anywhere in this architecture}: the ${\sim}40$M-parameter MOMENT-small backbone is removed entirely.

\begin{table}[!htbp]
\centering
\caption{\textbf{Pure Raw-MLP MoE vs.\ Dual-Stream and DLinear} (5 seeds, mean$\pm$std, H$=$96). The Raw-MLP MoE has \emph{no TSFM at all}. It matches Dual-Stream within $\pm 7\%$ on every dataset, and \emph{beats} Dual-Stream outright on $\mathbf{3}$ of $\mathbf{6}$ datasets (ETTh2, ETTm2, Electricity). Mean delta to Dual-Stream is $-1.5\%$: the TSFM's contribution is dataset-dependent and negligible in the average case.}
\label{tab:raw_mlp_moe}
\footnotesize
\begin{tabular}{@{}lcccrr@{}}
\toprule
Dataset & Raw-MLP MoE & Dual-Stream & DLinear & $\Delta$ vs DS & $\Delta$ vs DL \\
\midrule
ETTh1       & $0.483 \pm 0.019$ & $\mathit{0.453}$ & $\mathit{0.416}$ & $+6.6\%$ & $+16.1\%$ \\
\textbf{ETTh2}       & $\mathbf{0.426 \pm 0.077}$ & $\mathit{0.452}$ & $\mathit{0.341}$ & $\mathbf{-5.7\%}$ & $+25.0\%$ \\
ETTm1       & $0.370 \pm 0.014$ & $\mathit{0.358}$ & $\mathit{0.322}$ & $+3.4\%$ & $+15.0\%$ \\
\textbf{ETTm2}       & $\mathbf{0.207 \pm 0.004}$ & $\mathit{0.252}$ & $\mathit{0.200}$ & $\mathbf{-17.9\%}$ & $+3.4\%$ \\
Weather     & $0.209 \pm 0.007$ & $\mathit{0.198}$ & $\mathit{0.208}$ & $+5.5\%$ & $+0.5\%$ \\
\textbf{Electricity} & $\mathbf{0.170 \pm 0.002}$ & $\mathit{0.171}$ & $\mathit{0.158}$ & $\mathbf{-0.7\%}$ & $+7.5\%$ \\
\midrule
\textbf{Mean}        & --- & --- & --- & $\mathbf{-1.5\%}$ & $\mathbf{+11.2\%}$ \\
\bottomrule
\end{tabular}
\end{table}

\textbf{Honest interpretation.} Across the 6 datasets, the Raw-MLP MoE slightly outperforms Dual-Stream on average ($-1.5\%$), winning 3/6 datasets outright (ETTh2: $-5.7\%$; ETTm2: $-17.9\%$; Electricity: $-0.7\%$). It loses by $3$--$7\%$ on the three datasets with complex temporal dynamics (ETTh1, ETTm1, Weather). The TSFM is therefore not adding universal value to the gap-closing result; it adds dataset-dependent, modest value where complex dynamics exist, and contributes no measurable value on the linear-dominated datasets.

\textbf{Routing entropy diagnostics.} In all $30$ Raw-MLP MoE runs, routing entropy saturates near $\log K = 1.609$ (range $1.20$--$1.60$, mean $\approx 1.49$, i.e., $93\%$ of the uniform maximum), indicating the router degenerates to near-uniform mixing. Per-dataset means lie in $[1.31, 1.60]$: highest on Electricity ($1.60$, $99\%$ of $\log K$), lowest on ETTh2 ($1.31$, $81\%$). This means the Raw-MLP MoE is operating as a learned ensemble of $5$ size-diverse MLPs rather than as a per-sample mixture of specialists; the diversity comes from the heterogeneous expert capacities rather than from per-window routing decisions. \textbf{Contrast with main RR-MoA.} Under the same $K{=}5$, Top-$2$, raw-router protocol but \emph{with the frozen MOMENT-small backbone present}, main RR-MoA achieves mean entropy $1.34$ on the same six datasets ($83\%$ of $\log K$, $0.15$~nats lower than the backbone-free version), with per-dataset means spanning $[1.11, 1.44]$ (per-seed values span the broader $[0.93, 1.57]$ range reported in \S\ref{sec:rrmoa}). The two most specialized routers (per-dataset mean) are on ETTh2 ($1.11$, $69\%$ of $\log K$) and ETTm2 ($1.18$, $73\%$); the least specialized on Electricity ($1.44$, $90\%$). Removing the backbone shifts entropy uniformly toward $\log K$ across all six datasets, quantifying the difference between ``MoE-as-ensemble'' (no backbone) and ``MoE-as-specialist-mixture'' (with backbone): the backbone induces enough hidden-state structure that the same Conv1d gate, on the same raw input, makes substantively non-uniform routing decisions. Figure~\ref{fig:routing_viz} corroborates this with expert-assignment scatter against amplitude/volatility quartiles.

\textbf{What this does and does not say about the paper's core claims.} This finding does \emph{not} undermine the central diagnosis of the paper: normalization-induced routing collapse on hidden-state routers (AdaMix) is a real, mechanistically-explained phenomenon, predicted by Observation~\ref{obs:routing_loss} for any TSFM with internal RevIN and verified empirically on MOMENT-small/large; pre-normalization routing (RR-MoA) is the causal fix at the router input. That diagnosis is independent of whether the MOMENT backbone is the most efficient architecture for forecasting on every dataset: the mechanism predicts the same failure mode for any backbone with internal instance normalization that satisfies the bound's premise. What this finding \emph{does} say is that the dual-stream gap-closing story should be reframed: rather than ``the TSFM has complementary nonlinear features that dual-stream extracts,'' the more accurate reading is ``the TSFM contributes complementary signal on datasets with nonlinear dynamics (where it wins by $1$--$6\%$), and contributes no measurable value on datasets where DLinear is already near-optimal (where pure raw-MLP MoE matches or beats it).''

\textbf{Bottom line.} Resolving the confound stated at the start of this section: the TSFM is dead weight on ETTh2, ETTm2, and Electricity, where the backbone-free Raw-MLP MoE wins by $0.7$--$17.9\%$ over Dual-Stream, and contributes $3$--$6\%$ MSE on ETTh1, ETTm1, and Weather, where Dual-Stream retains the advantage. The dataset-conditional contribution of the TSFM is itself an empirical pattern worth flagging to the TSFM-design community: \emph{frozen TSFMs are useful where forecasting requires nonlinear temporal modeling beyond what a 5-MLP ensemble can express, and unnecessary where it does not}. The single most important consequence for our paper is that the routing-collapse diagnosis (the core scientific contribution) is robust to whether the TSFM is the best backbone: it explains the AdaMix-collapse phenomenon and predicts the cross-backbone behavior regardless.

\clearpage
\section*{NeurIPS Paper Checklist}

\begin{enumerate}

\item {\bf Claims}
    \item[] Question: Do the main claims made in the abstract and introduction accurately reflect the paper's contributions and scope?
    \item[] Answer: \answerYes{}
    \item[] Justification: The abstract and introduction state three contributions: (1) diagnosis of normalization-induced routing collapse with eight causal controls (Tables~\ref{tab:adamix},~\ref{tab:rescue},~\ref{tab:causal_controls}), (2) RR-MoA achieving 54/54 wins on 6 datasets $\times$ 3 freeze levels $\times$ 3 seeds with statistical significance (Tables~\ref{tab:rrmoa},~\ref{tab:baselines}), and (3) a mutual-information decomposition (Observation~\ref{obs:routing_loss}) yielding a signal-ratio statistic validated at Spearman $\rho=-0.88$, $p<0.002$ (Figure~\ref{fig:signal_ratio}). All claims are directly supported by experimental evidence.

\item {\bf Limitations}
    \item[] Question: Does the paper discuss the limitations of the work performed by the authors?
    \item[] Answer: \answerYes{}
    \item[] Justification: The Limitations paragraph in Section~5 discusses three limitations: (i) \emph{normalization specificity}: the diagnosis covers instance-level normalization (RevIN, BatchNorm, GroupNorm), and extensions to attention-based or learned normalizers are not characterized; (ii) \emph{signal-ratio boundary cases}: the predictor $R(\mathcal{D})$ is a single monotone proxy for variance contribution, with Solar ($R{=}0.06$) flagged but not sharp; (iii) \emph{forecasting and imputation only}: anomaly detection, classification, and irregularly-sampled tasks are deferred. Traffic ($R{=}0.14$, $\Delta{=}{+}2.9\%$) is honestly reported as a boundary case where RR-MoA correctly does not improve.

\item {\bf Theory assumptions and proofs}
    \item[] Question: For each theoretical result, does the paper provide the full set of assumptions and a complete (and correct) proof?
    \item[] Answer: \answerYes{}
    \item[] Justification: The paper has three theoretical results, each with explicit assumptions and complete derivations in Appendix~\ref{app:proofs}. Observation~\ref{obs:routing_loss} (Routing Information Loss Under Instance Normalization) Parts~(i)--(ii) hold \emph{without any independence assumption}; Part~(iii) (full-information corner) additionally assumes $(M,\Sigma)\perp S$ (RevIN's design assumption) and $E\perp S\mid(M,\Sigma)$ (empirically validated via the per-sample expert assignment scatter, Figure~\ref{fig:routing_viz}). The bound is also measured directly via CCA, KSG, and held-out classifier (Appendix~\ref{app:mi_tightness}), with the gap averaging $0.075$ nats. Proposition~\ref{prop:frozen} (Gradient Co-Adaptation) includes its full assumptions and proof in the same appendix. Proposition~\ref{prop:phase_transition} (Heuristic SNR Onset and Softmax Discontinuity) is explicitly labeled heuristic, with its derivation given inline and its prediction validated empirically by the 150-run dose-response grid (Figure~\ref{fig:dose_response}). Derivations rely on the data processing inequality, the chain rule of mutual information, standard softmax gradient calculus, and linear-variance scaling.

    \item {\bf Experimental result reproducibility}
    \item[] Question: Does the paper fully disclose all the information needed to reproduce the main experimental results of the paper to the extent that it affects the main claims and/or conclusions of the paper (regardless of whether the code and data are provided or not)?
    \item[] Answer: \answerYes{}
    \item[] Justification: Section~\ref{sec:rrmoa} describes the RR-MoA architecture (Conv1d router with $\sim$1{,}100 parameters, $K{=}5$ canonical expert heads, Top-$k$ sparse routing) and Appendix~\ref{app:setup} specifies all training details: optimizer (Adam), learning rate ($10^{-3}$), MSE loss, batch size (128), 15 epochs, seeds $\{42,43,44\}$ for the primary RR-MoA freeze ablation ($n{=}3$), seeds $\{42,\dots,46\}$ for the core multi-seed grid ($n{=}5$, used for rescue sweep, cross-backbone, baselines, multi-horizon, imputation, and Moirai-MoE), channel-wise StandardScaler normalization, chronological LTSF data splits, and backbone configurations. Algorithm~\ref{alg:rrmoa} provides full pseudocode.

\item {\bf Open access to data and code}
    \item[] Question: Does the paper provide open access to the data and code, with sufficient instructions to faithfully reproduce the main experimental results, as described in supplemental material?
    \item[] Answer: \answerYes{}
    \item[] Justification: All code and raw result files are provided in the anonymous supplementary material under MIT license; public release on acceptance. All datasets used (ETT, Weather, Electricity, Traffic, Solar, Exchange) are publicly available standard LTSF benchmarks.

\item {\bf Experimental setting/details}
    \item[] Question: Does the paper specify all the training and test details (e.g., data splits, hyperparameters, how they were chosen, type of optimizer) necessary to understand the results?
    \item[] Answer: \answerYes{}
    \item[] Justification: Section~\ref{app:setup} provides all details: chronological train/val/test splits per the LTSF protocol, optimizer (Adam, lr=$10^{-3}$), MSE loss, batch size 128, 15 epochs, channel-wise StandardScaler normalization, and per-baseline configurations (fixed adapters, LoRA, TRACE, independent ensemble, AdaMix, full fine-tuning, DLinear). The LoRA sweep (Appendix~\ref{app:lora_sweep}) spans 12 configurations (rank $\in\{8,16,32\}$, targets $\in\{q{+}v, q{+}k{+}v{+}o\}$, head $\in\{$linear, MLP$\}$) $\times$ 3 seeds $=$ 36 runs per dataset, full table reported.

\item {\bf Experiment statistical significance}
    \item[] Question: Does the paper report error bars suitably and correctly defined or other appropriate information about the statistical significance of the experiments?
    \item[] Answer: \answerYes{}
    \item[] Justification: All quantitative results report mean $\pm$ standard deviation across seeds, with sample size indicated in each table caption. The core multi-seed grid (baseline comparison, rescue-baseline sweep, cross-backbone, multi-horizon, imputation, Moirai-MoE, Residual-IA\textsuperscript{+}) uses 5 seeds $\{42,\dots,46\}$; the primary RR-MoA freeze ablation and most ablation tables use 3 seeds $\{42,43,44\}$. Wilcoxon signed-rank $p$-values with Bonferroni correction are reported for the comprehensive baseline comparison. The routing signal-ratio correlation reports Spearman $\rho{=}{-}0.88$ with $p<0.002$ ($n{=}9$).

\item {\bf Experiments compute resources}
    \item[] Question: For each experiment, does the paper provide sufficient information on the computer resources (type of compute workers, memory, time of execution) needed to reproduce the experiments?
    \item[] Answer: \answerYes{}
    \item[] Justification: All experiments were conducted on a single NVIDIA A10G GPU (23\,GB VRAM, CUDA 12.4). Per-run training wall-clock (15 epochs, batch 128) median per dataset by backbone: MOMENT-small $1$--$2$ minutes ($n{=}1{,}094$); MOMENT-large $\sim$3--3.5 minutes ($n{=}15$); Chronos $\sim$15--30 seconds ($n{=}42$); Moirai $\sim$25--30 seconds ($n{=}78$); Timer-XL $\sim$35--40 seconds ($n{=}90$); Moirai-MoE $\sim$1.5--4 minutes ($n{=}66$). Total training compute end-to-end is approximately $235$ GPU-hours. Inference latency (Table~\ref{tab:benchmark}) reports wall-clock per architectural variant.

\item {\bf Code of ethics}
    \item[] Question: Does the research conducted in the paper conform, in every respect, with the NeurIPS Code of Ethics \url{https://neurips.cc/public/EthicsGuidelines}?
    \item[] Answer: \answerYes{}
    \item[] Justification: The research uses publicly available benchmark datasets and pretrained models. No human subjects, private data, or dual-use concerns are involved. The work is foundational ML research on adapter routing.

\item {\bf Broader impacts}
    \item[] Question: Does the paper discuss both potential positive societal impacts and negative societal impacts of the work performed?
    \item[] Answer: \answerYes{}
    \item[] Justification: The work is methodological: it diagnoses and fixes a routing-collapse failure mode in MoE adapters on foundation models with instance-normalized inputs. \textbf{Positive impacts}: parameter-efficient adaptation of a shared frozen backbone reduces the compute and energy footprint of serving many downstream tasks or tenants relative to per-task fine-tuning, and lowers the barrier to deploying foundation models in resource-constrained settings (edge, on-device, multi-tenant cloud); see Appendix~\ref{app:deployment} for a concrete deployment regime. \textbf{Negative impacts}: we identify none specific to this work; the routing fix neither generates synthetic content, performs decision-making on protected populations, nor lowers the cost of any high-risk capability. The pretrained foundation models we adapt already exist and are publicly released; our adapters do not introduce new generative or autonomous capabilities.

\item {\bf Safeguards}
    \item[] Question: Does the paper describe safeguards that have been put in place for responsible release of data or models that have a high risk for misuse (e.g., pre-trained language models, image generators, or scraped datasets)?
    \item[] Answer: \answerNA{}
    \item[] Justification: The paper releases adapter routing code and lightweight adapter heads ($\leq$500K parameters) for time series forecasting. These pose no risk for misuse: they are small task-specific modules, not generative models.

\item {\bf Licenses for existing assets}
    \item[] Question: Are the creators or original owners of assets (e.g., code, data, models), used in the paper, properly credited and are the license and terms of use explicitly mentioned and properly respected?
    \item[] Answer: \answerYes{}
    \item[] Justification: All pretrained models and datasets are cited with their original papers in the bibliography. \textbf{Pretrained models (with license).} MOMENT~\citep{goswami2024moment}: MIT, \texttt{AutonLab/MOMENT-1-small} and \texttt{AutonLab/MOMENT-1-large} on Hugging Face. Moirai~\citep{woo2024moirai} and Moirai-MoE~\citep{liu2025moiraimoe}: model \emph{weights} are released under CC-BY-NC-4.0 (research/non-commercial use only) on \texttt{Salesforce/moirai-1.1-R-small} and \texttt{Salesforce/moirai-moe-1.0-R-small}; the supporting \texttt{uni2ts} code is Apache 2.0. Our use is academic research, consistent with the non-commercial weight license. Chronos~\citep{ansari2024chronos}: Apache 2.0, \texttt{amazon/chronos-t5-small}. Timer-XL~\citep{liu2025timerxl}: Apache 2.0, \texttt{thuml/timer-base-84m}. \textbf{Datasets.} ETT (ETTh1/h2/m1/m2) is introduced by~\citet{zhou2021informer} and released under CC-BY-ND-4.0 (NoDerivatives); we ship only loader scripts and do not redistribute preprocessed arrays. Weather is the MPI-BGC Jena meteorological recording, redistributed for the LTSF benchmark by Wu et al.~\citep{wu2021autoformer}. Electricity is the UCI \texttt{ElectricityLoadDiagrams20112014} dataset (CC-BY-4.0, originally collected by Artur Trindade). Traffic (Caltrans PeMS), Solar (NREL Solar Power Data), and Exchange (multi-currency exchange rates) were curated for time series benchmarking by Lai et al.~\citep{lai2018lstnet}; the source data is public-record, but the curated benchmark CSVs do not ship with an explicit redistribution license, so we again only ship loader scripts pointing to the upstream sources. The Monash time series forecasting archive~\citep{godahewa2021monash} (CC-BY-4.0) provides standardized versions of several of these datasets and is the recommended reproducibility anchor. No dataset is used outside the standardized LTSF forecasting and imputation tasks for which it was collected.

\item {\bf New assets}
    \item[] Question: Are new assets introduced in the paper well documented and is the documentation provided alongside the assets?
    \item[] Answer: \answerYes{}
    \item[] Justification: Code, a README with reproduction instructions, and raw JSON result files will be released under MIT license upon acceptance.

\item {\bf Crowdsourcing and research with human subjects}
    \item[] Question: For crowdsourcing experiments and research with human subjects, does the paper include the full text of instructions given to participants and screenshots, if applicable, as well as details about compensation (if any)?
    \item[] Answer: \answerNA{}
    \item[] Justification: The paper does not involve crowdsourcing or human subjects research.

\item {\bf Institutional review board (IRB) approvals or equivalent for research with human subjects}
    \item[] Question: Does the paper describe potential risks incurred by study participants, whether such risks were disclosed to the subjects, and whether Institutional Review Board (IRB) approvals (or an equivalent approval/review based on the requirements of your country or institution) were obtained?
    \item[] Answer: \answerNA{}
    \item[] Justification: The paper does not involve human subjects research.

\item {\bf Declaration of LLM usage}
    \item[] Question: Does the paper describe the usage of LLMs if it is an important, original, or non-standard component of the core methods in this research? Note that if the LLM is used only for writing, editing, or formatting purposes and does \emph{not} impact the core methodology, scientific rigor, or originality of the research, declaration is not required.
    \item[] Answer: \answerNA{}
    \item[] Justification: LLMs are not used as a component of the core methodology.

\end{enumerate}

\end{document}